\documentclass{article}
\usepackage{microtype} 
\usepackage{graphicx}
\usepackage{subfigure}
\usepackage{booktabs}  
\usepackage{hyperref} 

\usepackage[accepted]{icml2023}
\usepackage{amsmath}
\usepackage{amssymb}
\usepackage{mathtools}
\usepackage{amsthm}
\usepackage[capitalize,noabbrev]{cleveref}
\theoremstyle{plain}
\newtheorem{theorem}{Theorem}[section]
\newtheorem{proposition}[theorem]{Proposition}
\usepackage[textsize=tiny]{todonotes}
 
\usepackage{multirow}
\usepackage{dblfloatfix} 
\usepackage{float}
\usepackage{caption}
\usepackage{xcolor} %for editing

\icmltitlerunning{Revisiting Bellman Errors for Offline Model Selection}
\begin{document}

\twocolumn[
\icmltitle{Revisiting Bellman Errors for Offline Model Selection}

\icmlsetsymbol{equal}{*}

\begin{icmlauthorlist}
\icmlauthor{Joshua P Zitovsky}{uncbios}
\icmlauthor{Daniel de Marchi}{uncbios}
\icmlauthor{Rishabh Agarwal}{brain,mila,equal}
\icmlauthor{Michael R Kosorok}{uncbios,equal}
\end{icmlauthorlist}

\icmlaffiliation{uncbios}{Department of Biostatistics, UNC Chapel Hill, North Carolina USA}
\icmlaffiliation{brain}{Google DeepMind}
\icmlaffiliation{mila}{Mila}
\icmlcorrespondingauthor{Josh Zitovsky}{joshz@live.unc.edu}

\icmlkeywords{Offline Reinforcement Learning, Offline RL, Deep Reinforcement Learning, Deep RL, Off-Policy Evaluation, OPE, Model Selection, Offline Model Selection, Hyperparameter Tuning, Offline Hyperparameter Tuning, Bellman Errors, Precision Medicine, Dynamic Treatment Regimes}
\vskip 0.3in
]

\printAffiliationsAndNotice{\icmlEqualContribution} % otherwise use the standard text.

\begin{abstract}
Offline model selection (OMS), that is, choosing the best policy from a set of many policies given only logged data, is crucial for applying offline RL in real-world settings. One idea that has been extensively explored is to select policies based on the mean squared Bellman error (MSBE) of the associated Q-functions. However, previous work has struggled to obtain adequate OMS performance with Bellman errors, leading many researchers to abandon the idea. To this end, we elucidate why previous work has seen pessimistic results with Bellman errors and identify conditions under which OMS algorithms based on Bellman errors will perform well. Moreover, we develop a new estimator of the MSBE that is more accurate than prior methods. Our estimator obtains impressive OMS performance on diverse discrete control tasks, including Atari games. %We open-source our data and code to enable researchers to conduct OMS experiments more easily.
\end{abstract}

\vspace{-0.2cm}
\section{Introduction}
\vspace{-0.1cm}
\label{sec:intro}

Offline reinforcement learning (RL)~\citep{Ernst2005, Levine2020} focuses on training an agent solely from a fixed dataset of environment interactions. By not requiring any online interactions, offline RL can be applied to real-world settings, such as autonomous driving \cite{Yu2020} and healthcare \cite{Weltz2022}, where online data collection may be expensive or unsafe but large amounts of previously-logged interactions are available. 

While there has been a recent surge of methods that can train an agent offline \citep{Fu2020, gulcehre2020}, such methods typically tune their hyperparameters using online interactions, which undermines the aim of offline RL. To this end, we focus on the problem of \textit{offline model selection (OMS)}, that is, selecting the best policy from a set of many policies using only logged data. Common OMS approaches are based on off-policy evaluation (OPE) algorithms that estimate the expected returns under a target policy using only offline data \citep{Voloshin2021}. Unfortunately, such estimates are often inaccurate \citep{Fu2021}. As an alternative, many works have explored using empirical Bellman errors to perform OMS, but have found them to be poor predictors of value model accuracy \citep{Irpan2019, Paine2020}. This has led to a belief among many researchers that Bellman errors are not useful for OMS \citep{Geron2019, Fujimoto2022}.

To this end, we propose a new algorithm, \textit{Supervised Bellman Validation (SBV)}, that provides a better proxy for the true Bellman errors than empirical Bellman errors. SBV achieves strong performance on diverse tasks ranging from healthcare problems \cite{Klasnja2015} to Atari games \citep{Bellemare2013}. In contrast, competing baselines suffer from limitations that hinder real-world applicability and perform no better than random chance on certain tasks. In addition to demonstrating the potential utility of Bellman errors in OMS, we also investigate \textit{when} they are effective by exploring the factors most predictive of their performance. Our investigations help explain why Bellman errors have achieved mixed performance in the past, provide guidance on how to achieve better performance with these errors, and highlight several avenues for future work. Finally, we open-source our code at \url{https://github.com/jzitovsky/SBV}. To help others conduct experiments on Atari, our repository includes over 1000 trained Q-functions as well as efficient implementations for several deep OMS algorithms (Appendix \ref{append:baselines}).

\vspace{-0.1cm}
\section{Preliminaries} 
\label{sec:background}

\vspace{-0.05cm}
\subsection{Offline Reinforcement Learning}
\vspace{-0.05cm}

In offline RL, we have a static dataset $\mathcal D=\{(s,a,r,s')\}$ of transitions, where we observe the reward $r$ and next state $s'$ after taking action $a$ on state $s$. We assume the data comes from a Markov decision process (MDP) $\mathcal M=(\mathcal S,\mathcal A,T,R,d_0,\gamma)$ \cite{Putterman1994} with state and action space $\mathcal S$ and $\mathcal A$, transition probabilities $T(s'|s,a)$, rewards $R(s,a,s')=r$, initial state probabilities $d_0(s_0)$ and discount factor $\gamma\in[0,1)$. Throughout we assume that $\mathcal A$ is discrete. Our proposed method has limited applicability in continuous control problems for reasons discussed in Appendix \ref{append:cts}, though we discuss potential extensions to overcome this limitation in Section \ref{sec:conclusions}. We assume that the observed state-action pairs in $\mathcal D$ are identically distributed as $P^\mu(s,a)=d^\mu(s)\mu(a|s)$ where $\mu$ is the behavioral policy and $d^\mu$ is the marginal distribution of states over time points induced by policy $\mu$ and MDP $\mathcal M$ \cite{Levine2020}. 

A \textit{Q-function} is any real-valued function of state-action pairs. One such Q-function is the action-value function for policy $\pi$, $Q^\pi(s,a)=\mathbb E_\pi[\sum_{t=0}^\infty\gamma^t R_t|S_0=s,A_0=a]$, where $\mathbb E_\pi$ denotes expectation over MDP $\mathcal M$ and policy $\pi$ \footnote{Many works define Q-functions akin to action-value functions but also refer to their estimates as Q-functions. To avoid confusion, we define a Q-function as \textit{any} function of the state-action.}. The optimal policy $\pi^*$ is a policy whose action-value function equals the optimal action-value function $Q^*$. It is well-known that $\pi^*=\pi_{Q^*}$ where $\pi_Q(s)=\mbox{argmax}_a Q(s,a)$ is the greedy policy of Q-function $Q$ \cite{Sutton2018}. 

Our paper focuses on the \textbf{offline model selection (OMS)} problem where we have a \textit{candidate set} $\mathcal Q=\{Q_1,...,Q_M\}$, or set of estimates for $Q^*$, and our goal is to choose the ``best" among them based on some criterion. For example, $Q_1,...,Q_M$ can be obtained by running a deep RL (DRL) algorithm \cite{Arulkumaran2017} for $M$ iterations and evaluating the Q-Network after each iteration, or it can be obtained by running $M$ different DRL algorithms to convergence. Moreover, the number of candidates $M$ need not be fixed in advance. For example, if one evaluated all the value models in $\mathcal Q$ and determined that none were adequate, one could then augment $\mathcal Q$ with more Q-functions obtained by running more DRL algorithms and evaluate those Q-functions as well.

\vspace{-0.05cm}
\subsection{Bellman Errors}
\label{sec:bellman_error}
\vspace{-0.05cm}

For any Q-function $Q$, the Bellman operator $\mathcal B^*$ satisfies:
\begin{equation}
\resizebox{\linewidth}{!} 
{
$\mathcal B^*Q(s,a)=\mathbb E\left[R_t+\gamma \max\limits_{a'}Q(S_{t+1},a')|S_t=s,A_t=a\right]$.}\label{eq:Bellman}
\end{equation}

It is known that $Q=\mathcal B^*Q$ if and only if $Q=Q^*$ \cite{Sutton2018}. The function $(\mathcal B^*Q)(s,a)$ is known as the \textit{Bellman backup} of Q-function $Q$ and $(Q-\mathcal B^*Q)(s,a)$ is known as its \textit{Bellman error}. As the Bellman errors are zero uniquely for $Q^*$, a reasonable approach is to assess candidates $Q_m, 1\leq m \leq M$ via their \textbf{mean squared Bellman error (MSBE)}: 
\begin{equation}
\mathbb E_{(s,a)\sim P^\mu}\left[\left(Q_m(s,a)-(\mathcal B^*Q_m)(s,a)\right)^2\right].\label{eq:msbe}
\end{equation}
Unfortunately, directly estimating the MSBE from our dataset $\mathcal D$ is not straightforward. For example, consider the \textbf{empirical mean squared Bellman error (EMSBE)}:
\begin{equation}
\mathbb E_{\mathcal D}\left[\left(Q_m(s,a)-r-\gamma\max_{a'\in\mathcal A}Q_m(s',a')\right)^2\right],\label{eq:emsbe}
\end{equation}
where $\mathbb E_{\mathcal D}$ denotes the empirical expectation over observed transitions $(s,a,r,s')\in\mathcal D$. \textit{Empirical} Bellman errors replace the true Bellman backup with a single sample bootstrapped from the observed dataset. Unless the environment is deterministic, the EMSBE will be biased for the true MSBE \cite{Baird1995, Farahmand2010}.

Fitted Q-Iteration (FQI) \cite{Ernst2005} and the DQN algorithm \cite{Mnih2015} perform updates:
\begin{equation*}
\vspace{-0.2cm}
\resizebox{\linewidth}{!} 
{
    $Q^{(k+1)}\leftarrow \underset{f}{\mbox{argmin}} \mathbb E_{\mathcal D}\left[\left(f(s,a)-r-\gamma\max\limits_{a'}Q^{(k)}(s',a')\right)^2\right].$
}
\end{equation*}

The terms $Q^{(k+1)}(s,a)-r-\gamma\max_{a'}Q^{(k)}(s',a')$ are often referred to as empirical Bellman errors as well, with the true Bellman error being $\epsilon^{(k+1)}=Q^{(k+1)}-\mathcal B^*Q^{(k)}$.  We refer to $\epsilon^{(k+1)}$ as a \textit{fixed-target} Bellman error, as the Bellman backup $\mathcal B^*Q^{(k)}$ remains fixed while the Q-function $Q^{(k+1)}$ is being evaluated. In contrast, our version of the Bellman error evaluates a Q-function by taking the difference between itself and its \textbf{own} Bellman backup, i.e. it is a \textit{variable-target} Bellman error. Unlike variable-target Bellman errors, fixed-target Bellman errors can be reliably replaced by their empirical counterparts, at least when using them for model training. As a result, FQI updates will often do a good job at minimizing the true fixed-target Bellman errors $\epsilon^{(k+1)}$ as well. Differences between these errors as well as between FQI and SBV are further discussed in Appendix \ref{append:fqi}.  Unless otherwise specified, ``Bellman errors" refers to variable-target Bellman errors.

\vspace{-0.1cm}
\section{Related Work}
\label{sec:previous}
\vspace{-0.1cm}

The most popular approach for OMS is to use an \textit{off-policy evaluation (OPE)} algorithm, which estimates the marginal expectation of returns $J(\pi)=\mathbb E_{\pi}[\sum_{t=0}^\infty \gamma^tR_t]$ under policies of interest $\pi\in\{\pi_{Q_1},\pi_{Q_2},...,\pi_{Q_M}\}$ from $\mathcal D$ \cite{Voloshin2021}. For example, importance sampling (IS) estimators such as per-decision IS estimators \cite{Precup2000}, doubly-robust IS estimators \cite{Jiang2016, Thomas2016} and marginal IS estimators \cite{Xie2019, Yang2020} estimate $J(\pi)$ from $\mathcal D$ by using importance weights to adjust for the distribution shift. Fitted Q-Evaluation (FQE) estimates $Q^{\pi}$ with an off-policy RL algorithm \cite{Le2019, Paine2020}. Some methods also perform statistical inference on the value function to aid in OMS \citep{Thomas2015, Shi2021}.

Unfortunately, these approaches often have difficulties with accurately estimating $J(\pi)$ \cite{Fu2021}. For example, per-decision IS usually has prohibitively large estimation variance while FQE introduces its own hyperparameters that cannot be easily tuned offline. Doubly-robust and most marginal IS estimators use function approximation to reduce variance, but at the cost of introducing hyperparameter-tuning difficulties shared by FQE.

In contrast to the previous approaches which are \textit{model-free}, model-based approaches estimate the underlying MDP using density estimation techniques \citep{Zhang2021, Voloshin2021model}. Accurately modelling the MDP in complex and high-dimensional settings can be difficult, and the most well-known OMS experiments involving models are restricted to MDPs with low-dimensional states \cite{Fu2021}. In contrast, while there have been a few works that use models to help train policies in pixel-valued settings \cite{Rafailov2021, schrittwieser2021}, we are unaware of previous attempts to tune hyperparameters offline on Atari from model-based roll-outs, and the difficulty of doing this successfully would likely warrant a paper in and of itself. As our proposed method (which achieves strong performance on Atari) is model-free, it is especially advantageous when estimating a model is difficult. That being said, even on a low-dimensional MDP, our method still outperforms model-based roll-outs (see Table \ref{table:model}).

The poor performance of empirical Bellman errors has led to several proposed alternatives. One such alternative, BErMin \cite{Farahmand2010}, estimates an upper bound on the MSBE with strong theoretical guarantees, and uses regression to estimate the Bellman backup similar to our method. Unlike our method, however, BErMin requires calculating tight excess risk bounds of the regression algorithm, which is often impractical in empirical settings. Furthermore, BErMin uses separate datasets to estimate the Q-functions and Bellman backups, reducing the amount of data we can use to estimate both. Finally, empirical performance was never evaluated. 

Another alternative, BVFT \citep{Xie2021, Zhang2021ps}, takes advantage of several theoretical properties of piecewise-constant projections, including the fact that an $L_2$ piecewise-constant projection of the Bellman operator will still be an $L_\infty$ contraction with the same fixed point under restrictive conditions. Instead of estimating the Bellman error directly, BVFT calculates a related criterion with its own theoretical guarantees. Our experiments on BVFT suggest that our method is more robust (see Figure \ref{figure:bvft}). Lastly, ModBE \cite{Lee2022} compares candidate functional classes by running FQI using one functional class and then using the fixed-target empirical Bellman errors minimized at every iteration to evaluate alternative classes. In contrast to our work, ModBE performs OMS based on fixed-target Bellman errors and can only compare nested model classes. See Appendix \ref{append:relation} for more discussion comparing our method to BErMin, BVFT and ModBE. 

\vspace{-0.1cm}
\section{Supervised Bellman Validation}
\label{sec:sbv}

\vspace{-0.05cm}
\subsection{Methodology}
\label{sec:method}
\vspace{-0.05cm}

To understand the intuition behind our algorithm, consider the case where $Q^*(s,a)$ is actually \textbf{known} for observed state-action pairs $(s,a)\in\mathcal D$ and we wish to evaluate candidates $Q_m$, $1\leq m \leq M$ based on how well they estimate $Q^*$. An obvious criterion in this case would be the mean squared error (MSE): 
\begin{equation}
\mathbb E_{(s,a)\sim P^\mu}\left[\left(Q^*(s,a)-Q_m(s,a)\right)^2\right].\label{eq:q_mse}
\end{equation}
While $P^\mu$ is unknown, we can still estimate the expectation in Equation \ref{eq:q_mse} by randomly partitioning $80\%$ of the trajectories present in $\mathcal D$ into a training set $\mathcal D_T$ and reserving the remaining $20\%$ of trajectories as a validation set $\mathcal D_V$. We would then generate candidates $\mathcal Q=\{Q_1,...,Q_M\}$ by running DRL algorithms on $\mathcal D_T$ with $M$ different hyperparameter configurations, and use $\mathcal D_V$ to estimate the MSE for each $Q_m$ as $\mathbb E_{\mathcal D_V}\left[\left(Q^*(s,a)-Q_m(s,a)\right)^2\right]$. 

Typically, the targets $Q^*(s,a), (s,a)\in\mathcal D$ are not known: this is what separates supervised learning from RL. Instead of using a criterion based on Equation \ref{eq:q_mse}, our algorithm, \textit{Supervised Bellman Validation (SBV)}, uses a surrogate criterion based on the MSBE (Equation \ref{eq:msbe}). The relationship between estimation error and Bellman error is discussed more in Section \ref{sec:theory}. Similar to the supervised learning case, SBV creates a training set $\mathcal D_T$ and a validation set $\mathcal D_V$ by randomly partitioning trajectories from $\mathcal D$, and trains $M$ Q-functions $\mathcal Q=\{Q_1,...,Q_M\}$ on $\mathcal D_T$. 

Note that the MSBE contains \textbf{two} unknown quantities: the population density $P^\mu$, and the $M$ Bellman backup functions $\mathcal B^*Q_m, 1 \leq m \leq M$. We can see from Equation \ref{eq:Bellman} that each $(\mathcal B^*Q_m)(s,a)$ is just a conditional expectation. Moreover, it is well-known that a regression algorithm with an MSE loss function will estimate the conditional expectation of its targets \cite{Hastie2009}. Therefore, the $M$ Bellman backup functions can be estimated by running $M$ regression algorithms on $\mathcal D_T$, with the $m$th such algorithm estimating $\mathcal B^*Q_m$ by fitting function $f$ to minimize:
\begin{equation}
\mathbb E_{\mathcal D_T}\left[\left(r+\gamma\max_{a'}Q_m(s',a')-f(s,a)\right)^2\right].\label{eq:mse_backup}
\end{equation}
We refer to Equation \ref{eq:mse_backup} as the \textit{Bellman backup MSE} of $Q_m$. Denote the fitted models from our regression algorithms as $\widehat{\mathcal B}^*Q_1,...,\widehat{\mathcal B}^*Q_M$. The MSBE for each candidate $Q_m$ can then be estimated as $\mathbb E_{\mathcal D_V}[(Q_m(s,a)-(\widehat{\mathcal B}^*Q_m)(s,a))^2]$.

\begin{algorithm}[t]
  \caption{Supervised Bellman Validation (SBV)}\label{alg:sbv}
  \begin{algorithmic}[1]
    \REQUIRE{Offline dataset $\mathcal D=\{(s,a,r,s')\}$} 
    \REQUIRE{Set of offline RL algorithms \\ $\mathcal H=\{H_1,...,H_M\}$}
    \STATE{Randomly partition trajectories in $\mathcal D$ to training set $\mathcal D_T$ and validation set $\mathcal D_V$}
    \FOR{algorithm $m\in\{1,...,M\}$}
        \STATE{Estimate $Q^*$ as $Q_m$ by running offline RL algorithm $H_m$ on $\mathcal D_T$}
        \STATE{Estimate $\mathcal B^*Q_m$ as $\widehat{\mathcal B}^*Q_m$ by minimizing the Bellman backup MSE of $Q_m$ on $\mathcal D_T$ (Equation \ref{eq:mse_backup})}
        \STATE{Estimate the MSBE of $Q_m$ as \\ $\mathbb E_{\mathcal D_V}[(Q_m(s,a)-(\widehat{\mathcal B}^*Q_m)(s,a))^2]$}
    \ENDFOR
    \STATE{{\bfseries Output:} $Q_{m^*}$ as our estimate of $Q^*$ where $m^*=\mbox{argmin}_{1\leq m \leq M}\mathbb E_{\mathcal D_V}[(Q_m(s,a)-(\widehat{\mathcal B}^*Q_m)(s,a))^2]$}
  \end{algorithmic}
\end{algorithm}

Our algorithm is summarized in Algorithm \ref{alg:sbv}. Here $H_m$ fully specifies an algorithm and its relevant hyperparameters for estimating $Q^*$. For DRL, this would include the training algorithm (e.g. dueling DQN \cite{Wang2016} or QR-DQN \cite{Dabney2018}), the Q-network architecture and the number of training iterations. While the RL algorithm can be tuned via SBV, the regression algorithm employed by SBV to estimate the relevant Bellman backups itself must be tuned. Fortunately, regression algorithms can easily be tuned offline using MSE on a held-out validation set, which in this case would be $\mathcal D_V$. For example, Algorithm \ref{alg:sbv_tune} extends Algorithm \ref{alg:sbv} to tune the regression algorithm, separately for each Bellman backup that is estimated. 

For complex control problems where deep RL is required, it is a good idea to estimate each $\mathcal B^*Q_m, 1 \leq m \leq M$ with a neural network approximator. In such cases, we will refer to this neural network as a \textit{Bellman network}. For computational efficiency, the same Bellman network architecture and training configuration can be used for estimating all Bellman backups. For example, Algorithm \ref{alg:dqn_sbv} provides a computationally-efficient implementation of SBV for tuning the number of training iterations used by DQN.

\vspace{-0.05cm}
\subsection{Theoretical Analysis}
\label{sec:theory}
\vspace{-0.05cm}

We begin with some basic theoretical properties of Bellman errors, empirical Bellman errors and SBV. We assume $|\mathcal S|<\infty$ here: Extensions to uncountable state spaces and relevant mathematical proofs can be found in Appendix \ref{append:proofs}. For any Q-function $Q$, density $P$ of state-action pairs and dataset of transitions $\mathcal{D}=\{(s,a,r,s')\}$, let $||Q||_{P}^2=\mathbb E_{(s,a)\sim P}[Q(s,a)^2]$ and $||Q||_{\mathcal{D}}^2=|\mathcal{D}|^{-1}\sum_{(s,a)\in \mathcal{D}}[Q(s,a)^2]$. Define the empirical Bellman backups for Q-function $Q$ and dataset $\mathcal{D}$ as $(\mathcal B_{\mathcal{D}}Q)(s,a)=r+\gamma\max_{a'}Q(s',a'), \ (s,a,r,s')\in\mathcal{D}$. For example, the MSBE (Equation \ref{eq:msbe}) and the EMSBE (Equation \ref{eq:emsbe}) can be re-written as $||Q_m-\mathcal B^*Q_m||_{P^\mu}^2$ and $||Q_m-\mathcal B_{\mathcal D}Q_m||_{\mathcal D}^2$, respectively. 

The results and proof techniques of our first proposition, Proposition \ref{prop1}, resembles those used by more recent theoretical work (e.g. \citet{Xie2021, Chen2022, Uehara2022}).  Proposition \ref{prop1} states that the candidate Q-function selected by the MSBE is guaranteed to be an accurate estimate of $Q^*$ and have a high-performing greedy policy provided the MSBE of the selected policy is sufficiently small and the observed data covers the state-action space adequately. In other words, the MSBE upper bounds estimation error, lower bounds policy performance and is minimized uniquely at $Q^*$. Moreover, even if the MSBE is unknown and estimated with error, the same results will hold for the estimated MSBE provided the estimation is sufficiently accurate. These properties constitute strong guarantees of the MSBE and justify its utility in OMS.

\begin{proposition}
\label{prop1}
Assume $P^\mu(s,a)\geq \psi$ for some $\psi>0$ and all $(s,a)\in\mathcal S\times\mathcal A$. Let $\hat m(Q_m)$ be an estimate of \\ $||Q_m-\mathcal B^*Q_m||_{P^\mu}$ with absolute estimation error $e(\hat m(Q_m))$ and assume $\hat m(Q_m)\leq\epsilon$ and $e(\hat m(Q_m))\leq\delta$.\\ Then i) $||Q_m-Q^*||_{P^\mu}\leq \frac{1}{\sqrt{\psi}(1-\gamma)}(\epsilon+\delta)$ and \\ ii) $J(\pi^*)-J(\pi_{Q_m})\leq \frac{2}{\psi(1-\gamma)^2}(\epsilon+\delta)$.
\end{proposition}

Proposition \ref{prop1} also suggests a few reasons why the MSBE has performed poorly in previous literature: (1) The MSBE was not estimated accurately; (2) The behavioral policy did not perform enough exploration and there was not sufficient diversity in the observed state-action pairs; (3) The evaluated Q-functions all had MSBE values that were too high. In Appendix \ref{append:proofs}, we conduct a more thorough theoretical analysis of the MSBE and propose additional factors that could impact its performance. We also discuss how we can relax our assumption that $P^\mu(s,a)\geq \psi$ to a slightly weaker coverage assumption that better resembles those made by previous work \cite{Munos2005}. Note that the bounds on estimation error are tighter than those on policy regret: This implies that Bellman errors are more closely associated with estimation error than with policy performance, and will primarily select high-quality policies by selecting accurate Q-functions. 
 
Let $P^{\mathcal D}(s,a)$ be the proportion of state-action pairs in dataset $\mathcal D$ equal to $(s,a)$. When studying the theoretical performance of the EMSBE and SBV, we focus on the setting where $|\mathcal D|=\infty$ or $P^{\mathcal D}=P^\mu$. An important avenue for future work is to extend our theory to finite-sample settings. Proposition \ref{prop2} is similar to theoretical results derived in previous work (e.g. \citet{Farahmand2010}) and states that the EMSBE is not equal to the true MSBE even with infinite samples unless the environment is deterministic. This implies that the EMSBE is biased, with the degree of bias depending on the amount of noise in the MDP. 

\begin{proposition}
\label{prop2}
Assume that $P^{\mathcal D}=P^\mu$. Then:
\begin{equation*}
\begin{split}
||Q_m-\mathcal B_{\mathcal D}Q_m||^2_{\mathcal D}-||Q_m-\mathcal B^*Q_m||^2_{P^\mu}=\quad\quad&\\\mathbb E_{(S_t,A_t)\sim P^\mu}\left\{\mbox{Var}\left[R_t+\gamma\max_{a'\in\mathcal A}Q_m(S_{t+1},a')|S_t,A_t\right]\right\}.
\end{split}
\end{equation*}
\end{proposition}

SBV reduces bias of the EMSBE by using a regression algorithm. The goal of $\widehat{\mathcal B}^* Q_m$ is not to be close to the targets $R_t+\gamma\max_{a'}Q_m(S_{t+1},a')$ from Equation \ref{eq:mse_backup} per se. Instead, we want $\widehat{\mathcal B}^* Q_m$ to be close to $\mathcal B^*Q_m$, or the expectation of these targets conditional on $(S_t,A_t)$. This difference matters because using these targets directly when estimating the MSBE leads to bias, as shown in Proposition \ref{prop2}. As the MSE loss function $\mathbb E[(Y-f(X))^2]$ is minimized at 
$f(X)=\mathbb E[Y|X]$ \cite{Hastie2009}, minimizing Equation \ref{eq:mse_backup} is effective at recovering $\mathcal B^*Q_m$. This is formalized in Proposition \ref{prop3}, which states that SBV recovers the MSBE asymptotically and indicates its potential in reducing bias. 

\begin{proposition}
\label{prop3}
Assume $P^{\mathcal D_V}=P^{\mathcal D_T}=P^\mu$ and $\widehat{\mathcal B}^*Q_m=\mbox{argmin}_{f}||\mathcal B_{\mathcal D_T}Q_m-f||^2_{\mathcal D_T}$. Then $||Q_m-\widehat{\mathcal B}^*Q_m||^2_{\mathcal D_V}=||Q_m-{\mathcal B}^*Q_m||^2_{P^\mu}$.
\end{proposition}

\vspace{-0.1cm}
\section{Empirical Results}
\label{sec:experiments}

\vspace{-0.05cm}
\subsection{Case Study: Toy Environment}
\label{sec:toy}
\vspace{-0.05cm}

\begin{figure}[t]
\vspace{-0.1cm}
\includegraphics[width=\linewidth]{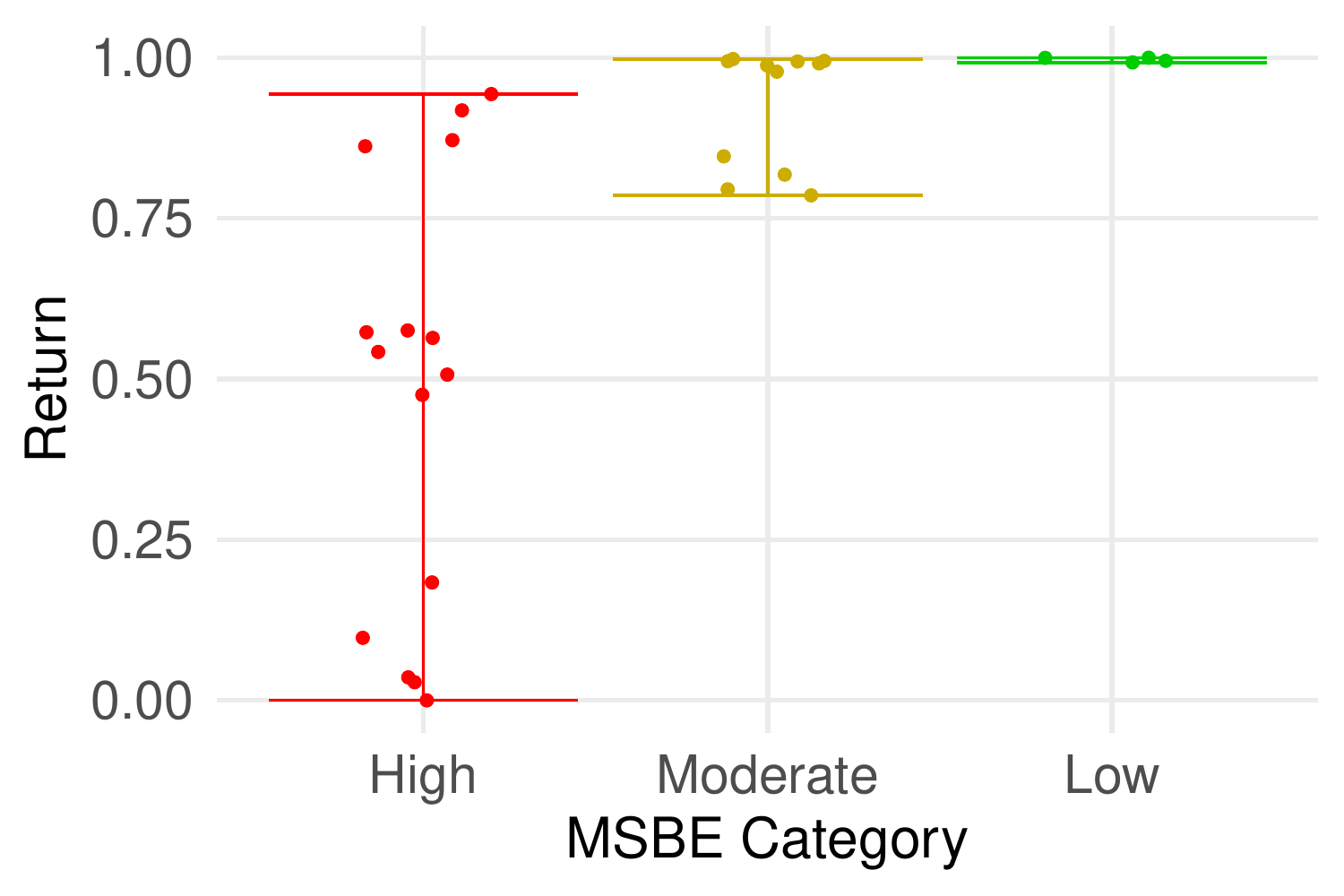}
\centering
\caption{Returns vs. MSBE. Each data point represents an estimate of $Q^*$ and the mean return of its greedy policy. Estimates are grouped by the size of their MSBE. The vertical bars represent the range of observed returns within each category. As the MSBE decreases, this range decreases and returns get more concentrated around the optimal return.
}
\label{figure:msbe}
\vspace{-0.4cm}
\end{figure}

To study the empirical properties of the MSBE (Equation \ref{eq:msbe}), EMSBE (Equation \ref{eq:emsbe}) and SBV algorithm (Algorithm \ref{alg:sbv}), we consider a simple 4-state MDP where the candidate set
consists of $Q^*$ as well as $29$ Q-functions generated by running ridge-regularized polynomial FQI on a small offline dataset (see Appendix \ref{append:toy} for details). In Figure \ref{figure:msbe}, we plot the returns of our various Q-functions and group Q-functions by their MSBE values. MSBE values greater than that of the zero function are considered ``high" while estimates with MSBE values close to zero are considered ``low". We can see that as the MSBE decreases, the floor of the observed return distribution increases and returns get more concentrated around the optimal return. These empirical findings are in-line with Proposition \ref{prop1}, verifying that Bellman errors lower bound the expected return.

We can see from Figure \ref{figure:msbe_extra} that the Spearman correlation between the MSBE and returns is imperfect, but this does not preclude the MSBE from selecting high-performing policies. Because high Spearman correlation is not necessary for OMS, we do not focus on this metric for our experiments. We can also see from Figure \ref{figure:msbe_extra} that among the high MSBE Q-functions, the Q-function with smallest MSBE only has return $0.185$, while the best Q-function still has a return of $0.953$. The issue is that the MSBE values are all too high to be informative. However, once the Q-functions with low MSBE are included, the top Q-functions selected by the MSBE all have returns very close to that of the optimal policy. These results imply that the MSBE will be effective for OMS if our candidate set contains Q-functions with sufficiently low MSBE (again in-line with Proposition \ref{prop1}). 

The noise in the MDP dynamics is controlled by a stochasticity parameter $\phi$, where $\phi=0$ corresponds to a deterministic MDP. Figure \ref{figure:msbe} uses $\phi=0.25$. We then generated offline datasets for different values of $\phi$, and generated our candidate set similar to before. For each candidate $Q_j,1\leq j \leq 30$, SBV estimated $\mathcal B^*Q_j$ using polynomial ridge regression with hyperparameters tuned to minimize Bellman backup MSE on the validation set, as discussed in Algorithm \ref{alg:sbv_tune}. We compared SBV to the EMSBE over the validation set (the EMSBE over the full dataset performed worse). From Figure \ref{figure:noise}, we can see that the EMSBE's performance declines rapidly as noise increases, while the performance of SBV remains stable. Relative to the EMBSE, SBV is more robust to environment noise and reduces bias, in-line with Propositions \ref{prop2} and \ref{prop3}.

\begin{figure}[t]
\vspace{-0.1cm}
\includegraphics[width=\linewidth]{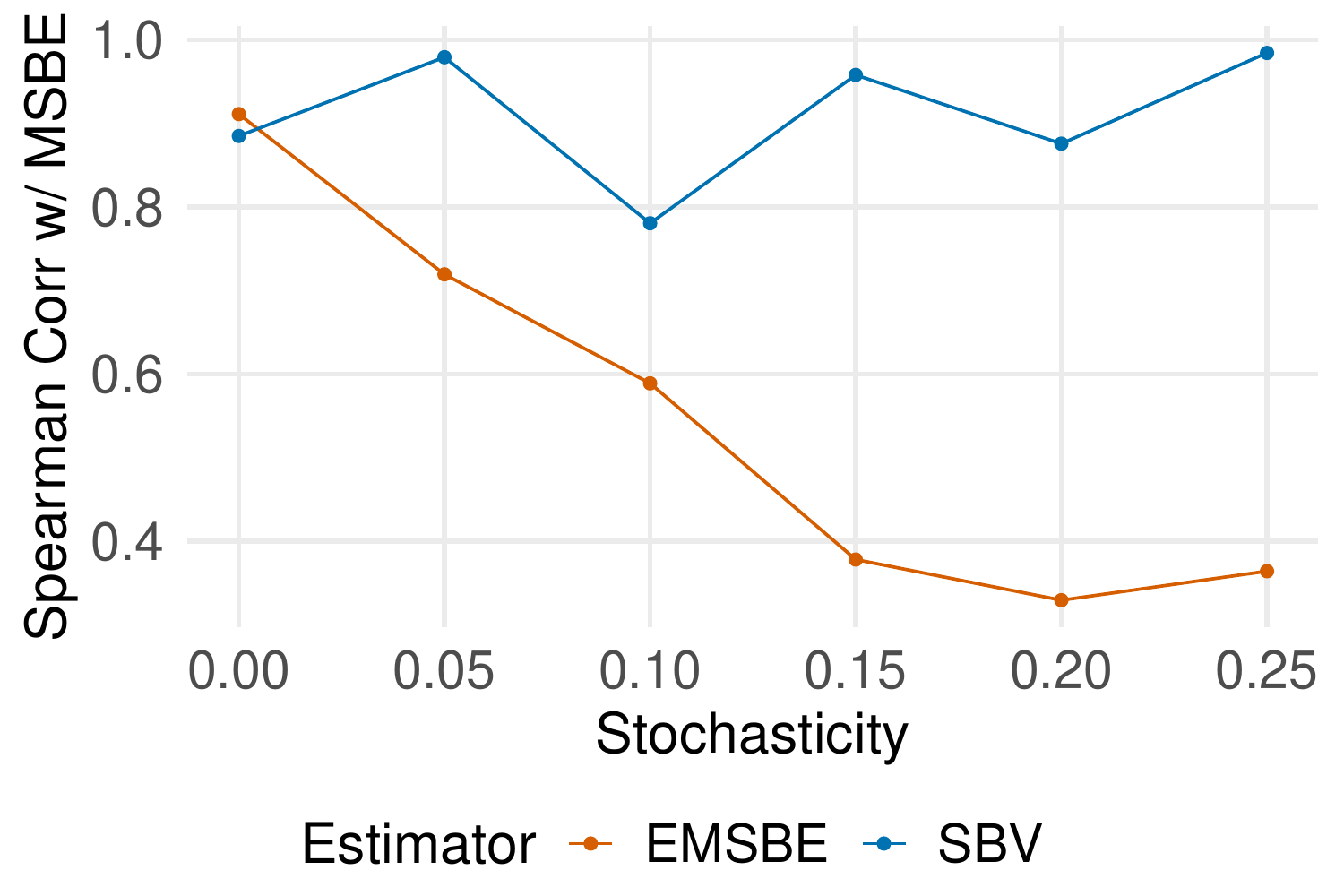}
\centering
\caption{Performance vs. Stochasticity. The environment is deterministic when the stochasticity parameter is zero, and becomes noisier as the stochasticity parameter increases.}
\label{figure:noise}
\vspace{-0.4cm}
\end{figure}

\vspace{-0.05cm}
\subsection{Robotics and Healthcare Environments}
\label{sec:Bike}
\vspace{-0.05cm}

\begin{table*}[b]
\vspace{-0.1cm}
\centering
\captionsetup{margin=1cm}
\caption{Top Policies According to FQE. Note that FQE performance (top-3 policy value) is sensitive to its training algorithm. As FQE cannot tune its hyperparameters offline, this sensitivity precludes it from being a practical OMS algorithm. SBV doesn't have this problem because it can be tuned via validation MSE.}\label{table:fqe}
\begin{tabular}{ llllll } 
\toprule
Dataset & FQE Training Algorithm & Top-3 Policy Value & Top-Ranked Estimator for $Q^*$ & \\
\midrule
\multirow{2}{4em}{\textbf{Bike}} & FQI, $n_{\min}=625,m_{\text{try}}=5$ & 0.239 & FQI, $n_{\min}=1,m_{\text{try}}=1$ \\ 
& FQI, $n_{\min}=5,m_{\text{try}}=3$ & 0.878 & FQI, $n_{\min}=5,m_{\text{try}}=3$  \\
\midrule 
\multirow{2}{4em}{\textbf{mHealth}} & Quadratic LSPI, $\lambda=100$ & 0 & Quadratic LSPI, $\lambda=100$ \\ 
& Quadratic LSPI, $\lambda=0$ & 0.984 & Quadratic LSPI, $\lambda=0$ \\
\bottomrule
\end{tabular}
\vspace{-0.35cm}
\end{table*}

We next assessed the empirical performance of SBV on two well-known discrete control problems: The bicycle balancing problem \cite{Randlov1998} and the mobile health (mHealth) problem \cite{Luckett2020}. These environments were chosen due to their diverse characteristics: the Bicycle MDP has highly nonlinear transition dynamics, sparse rewards and little environmental noise, and is typically associated with larger offline datasets. In contrast, the mHealth MDP has simple transition dynamics, dense rewards and a large amount of environmental noise, and is typically associated with very small offline datasets.

In addition to SBV and validation EMSBE, we also evaluated weighted per-decision importance sampling (WIS) \cite{Precup2000} and Fitted Q-Evaluation (FQE) \cite{Le2019}: WIS is one of the few OPE algorithms that can tune its hyperparameters offline, while FQE has achieved state-of-the-art performance in terms of model-free OMS \cite{Fu2021, Tang2021}. See Appendix \ref{append:baselines} for more discussion of these baselines. As doubly-robust and marginal IS estimators suffer from large variance like WIS or have hyperparameters that cannot be easily tuned offline like FQE, we conjectured that the problems observed from our selected OPE benchmarks would also be observed by these estimators. Limited experiments on BVFT and model-based evaluations were also discussed in Section \ref{sec:previous}, Figure \ref{figure:bvft} and Table \ref{table:model}, though we leave a more comprehensive evaluation to future work.

For the Bicycle control problem, we generated 10 offline datasets consisting of $240$ episodes of $500$ time steps each and our candidate Q-functions were primarily random forest functions fit using FQI, following \citet{Ernst2005}. For the mHealth control problem, we generated 10 offline datasets consisting of $30$ episodes of $25$ time steps each, following \citet{Luckett2020}, and candidate Q-functions were primarily polynomial functions fit by Least Square Policy Iteration (LSPI) \cite{Lagoudakis2003}. When implementing SBV, each Bellman backup function was estimated using a different regression algorithm tuned to minimize validation MSE, in-line with Algorithm \ref{alg:sbv_tune}. Moreover, the true behavioral policy was used when implementing WIS. More details on our environments and experimental setup can be found in Appendix \ref{append:non-atari}. Results are given in Figure \ref{figure:bike}.

\begin{figure}[t]
\vspace{-0.1cm}
\includegraphics[width=\linewidth]{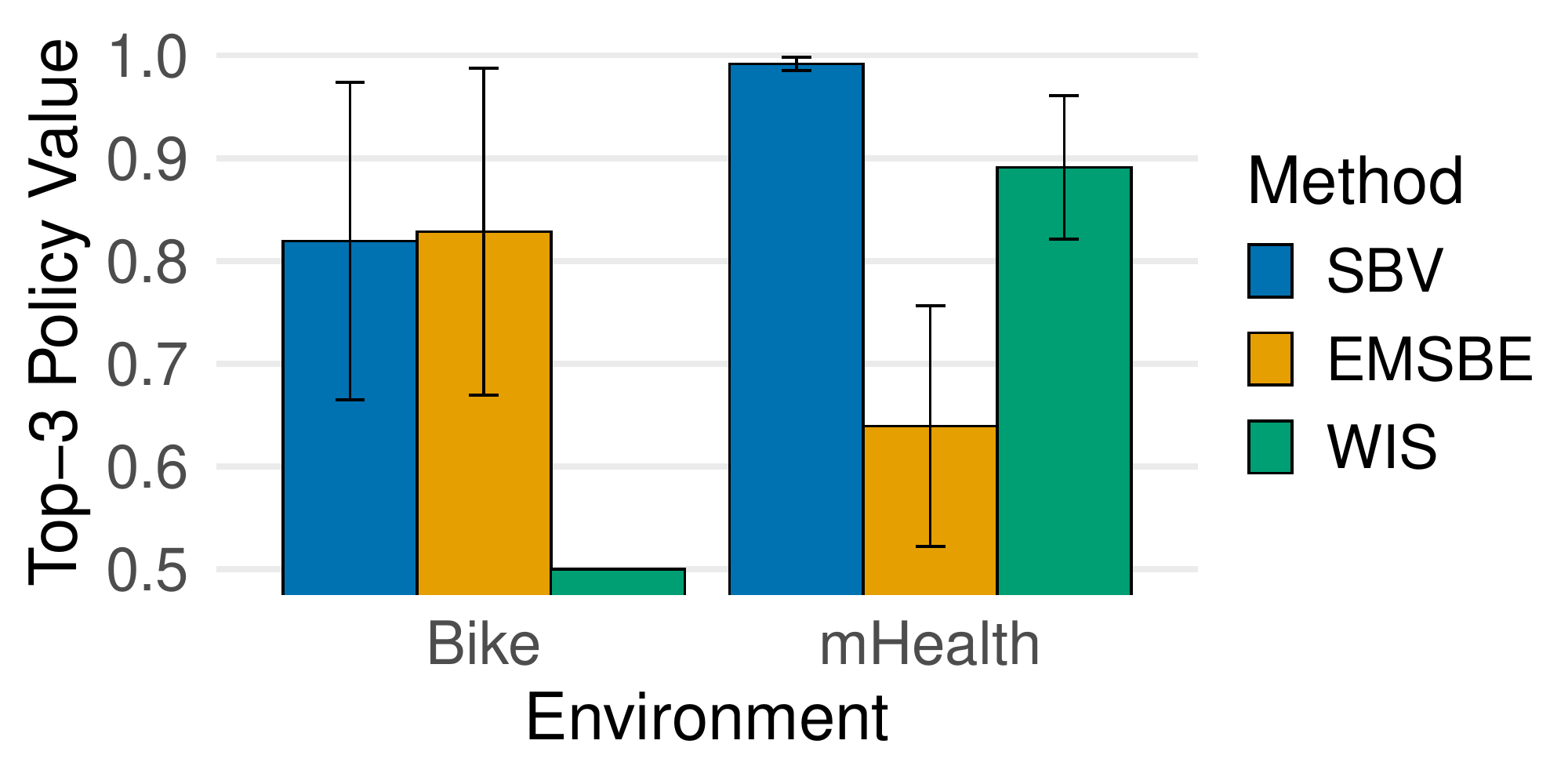}
\centering
\caption{Mean Top-3 Policy Values. For each dataset and method, we calculated the mean policy value of the top-3 policies and standardized to $[0,1]$. Solid bars show the mean and error bars show the std of this metric across datasets. Only SBV performs well on both environments.}
\label{figure:bike}
\vspace{-0.4cm}
\end{figure}

The estimation variance of WIS makes it difficult to account for long-term consequences of actions, as well as rewards that occur far away from the initial state. For the Bicycle datasets, a non-zero reward is usually only observed after nearly 100 time steps. As a result, WIS gives an identical estimate of zero for almost all policies (see Appendices \ref{append:baselines} and \ref{append:bike} for more details). On the other hand, the EMSBE performs well as there is only a small amount of noise in the MDP. For the mHealth datasets, rewards are dense and long-term consequences of actions are less important, but the MDP is noisier. Therefore, WIS performs much better, while EMSBE performs much worse. Only SBV performs well on both environments. 

\begin{table*}[b]
\vspace{-0.1cm}
\centering
\captionsetup{margin=1cm}
\caption{Standardized Top-5 Policy Mean Returns. A mean return of 0\% (100\%) for a dataset implies that the method choose the worst (best) five policies possible on the given dataset. Due to its computational demands, we only report mean returns of FQE for one dataset per game. For other methods, we report the average and range of these mean returns across three datasets per game.
}\label{table:atari}
\begin{tabular}{lllll}
 \toprule
Method& Pong & Breakout & Asterix & Seaquest \\
 \midrule
SBV (Ours) & \textbf{95\% (93-98\%)} & \textbf{81\% (73-90\%)} & \textbf{69\% (62-74\%)} & \textbf{65\% (60-71\%)} \\
EMSBE (Equation \ref{eq:emsbe}) & 87\% (77-98\%) & 64\% (43-77\%) & 60\% (51-67\%) & 47\% (44-52\%)\\
WIS \cite{Precup2000}  & 66\% (45-90\%) & 37\% (34-39\%) & 43\% (37-55\%)& 24\% (13-34\%) \\
FQE \cite{Le2019} & 98\% & 41\% & 53\% & 34\% \\
 \bottomrule
\end{tabular}
\vspace{-0.1cm}
\end{table*}

As SBV only requires a regression algorithm, its hyperparameters can be tuned offline using validation MSE. In contrast, FQE requires an offline RL training algorithm to estimate the action-value function, and tuning this algorithm's hyperparameters offline is not nearly as straightforward. This makes it difficult to compare FQE to competitors, as its performance will depend on the arbitrary choice of what algorithm we use to estimate the action-value function. For example, in Table \ref{table:fqe}, we find that FQE performance varies greatly with the algorithm utilized for estimating the action-value function. We can also see that FQE is biased towards $Q^*$ estimation algorithms similar to its own training algorithm (FQE choose its own training algorithm as the best training algorithm for estimating $Q^*$ in three out of four cases). More details about these training algorithms can be found in Appendix \ref{append:non-atari}. While FQE does perform well with the right training algorithm, we would not need OMS in the first place if we knew in advance which RL training algorithm performed best.

Following previous work \cite{Paine2020, Fu2021}, we also compared performance based on Spearman correlation in Figure \ref{figure:spearman}, and based on max top-$k$ policy value for varying values of $k$ in Figure \ref{figure:bike_oracle}. We chose to focus on mean top-3 policy value here instead of top-1 policy value as the former relies on more than a single Q-function, thus providing a more stable and robust measure of performance. In this case, however, looking at top-1 policy values instead yields similar conclusions (see Figure \ref{figure:bike_oracle}).

\vspace{-0.05cm}
\subsection{High-Dimensional Atari Environments}
\label{sec:atari}
\vspace{-0.05cm}

Finally, we evaluated SBV (Algorithm \ref{alg:sbv}) on 12 offline DQN-Replay datasets \cite{Agarwal2020}, corresponding to three seeds each for four Atari games: Pong, Breakout, Asterix and Seaquest. Atari games have high-dimensional state spaces, making them more challenging than previous environments evaluated so far. We chose to focus on these four games in particular as they have received more attention in recent literature \citep{Kumar2020, kumar2021implicit}.  The performance of DQN is also sensitive to the number of training iterations for most of these games, making OMS more challenging. As in Section \ref{sec:Bike}, we also evaluated validation EMSBE, WIS and FQE.

Following \citet{Agarwal2020}, we performed uniform sub-sampling to obtain 12 training and validation datasets with 10M and 2.5M transitions each, respectively. We used two training configurations for DQN: a \textit{shallow} configuration that uses the ``DQN (Adam)" setup from \citet{Agarwal2020}, and a \textit{deep} configuration that uses a deeper architecture, a slower target update frequency and double Q-learning targets \cite{Hasselt2016}. For each training configuration, we ran DQN for 50 iterations (one iteration = 640k gradient steps) and evaluated the Q-network after each iteration. This resulted in evaluating 100 Q-functions for each Atari dataset. 

Unlike in previous experiments, the same Bellman network training configuration was used by SBV to estimate most Bellman backup functions\footnote{For Pong datasets, we used a simpler Bellman network and only evaluated the shallow Q-networks to speed-up experiments.}, and was tuned offline so as to minimize validation error across Bellman backups and datasets. The Bellman network (Section \ref{sec:method}) incorporates prevalent design choices for image classification such as batch normalization \cite{Ioffe2015}, skip connections \cite{He2016} and squeeze-and-excitation units \cite{Chollet2017}. While the behavioral policy was known in previous datasets, it is unknown for our Atari datasets. Thus, we estimated it with behavioral cloning \cite{Osa2018} using a similar training configuration as that of the Bellman network prior to running WIS. Due to the computational cost of FQE (Appendix \ref{append:wis_atari}) and the difficulty of tuning its hyperparameters offline, we only applied FQE to a single dataset per game using the same Q-network architecture and target update frequency as \citet{Mnih2015}.  See Appendix \ref{append:atari} for more details on our experimental setup. 

\begin{figure*}[t]
\vspace{-0.1cm}
\includegraphics[width=\linewidth]{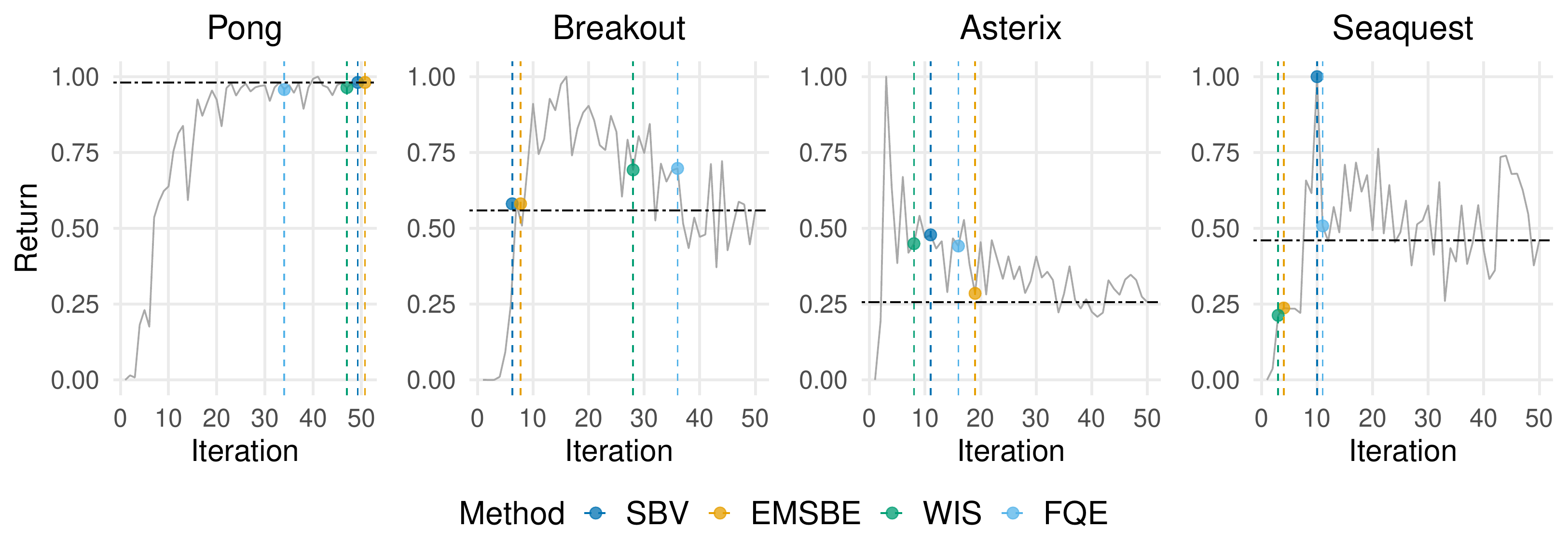}
\centering
\caption{Learning Curves from the Best Configuration for Four Datasets. Returns are standardized to $[0,1]$. The dashed horizontal line represents performance with no early stopping. The vertical lines represent the iterations where training was stopped according to different methods. SBV and FQE performs as well as or better than no early stopping for all games. However, the same cannot be said for WIS and the EMSBE.}
\label{figure:atari}
\vspace{-0.2cm}
\end{figure*}

From Tables \ref{table:atari} and \ref{table:atari_fqe}, we can see that SBV performed comparable to or better than competing methods on every environment with respect to its top-5 selected policies. We also evaluated the ability of each method to perform early stopping in Figures \ref{figure:atari} and \ref{figure:atari_extra}, assuming the optimal training configuration was used. SBV performs as well as or better than no early stopping for all datasets, while the same cannot be said for WIS and the EMSBE. This suggests that SBV is a more robust early stopping procedure. While FQE was more effective in tuning the number of iterations than WIS and EMSBE, it usually ranked shallow Q-functions as superior to deep Q-functions, even though the best-performing Q-functions for Breakout, Asterix and Seaquest were from the deep configuration. This is why overall performance for FQE was poor for these games (see Table \ref{table:atari}).

Compared to Section \ref{sec:Bike} where we only looked at the top-3 policies, we looked at the top-5 policies here as the total number of Q-functions being evaluated was much higher. We also compared performance based on max top-$k$ policy values in Figure \ref{figure:atari_oracle} and obtained similar conclusions. The tricks we employed to speed-up computations involving SBV hindered us from calculating Spearman correlations with policy returns (see Appendix \ref{append:sbv_atari}), though as discussed in Section \ref{sec:toy}, this metric is not critical for OMS anyway.

\vspace{-0.05cm}
\subsection{Ablation Experiments}
\label{sec:ablations}
\vspace{-0.05cm}
Recall that SBV uses the same dataset $\mathcal D_T$ to both estimate the Q-functions $Q_1,...,Q_M$ via offline RL and estimate their Bellman backups $\mathcal B^*Q_1,...,\mathcal B^*Q_M$ via regression. An alternative strategy was proposed by BeRMin~\citep{Farahmand2010} to reduce estimation bias of the Bellman backup estimators. In this alternative strategy, we further partition $\mathcal D_T$ into two training sets $\mathcal D_{T_1}$ and $\mathcal D_{T_2}$, generate Q-functions by running offline RL on $\mathcal D_{T_1}$ and estimate their Bellman backups by running regression on $\mathcal D_{T_2}$.

When using separate partitions for estimating the Q-functions and their Bellman backups, we expect no more than 50\% of the data reserved for Bellman backup estimation, with the rest used for estimating the Q-functions. Thus, we investigated whether training the Bellman network on a dataset independent to the Q-functions and of 50\% size achieves better performance than SBV's trained Bellman network. From Figures \ref{figure:part} and \ref{figure:part_breakout}, we see that SBV consistently yields lower validation error of the Bellman network. While using the same data to estimate both the Q-function and its Bellman backup may increase bias of the estimated backup, this is offset by a reduction in variance from using more data. Moreover, Figure \ref{figure:part} simplifies the comparison by assuming each partitioning scheme generates the same Q-functions. In practice, requiring separate partitions for the Q-functions and Bellman backups will also mean less data for training the Q-functions, which means Q-functions will perform worse as well.

\begin{figure}[!b]
\vspace{-0.4cm}
\includegraphics[width=\linewidth]{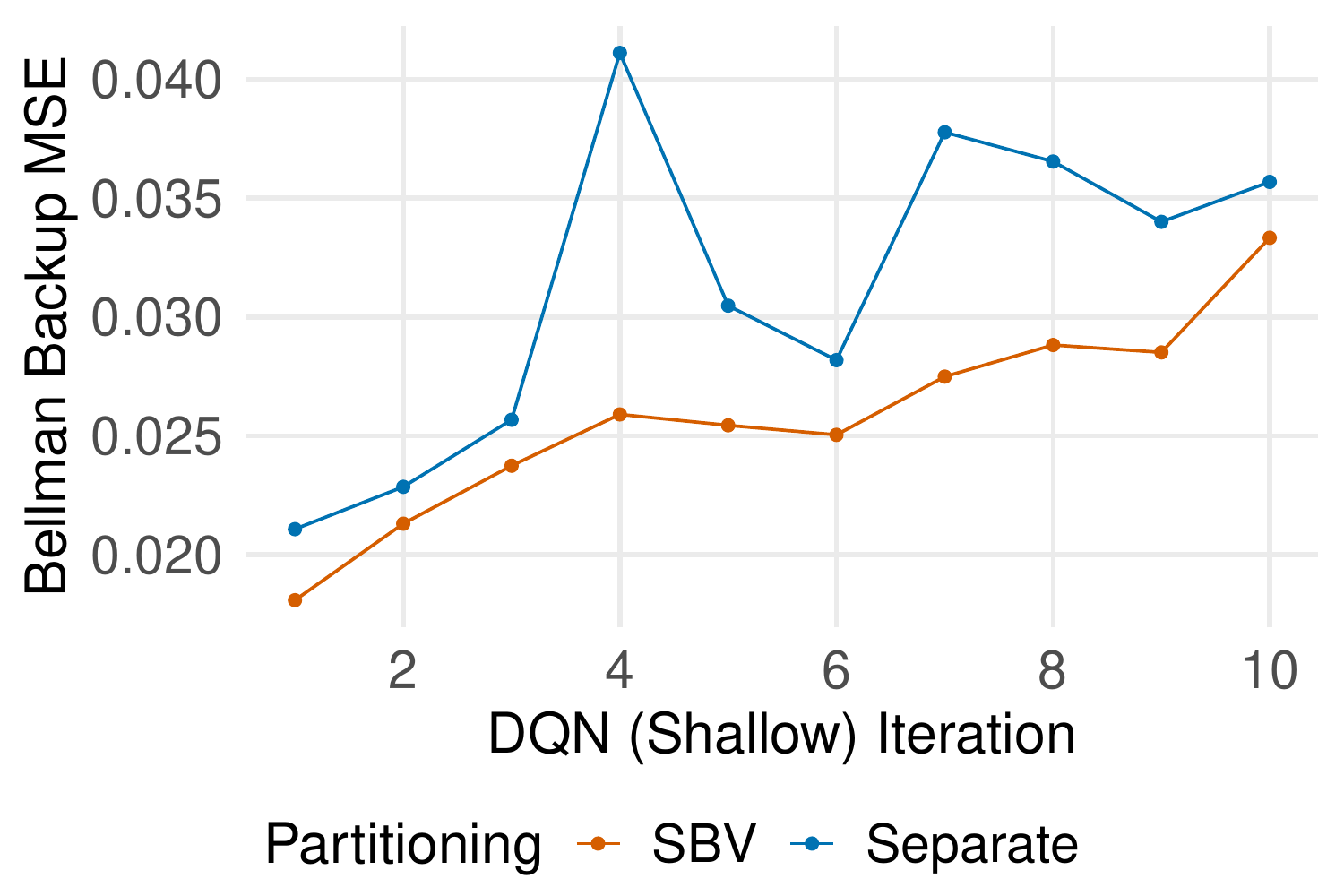}
\centering
\caption{Partitioning Ablation on Seaquest. For the first 10 Q-functions generated by shallow DQN, we estimate their Bellman backups using the same dataset as that used to run DQN and plot their validation MSE in red. We then do the same for a separate dataset of 50\% size in blue.}
\label{figure:part}
\vspace{-0.15cm}
\end{figure}

\begin{table*}[t]
\vspace{-0.1cm}
\centering
\captionsetup{margin=1cm}
\caption{Relationship between Bellman Network and Q-Network Performance. We evaluated policy performance when running DQN using different Q-network architectures, and evaluated validation MSE when using the same architectures as Bellman networks. Architectures that perform better as Bellman networks also perform better as Q-Networks.}\label{table:arch}
\begin{tabular}{lllll}
 \toprule
Architecture& Q-Network Return & Bellman Network MSE  \\
 \midrule
DQN Nature \cite{Mnih2015} & 7169 & 0.097 \\ 
IMAPLA-deep \cite{Espeholt2018} & 8691 & 0.089 \\
Moderate Depth & 18337 & 0.079 \\
Deep (Appendix \ref{append:dqn_deep}) & 22573 & 0.065 \\
REM, 5x data \cite{Agarwal2020} & $8795$ \\
 \bottomrule
\end{tabular}
\vspace{-0.1cm}
\end{table*}

As discussed in Sections \ref{sec:theory} and \ref{sec:toy}, the MSBE may not be effective for OMS unless the candidate set includes Q-functions with small MSBE. This was also observed on Atari. For example, Q-networks trained by the deep configuration yield low Bellman error when early stopping is applied, while Q-networks trained by the shallow configuration have large Bellman error at every iteration. As a result, SBV performs suboptimally when only shallow Q-functions are in the candidate set, selecting Q-functions much worse than the best shallow Q-functions. We can ensure that our candidate set contains Q-functions with small MSBE by exploring a large number of RL training configurations. However, evaluating many training configurations with SBV is computationally demanding, especially if the configurations are also sensitive to the number of training steps. In contrast, by using architectures for the Q-network that performed well as Bellman networks, and by reducing target update frequency, we were able to get Q-functions with low Bellman error without having to explore many RL configurations. 

To showcase this, we consider four Q-network architectures: the original Nature DQN architecture \citep{Mnih2015}, IMPALA-deep \citep{Espeholt2018}, the architecture used by our deep DQN configuration (Appendix \ref{append:dqn_deep})  and a moderately deep architecture that is in between the Nature and deep architectures. For each Q-network architecture, we ran DQN on one of the Seaquest datasets, with other training hyperparameters fixed from the deep DQN configuration and online evaluations used to apply early stopping. These architectures were then used to estimate the Bellman backup of one of the shallow DQN Q-functions via regression. From Table \ref{table:arch}, we can see that better Bellman network architectures also perform better as Q-networks. See Appendix \ref{append:Bellman_q} for further discussion. We also applied a more modern training algorithm, Random Ensemble Mixtures \cite{Agarwal2020}, using 5x as much data, and performance was still worse than DQN with deeper architectures. 

Recent offline RL literature has focused almost exclusively on improving the training algorithm \citep{Prudencio2023}, with the Q-network architecture and other training hyperparameters fixed. However, our results add to a growing body of work suggesting that improving the architecture and other hyperparameters may be quite important \cite{Wu2019, Kumar2021}. These results also suggest that SBV may be useful in developing better Q-network architectures, in addition to performing OMS.

\vspace{-0.1cm}
\section{Discussion and Future Work}
\label{sec:conclusions}
\vspace{-0.1cm}

In this work, we proposed a new algorithm based on the MSBE that was effective at selecting high-performing policies across diverse offline datasets, from small simulated clinical trials to large-scale Atari datasets. Tuning SBV's regression algorithm to minimize validation MSE was critical to achieving robust performance, as it allowed SBV to choose different regressors (e.g. linear models, trees, neural networks) based on what was ideal for a given offline dataset. In addition to demonstrating the potential utility of our proposed algorithm and of the MSBE more generally, we also investigated which factors were most predictive of Bellman error performance and developed guidelines on how to improve this performance in practice. These guidelines allowed us to develop a new Q-network architecture that achieves state-of-the-art performance on some of the Atari datasets. Overall, we believe that our paper challenges current beliefs and will help shape future research in OMS. 

Despite its achievements, our work still has a few limitations we hope will be addressed in future work. First, while implementing SBV on the non-Atari datasets took under 10 minutes, running SBV on Atari took almost one week per dataset with six A100 GPUs. Reducing this computational load would be very helpful. We should point out that while the EMSBE exhibits bias in stochastic environments, it can be computed much faster. In Appendix \ref{append:compute} we compare computational performance between SBV and competing algorithms and discuss when the EMSBE might be preferred over SBV. Second, theoretical guarantees of the MSBE require the observed data to adequately covers the state-action space. While guarantees of many OPE methods require similar assumptions \citep{Janner2019, Le2019, Xie2019}, extending the MSBE to have better guarantees in the face of partial coverage \cite{Uehara2022} could yield a more practical algorithm for narrow or biased datasets \cite{Fu2020}. Third, we have assumed throughout that the candidate set contains Q-functions with low MSBE, and extending SBV to perform well when this condition is violated would make it more applicable in settings where estimating an accurate Q-function is difficult. 

Finally, SBV cannot currently tune actor-critic or policy gradient algorithms for reasons discussed in Appendix \ref{append:cts}. Using SBV to tune FQE and then using FQE to select the policy could overcome this limitation. We also expect the MSBE to more closely correlate with FQE performance than with DQN performance as estimation accuracy is of direct interest with FQE. The main challenge would be combining SBV and FQE without making computation prohibitive.

\vspace{-0.2cm}
\section*{Acknowledgement}
\vspace{-0.1cm}

We thank Google Cloud for $\$20,000$ worth of GCP credits as well as Cameron Voloshin, George Tucker, Bo Dai and anonymous ICML reviewers for their review of the paper.

\bibliography{SBV}

\begin{thebibliography}{85}
\providecommand{\natexlab}[1]{#1}
\providecommand{\url}[1]{\texttt{#1}}
\expandafter\ifx\csname urlstyle\endcsname\relax
  \providecommand{\doi}[1]{doi: #1}\else
  \providecommand{\doi}{doi: \begingroup \urlstyle{rm}\Url}\fi

\bibitem[Abadi et~al.(2015)Abadi, Agarwal, Barham, Brevdo, Chen, Citro,
  Corrado, Davis, Dean, Devin, Ghemawat, Goodfellow, Harp, Irving, Isard, Jia,
  Jozefowicz, Kaiser, Kudlur, Levenberg, Man\'{e}, Monga, Moore, Murray, Olah,
  Schuster, Shlens, Steiner, Sutskever, Talwar, Tucker, Vanhoucke, Vasudevan,
  Vi\'{e}gas, Vinyals, Warden, Wattenberg, Wicke, Yu, and Zheng]{tf2015}
Abadi, M., Agarwal, A., Barham, P., Brevdo, E., Chen, Z., Citro, C., Corrado,
  G.~S., Davis, A., Dean, J., Devin, M., Ghemawat, S., Goodfellow, I., Harp,
  A., Irving, G., Isard, M., Jia, Y., Jozefowicz, R., Kaiser, L., Kudlur, M.,
  Levenberg, J., Man\'{e}, D., Monga, R., Moore, S., Murray, D., Olah, C.,
  Schuster, M., Shlens, J., Steiner, B., Sutskever, I., Talwar, K., Tucker, P.,
  Vanhoucke, V., Vasudevan, V., Vi\'{e}gas, F., Vinyals, O., Warden, P.,
  Wattenberg, M., Wicke, M., Yu, Y., and Zheng, X.
\newblock {TensorFlow}: Large-scale machine learning on heterogeneous systems,
  2015.
\newblock URL \url{https://www.tensorflow.org/}.
\newblock Software available from tensorflow.org.

\bibitem[Agarwal et~al.(2022)Agarwal, Jiang, Kakade, and Sun]{Agarwal2022}
Agarwal, A., Jiang, N., Kakade, S.~M., and Sun, W.
\newblock Reinforcement learning: Theory and algorithms, 2022.
\newblock URL \url{https://rltheorybook.github.io/}.
\newblock A working draft.

\bibitem[Agarwal et~al.(2020)Agarwal, Schuurmans, and Norouzi]{Agarwal2020}
Agarwal, R., Schuurmans, D., and Norouzi, M.
\newblock An optimistic perspective on offline reinforcement learning.
\newblock In \emph{Proceedings of the 37th International Conference on Machine
  Learning (ICML 2020)}, pp.\  104--114, 2020.
\newblock URL \url{https://proceedings.mlr.press/v119/agarwal20c.html}.

\bibitem[Antos et~al.(2007)Antos, Munos, and Szepesv{\'a}ri]{Antos2007fqi}
Antos, A., Munos, R., and Szepesv{\'a}ri, C.~a.
\newblock Fitted {Q}-iteration in continuous action-space {MDP}s.
\newblock In \emph{Advances in Neural Information Processing Systems (NeurIPS
  2007)}, volume~20, pp.\  9--16, 2007.
\newblock URL
  \url{https://papers.nips.cc/paper_files/paper/2007/file/da0d1111d2dc5d489242e60ebcbaf988-Paper.pdf}.

\bibitem[Arulkumaran et~al.(2017)Arulkumaran, Deisenroth, Brundage, and
  Bharath]{Arulkumaran2017}
Arulkumaran, K., Deisenroth, M.~P., Brundage, M., and Bharath, A.~A.
\newblock Deep reinforcement learning: A brief survey.
\newblock \emph{IEEE Signal Processing Magazine}, 34\penalty0 (6):\penalty0
  26--38, 2017.
\newblock \doi{10.1109/MSP.2017.2743240}.

\bibitem[Baird(1995)]{Baird1995}
Baird, L.
\newblock Residual algorithms: Reinforcement learning with function
  approximation.
\newblock In \emph{Proceedings of the 12th International Conference on Machine
  Learning (ICML 1995)}, pp.\  30--37, 1995.
\newblock \doi{10.1016/B978-1-55860-377-6.50013-X}.

\bibitem[Bellemare et~al.(2013)Bellemare, Naddaf, Veness, and
  Bowling]{Bellemare2013}
Bellemare, M.~G., Naddaf, Y., Veness, J., and Bowling, M.
\newblock The arcade learning environment: An evaluation platform for general
  agents.
\newblock \emph{Journal of Artificial Intelligence Research}, 47:\penalty0
  253--279, 2013.
\newblock \doi{10.1613/jair.3912}.

\bibitem[Busoniu et~al.(2010)Busoniu, Babuska, Schutter, and
  Ernst]{Busoniu2010}
Busoniu, L., Babuska, R., Schutter, B.~D., and Ernst, D.
\newblock \emph{Reinforcement Learning and Dynamic Programming Using Function
  Approximators}.
\newblock CRC Press, 2010.
\newblock ISBN 9781439821084.

\bibitem[Castro et~al.(2018)Castro, Moitra, Gelada, Kumar, and
  Bellemare]{dopamine}
Castro, P.~S., Moitra, S., Gelada, C., Kumar, S., and Bellemare, M.~G.
\newblock Dopamine: A research framework for deep reinforcement learning.
\newblock \emph{arXiv preprint arXiv:1812.06110v1}, 2018.

\bibitem[Chen \& Jiang(2019)Chen and Jiang]{Chen2019}
Chen, J. and Jiang, N.
\newblock Information-theoretic considerations in batch reinforcement learning.
\newblock In \emph{Proceedings of the 36th International Conference on Machine
  Learning (ICML 2019)}, pp.\  1042--1051, 2019.
\newblock URL \url{https://proceedings.mlr.press/v97/chen19e.html}.

\bibitem[Chen \& Jiang(2022)Chen and Jiang]{Chen2022}
Chen, J. and Jiang, N.
\newblock On well-posedness and minimax optimal rates of nonparametric
  {Q}-function estimation in off-policy evaluation.
\newblock In \emph{Proceedings of the 39th International Conference on Machine
  Learning (ICML 2022)}, pp.\  3558--3582, 2022.
\newblock URL \url{https://proceedings.mlr.press/v162/chen22u.html}.

\bibitem[Chollet(2017)]{Chollet2017}
Chollet, F.
\newblock Xception: Deep learning with depthwise separable convolutions.
\newblock In \emph{Proceedings of the IEEE Conference on Computer Vision and
  Pattern Recognition (CVPR 2017)}, 2017.
\newblock \doi{10.1109/CVPR.2017.195}.

\bibitem[Dabney et~al.(2018)Dabney, Rowland, Bellemare, and Munos]{Dabney2018}
Dabney, W., Rowland, M., Bellemare, M., and Munos, R.
\newblock Distributional reinforcement learning with quantile regression.
\newblock In \emph{Proceedings of the 32nd AAAI Conference on Artificial
  Intelligence (AAAI-18)}, 2018.
\newblock \doi{10.1609/aaai.v32i1.11791}.

\bibitem[Dozat(2016)]{Dozat2016}
Dozat, T.
\newblock Incorporating {N}esterov momentum into {A}dam.
\newblock In \emph{The 4th International Conference on Learning Representations
  (ICLR 2016)}, 2016.
\newblock URL \url{https://openreview.net/pdf/OM0jvwB8jIp57ZJjtNEZ.pdf}.

\bibitem[Ernst et~al.(2005)Ernst, Geurts, and Wehenkel]{Ernst2005}
Ernst, D., Geurts, P., and Wehenkel, L.
\newblock Tree-based batch mode reinforcement learning.
\newblock \emph{Journal of Machine Learning Research}, 6:\penalty0 503--556,
  2005.
\newblock URL \url{http://jmlr.org/papers/v6/ernst05a.html}.

\bibitem[Espeholt et~al.(2018)Espeholt, Soyer, Munos, Simonyan, Mnih, Ward,
  Doron, Firoiu, Harley, Dunning, Legg, and Kavukcuoglu]{Espeholt2018}
Espeholt, L., Soyer, H., Munos, R., Simonyan, K., Mnih, V., Ward, T., Doron,
  Y., Firoiu, V., Harley, T., Dunning, I., Legg, S., and Kavukcuoglu, K.
\newblock {IMPALA}: Scalable distributed deep-{RL} with importance weighted
  actor-learner architectures.
\newblock In \emph{Proceedings of the 35th International Conference on Machine
  Learning (ICML 2018)}, pp.\  1407--1416, 2018.
\newblock URL \url{http://proceedings.mlr.press/v80/espeholt18a.html}.

\bibitem[Farahmand \& Szepesvari(2010)Farahmand and Szepesvari]{Farahmand2010}
Farahmand, A.~M. and Szepesvari, C.
\newblock Model selection in reinforcement learning.
\newblock \emph{Machine Learning}, 85:\penalty0 299--332, 2010.
\newblock \doi{10.1007/s10994-011-5254-7}.

\bibitem[Fu et~al.(2020)Fu, Kumar, Nachum, Tucker, and Levine]{Fu2020}
Fu, J., Kumar, A., Nachum, O., Tucker, G., and Levine, S.
\newblock {D4RL}: Datasets for deep data-driven reinforcement learning.
\newblock \emph{arXiv preprint arXiv:2004.07219}, 2020.

\bibitem[Fu et~al.(2021)Fu, Norouzi, Nachum, Tucker, Wang, Novikov, Yang,
  Zhang, Chen, Kumar, Paduraru, Levine, and Paine]{Fu2021}
Fu, J., Norouzi, M., Nachum, O., Tucker, G., Wang, Z., Novikov, A., Yang, M.,
  Zhang, M.~R., Chen, Y., Kumar, A., Paduraru, C., Levine, S., and Paine, T.~L.
\newblock Benchmarks for deep off-policy evaluation.
\newblock In \emph{The 9th International Conference on Learning Representations
  (ICLR 2021)}, 2021.
\newblock URL \url{https://openreview.net/pdf?id=kWSeGEeHvF8}.

\bibitem[Fujimoto et~al.(2022)Fujimoto, Meger, Precup, Nachum, and
  Gu]{Fujimoto2022}
Fujimoto, S., Meger, D., Precup, D., Nachum, O., and Gu, S.~S.
\newblock Why should {I} trust you, {B}ellman? the {B}ellman error is a poor
  replacement for value error.
\newblock \emph{arXiv preprint arXiv:2201.12417}, 2022.

\bibitem[G{\'e}ron(2019)]{Geron2019}
G{\'e}ron, A.
\newblock \emph{Hands-On Machine Learning with {S}cikit-Learn, {K}eras and
  {T}ensor{F}low}.
\newblock O’Reilly, second edition, 2019.
\newblock ISBN 9781492032649.

\bibitem[Glorot \& Bengio(2010)Glorot and Bengio]{Glorot2010}
Glorot, X. and Bengio, Y.
\newblock Understanding the difficulty of training deep feedforward neural
  networks.
\newblock In \emph{Proceedings of the Thirteenth International Conference on
  Artificial Intelligence and Statistics (AISTATS 2010)}, pp.\  249--256, 2010.
\newblock URL \url{https://proceedings.mlr.press/v9/glorot10a.html}.

\bibitem[Gulcehre et~al.(2020)Gulcehre, Wang, Novikov, Paine, G{\'o}mez, Zolna,
  Agarwal, Merel, Mankowitz, Paduraru, et~al.]{gulcehre2020}
Gulcehre, C., Wang, Z., Novikov, A., Paine, T., G{\'o}mez, S., Zolna, K.,
  Agarwal, R., Merel, J.~S., Mankowitz, D.~J., Paduraru, C., et~al.
\newblock {RL} {U}nplugged: A suite of benchmarks for offline reinforcement
  learning.
\newblock In \emph{Advances in Neural Information Processing Systems (NeurIPS
  2020)}, volume~33, pp.\  7248--7259, 2020.
\newblock URL
  \url{https://proceedings.neurips.cc/paper_files/paper/2020/file/51200d29d1fc15f5a71c1dab4bb54f7c-Paper.pdf}.

\bibitem[Hasselt et~al.(2016)Hasselt, Guez, and Silver]{Hasselt2016}
Hasselt, H.~v., Guez, A., and Silver, D.
\newblock Deep reinforcement learning with double {Q}-learning.
\newblock In \emph{Proceedings of the 30th AAAI Conference on Artificial
  Intelligence (AAAI-2016)}, pp.\  2094–2100, 2016.
\newblock \doi{10.1609/aaai.v30i1.10295}.

\bibitem[Hastie et~al.(2009)Hastie, Tibshirani, and Friedman]{Hastie2009}
Hastie, T., Tibshirani, R., and Friedman, J.
\newblock \emph{The Elements of Statistical Learning}.
\newblock Springer, 2009.
\newblock \doi{10.1007/978-0-387-84858-7}.

\bibitem[He et~al.(2015)He, Zhang, Ren, and Sun]{He2015}
He, K., Zhang, X., Ren, S., and Sun, J.
\newblock Delving deep into rectifiers: Surpassing human-level performance on
  {I}mage{N}et classification.
\newblock In \emph{Proceedings of the IEEE International Conference on Computer
  Vision (ICCV 2015)}, pp.\  1026--1034, 2015.
\newblock \doi{10.1109/ICCV.2015.123}.

\bibitem[He et~al.(2016)He, Zhang, Ren, and Sun]{He2016}
He, K., Zhang, X., Ren, S., and Sun, J.
\newblock Deep residual learning for image recognition.
\newblock In \emph{Proceedings of the IEEE Conference on Computer Vision and
  Pattern Recognition (CVPR 2016)}, 2016.
\newblock \doi{10.1109/CVPR.2016.90}.

\bibitem[Hu et~al.(2018)Hu, Shen, and Sun]{Hu2018}
Hu, J., Shen, L., and Sun, G.
\newblock {S}queeze-and-{E}xcitation networks.
\newblock In \emph{2018 IEEE/CVF Conference on Computer Vision and Pattern
  Recognition}, pp.\  7132--7141, 2018.
\newblock \doi{10.1109/CVPR.2018.00745}.

\bibitem[Ioffe \& Szegedy(2015)Ioffe and Szegedy]{Ioffe2015}
Ioffe, S. and Szegedy, C.
\newblock {B}atch {N}ormalization: Accelerating deep network training by
  reducing internal covariate shift.
\newblock In \emph{Proceedings of the 32nd International Conference on Machine
  Learning (ICML 2015)}, pp.\  448--456, 2015.
\newblock URL \url{http://proceedings.mlr.press/v37/ioffe15.html}.

\bibitem[Irpan et~al.(2019)Irpan, Rao, Bousmalis, Harris, Ibarz, and
  Levine]{Irpan2019}
Irpan, A., Rao, K., Bousmalis, K., Harris, C., Ibarz, J., and Levine, S.
\newblock Off-policy evaluation via off-policy classification.
\newblock In \emph{Advances in Neural Information Processing Systems (NeurIPS
  2019)}, volume~32, 2019.
\newblock URL
  \url{https://proceedings.neurips.cc/paper/2019/file/b5b03f06271f8917685d14cea7c6c50a-Paper.pdf}.

\bibitem[Janner et~al.(2019)Janner, Fu, Zhang, and Levine]{Janner2019}
Janner, M., Fu, J., Zhang, M., and Levine, S.
\newblock When to trust your model: Model-based policy optimization.
\newblock \emph{Advances in Neural Information Processing Systems
  (NeurIPS2019)}, 32, 2019.
\newblock URL
  \url{https://proceedings.neurips.cc/paper_files/paper/2019/file/5faf461eff3099671ad63c6f3f094f7f-Paper.pdf}.

\bibitem[Jiang \& Li(2016)Jiang and Li]{Jiang2016}
Jiang, N. and Li, L.
\newblock Doubly robust off-policy value evaluation for reinforcement learning.
\newblock In \emph{Proceedings of the 33rd International Conference on Machine
  Learning (ICML 2016)}, pp.\  652--661, 2016.
\newblock URL \url{http://proceedings.mlr.press/v48/jiang16.html}.

\bibitem[Kahn et~al.(2021)Kahn, Abbeel, and Levine]{Khan2021}
Kahn, G., Abbeel, P., and Levine, S.
\newblock {BADGR}: An autonomous self-supervised learning-based navigation
  system.
\newblock \emph{IEEE Robotics and Automation Letters}, 6\penalty0 (2):\penalty0
  1312--1319, 2021.
\newblock \doi{10.1109/LRA.2021.3057023}.

\bibitem[Kalashnikov et~al.(2018)Kalashnikov, Irpan, Pastor, Ibarz, Herzog,
  Jang, Quillen, Holly, Kalakrishnan, Vanhoucke, et~al.]{Kalashnikov2018}
Kalashnikov, D., Irpan, A., Pastor, P., Ibarz, J., Herzog, A., Jang, E.,
  Quillen, D., Holly, E., Kalakrishnan, M., Vanhoucke, V., et~al.
\newblock {QT}-{O}pt: Scalable deep reinforcement learning for vision-based
  robotic manipulation.
\newblock In \emph{2nd Conference on Robot Learning (CoRL 2018)}, 2018.
\newblock \doi{10.48550/arXiv.1806.10293}.

\bibitem[Kazdin et~al.(2000)Kazdin, Association, et~al.]{Kazdin2000}
Kazdin, A.~E., Association, A.~P., et~al.
\newblock \emph{Encyclopedia of Psychology}, volume~8.
\newblock American Psychological Association Washington, DC, 2000.
\newblock ISBN 9781557981875.

\bibitem[Klasnja et~al.(2015)Klasnja, Hekler, Shiffman, Boruvka, Almirall,
  Tewari, and Murphy]{Klasnja2015}
Klasnja, P., Hekler, E.~B., Shiffman, S., Boruvka, A., Almirall, D., Tewari,
  A., and Murphy, S.~A.
\newblock Microrandomized trials: An experimental design for developing
  just-in-time adaptive interventions.
\newblock \emph{Health Psychology}, 34\penalty0 (S):\penalty0 1220, 2015.
\newblock \doi{10.1037/hea0000305}.

\bibitem[Kostrikov \& Nachum(2020)Kostrikov and Nachum]{Kostrikov2020}
Kostrikov, I. and Nachum, O.
\newblock Statistical bootstrapping for uncertainty estimation in off-policy
  evaluation.
\newblock \emph{arXiv preprint arXiv:2007.13609v1}, 2020.

\bibitem[Kotz et~al.(2006)Kotz, Read, Balakrishnan, Vidakovic, and
  Johnson]{Kotz2006}
Kotz, S., Read, C.~B., Balakrishnan, N., Vidakovic, B., and Johnson, N.~L.
\newblock \emph{Encyclopedia of Statistical Sciences}.
\newblock John Wiley {\&} Sons, 2006.
\newblock \doi{10.1002/0471667196}.

\bibitem[Kumar et~al.(2020)Kumar, Zhou, Tucker, and Levine]{Kumar2020}
Kumar, A., Zhou, A., Tucker, G., and Levine, S.
\newblock Conservative {Q}-learning for offline reinforcement learning.
\newblock In \emph{Advances in Neural Information Processing Systems (NeurIPS
  2020)}, volume~33, pp.\  1179--1191, 2020.
\newblock URL
  \url{https://proceedings.neurips.cc/paper/2020/file/0d2b2061826a5df3221116a5085a6052-Paper.pdf}.

\bibitem[Kumar et~al.(2021{\natexlab{a}})Kumar, Agarwal, Ghosh, and
  Levine]{kumar2021implicit}
Kumar, A., Agarwal, R., Ghosh, D., and Levine, S.
\newblock Implicit under-parameterization inhibits data-efficient deep
  reinforcement learning.
\newblock In \emph{International Conference on Learning Representations},
  2021{\natexlab{a}}.
\newblock URL \url{https://openreview.net/forum?id=O9bnihsFfXU}.

\bibitem[Kumar et~al.(2021{\natexlab{b}})Kumar, Singh, Tian, Finn, and
  Levine]{Kumar2021}
Kumar, A., Singh, A., Tian, S., Finn, C., and Levine, S.
\newblock A workflow for offline model-free robotic reinforcement learning.
\newblock In \emph{5th Conference on Robot Learning (CoRL 2021)},
  2021{\natexlab{b}}.
\newblock \doi{10.48550/arXiv.2109.10813}.

\bibitem[Lagoudakis \& Parr(2003)Lagoudakis and Parr]{Lagoudakis2003}
Lagoudakis, M.~G. and Parr, R.~E.
\newblock Least-squares policy iteration.
\newblock \emph{Journal of Machine Learning Research}, 4:\penalty0 1107--1149,
  2003.
\newblock URL \url{https://www.jmlr.org/papers/v4/lagoudakis03a.html}.

\bibitem[Le et~al.(2019)Le, Voloshin, and Yue]{Le2019}
Le, H., Voloshin, C., and Yue, Y.
\newblock Batch policy learning under constraints.
\newblock In \emph{Proceedings of the 36th International Conference on Machine
  Learning (ICML 2019)}, pp.\  3703--3712, 2019.
\newblock URL \url{http://proceedings.mlr.press/v97/le19a.html}.

\bibitem[Lee et~al.(2022)Lee, Tucker, Nachum, Dai, and Brunskill]{Lee2022}
Lee, J.~N., Tucker, G., Nachum, O., Dai, B., and Brunskill, E.
\newblock Oracle inequalities for model selection in offline reinforcement
  learning.
\newblock \emph{arXiv preprint arXiv:2211.02016}, 2022.

\bibitem[Levine et~al.(2020)Levine, Kumar, Tucker, and Fu]{Levine2020}
Levine, S., Kumar, A., Tucker, G., and Fu, J.
\newblock Offline reinforcement learning: Tutorial, review, and perspectives on
  open problems.
\newblock \emph{arXiv preprint arXiv:2005.01643}, 2020.

\bibitem[Liao et~al.(2021)Liao, Klasnja, and Murphy]{Liao2021}
Liao, P., Klasnja, P., and Murphy, S.
\newblock Off-policy estimation of long-term average outcomes with applications
  to mobile health.
\newblock \emph{Journal of the American Statistical Association}, 116\penalty0
  (533):\penalty0 382--391, 2021.
\newblock \doi{10.1080/01621459.2020.1807993}.

\bibitem[Liao et~al.(2022)Liao, Qi, Wan, Klasnja, and Murphy]{Liao2022}
Liao, P., Qi, Z., Wan, R., Klasnja, P., and Murphy, S.~A.
\newblock Batch policy learning in average reward {M}arkov decision processes.
\newblock \emph{The Annals of Statistics}, 50\penalty0 (6):\penalty0 3364 --
  3387, 2022.
\newblock \doi{10.1214/22-AOS2231}.

\bibitem[Lillicrap et~al.(2016)Lillicrap, Hunt, Pritzel, Heess, Erez, Tassa,
  Silver, and Wierstra]{Lillicrap2016}
Lillicrap, T.~P., Hunt, J.~J., Pritzel, A., Heess, N. M.~O., Erez, T., Tassa,
  Y., Silver, D., and Wierstra, D.
\newblock Continuous control with deep reinforcement learning.
\newblock In \emph{The 4th International Conference on Learning Representations
  (ICLR 2016)}, 2016.
\newblock URL \url{https://arxiv.org/pdf/1509.02971v6.pdf}.

\bibitem[Luckett et~al.(2020)Luckett, Laber, Kahkoska, Maahs, Mayer‐Davis,
  and Kosorok]{Luckett2020}
Luckett, D.~J., Laber, E.~B., Kahkoska, A.~R., Maahs, D.~M., Mayer‐Davis,
  E.~J., and Kosorok, M.~R.
\newblock Estimating dynamic treatment regimes in mobile health using
  {V}-learning.
\newblock \emph{Journal of the American Statistical Association}, 115:\penalty0
  692--706, 2020.
\newblock \doi{10.1080/01621459.2018.1537919}.

\bibitem[Miyaguchi(2022)]{Miyaguchi2022}
Miyaguchi, K.
\newblock Hyperparameter selection methods for fitted {Q}-evaluation with error
  guarantee.
\newblock \emph{arXiv preprint arXiv:2201.02300v2}, 2022.

\bibitem[Mnih et~al.(2015)Mnih, Kavukcuoglu, Silver, Rusu, Veness, Bellemare,
  Graves, Riedmiller, Fidjeland, Ostrovski, Petersen, Beattie, Sadik,
  Antonoglou, King, Kumaran, Wierstra, Legg, and Hassabis]{Mnih2015}
Mnih, V., Kavukcuoglu, K., Silver, D., Rusu, A.~A., Veness, J., Bellemare,
  M.~G., Graves, A., Riedmiller, M.~A., Fidjeland, A., Ostrovski, G., Petersen,
  S., Beattie, C., Sadik, A., Antonoglou, I., King, H., Kumaran, D., Wierstra,
  D., Legg, S., and Hassabis, D.
\newblock Human-level control through deep reinforcement learning.
\newblock \emph{Nature}, 518:\penalty0 529--533, 2015.
\newblock \doi{10.1038/nature14236}.

\bibitem[Munos(2005)]{Munos2005}
Munos, R.
\newblock Error bounds for approximate value iteration.
\newblock In \emph{Proceedings of the 20th National Conference on Artificial
  Intelligence (AAAI-05)}, pp.\  1006–--1011, 2005.
\newblock URL \url{https://www.aaaipress.org/Papers/AAAI/2005/AAAI05-159.pdf}.

\bibitem[Munos \& Szepesv{{\'a}}ri(2008)Munos and Szepesv{{\'a}}ri]{Munos2008}
Munos, R. and Szepesv{{\'a}}ri, C.
\newblock Finite-time bounds for fitted value iteration.
\newblock \emph{Journal of Machine Learning Research}, 9\penalty0
  (27):\penalty0 815--857, 2008.
\newblock URL \url{http://jmlr.org/papers/v9/munos08a.html}.

\bibitem[Nie et~al.(2022)Nie, Flet-Berliac, Jordan, Steenbergen, and
  Brunskill]{Nie2022}
Nie, A., Flet-Berliac, Y., Jordan, D., Steenbergen, W., and Brunskill, E.
\newblock Data-efficient pipeline for offline reinforcement learning with
  limited data.
\newblock In \emph{Advances in Neural Information Processing Systems (NeurIPS
  2022)}, volume~35, pp.\  14810--14823, 2022.

\bibitem[Osa et~al.(2018)Osa, Pajarinen, Neumann, Bagnell, Abbeel, Peters,
  et~al.]{Osa2018}
Osa, T., Pajarinen, J., Neumann, G., Bagnell, J.~A., Abbeel, P., Peters, J.,
  et~al.
\newblock An algorithmic perspective on imitation learning.
\newblock \emph{Foundations and Trends{\textregistered} in Robotics},
  7\penalty0 (1-2):\penalty0 1--179, 2018.
\newblock \doi{10.1561/2300000053}.

\bibitem[Paine et~al.(2020)Paine, Paduraru, Michi, Gulcehre, Zolna, Novikov,
  Wang, and de~Freitas]{Paine2020}
Paine, T.~L., Paduraru, C., Michi, A., Gulcehre, C., Zolna, K., Novikov, A.,
  Wang, Z., and de~Freitas, N.
\newblock Hyperparameter selection for offline reinforcement learning.
\newblock \emph{arXiv preprint arXiv:2007.09055}, 2020.

\bibitem[Precup et~al.(2000)Precup, Sutton, and Singh]{Precup2000}
Precup, D., Sutton, R.~S., and Singh, S.
\newblock Eligibility traces for off-policy policy evaluation.
\newblock In \emph{Proceedings of the 17th International Confrence on Machine
  Learning (ICML 2000)}, pp.\ ~80, 2000.
\newblock URL \url{https://scholarworks.umass.edu/cs_faculty_pubs/80}.

\bibitem[Prudencio et~al.(2023)Prudencio, Maximo, and Colombini]{Prudencio2023}
Prudencio, R.~F., Maximo, M. R. O.~A., and Colombini, E.~L.
\newblock A survey on offline reinforcement learning: Taxonomy, review, and
  open problems.
\newblock \emph{IEEE Transactions on Neural Networks and Learning Systems
  (Early Access)}, pp.\  1--0, 2023.
\newblock \doi{10.1109/TNNLS.2023.3250269}.

\bibitem[Puterman(1994)]{Putterman1994}
Puterman, M.~L.
\newblock \emph{{M}arkov Decision Processes: Discrete Stochastic Dynamic
  Programming}.
\newblock John Wiley \& Sons, 1994.
\newblock \doi{10.1002/9780470316887}.

\bibitem[Rafailov et~al.(2021)Rafailov, Yu, Rajeswaran, and Finn]{Rafailov2021}
Rafailov, R., Yu, T., Rajeswaran, A., and Finn, C.
\newblock Offline reinforcement learning from images with latent space models.
\newblock In \emph{Proceedings of the 3rd Conference on Learning for Dynamics
  and Control (L4DC 2021)}, pp.\  1154--1168, 2021.
\newblock URL \url{https://proceedings.mlr.press/v144/rafailov21a.html}.

\bibitem[Randl\o{}v \& Alstr\o{}m(1998)Randl\o{}v and Alstr\o{}m]{Randlov1998}
Randl\o{}v, J. and Alstr\o{}m, P.
\newblock Learning to drive a bicycle using reinforcement learning and shaping.
\newblock In \emph{Proceedings of the 15th International Conference on Machine
  Learning (ICML 1998)}, pp.\  463–471, 1998.

\bibitem[Russakovsky et~al.(2015)Russakovsky, Deng, Su, Krause, Satheesh, Ma,
  Huang, Karpathy, Khosla, Bernstein, Berg, and Fei-Fei]{ILSVRC15}
Russakovsky, O., Deng, J., Su, H., Krause, J., Satheesh, S., Ma, S., Huang, Z.,
  Karpathy, A., Khosla, A., Bernstein, M., Berg, A.~C., and Fei-Fei, L.
\newblock {ImageNet} large scale visual recognition challenge.
\newblock \emph{International Journal of Computer Vision (IJCV)}, 115\penalty0
  (3):\penalty0 211--252, 2015.
\newblock \doi{10.1007/s11263-015-0816-y}.

\bibitem[Schrittwieser et~al.(2021)Schrittwieser, Hubert, Mandhane, Barekatain,
  Antonoglou, and Silver]{schrittwieser2021}
Schrittwieser, J., Hubert, T., Mandhane, A., Barekatain, M., Antonoglou, I.,
  and Silver, D.
\newblock Online and offline reinforcement learning by planning with a learned
  model.
\newblock \emph{Advances in Neural Information Processing Systems (NeurIPS
  2021)}, 34:\penalty0 27580--27591, 2021.
\newblock URL
  \url{https://proceedings.neurips.cc/paper_files/paper/2021/file/e8258e5140317ff36c7f8225a3bf9590-Paper.pdf}.

\bibitem[Schulman et~al.(2015)Schulman, Levine, Abbeel, Jordan, and
  Moritz]{Schulman2015}
Schulman, J., Levine, S., Abbeel, P., Jordan, M., and Moritz, P.
\newblock Trust region policy optimization.
\newblock In \emph{Proceedings of the 32nd International Conference on Machine
  Learning (ICML 2015)}, pp.\  1889--1897, 2015.
\newblock URL \url{https://proceedings.mlr.press/v37/schulman15.html}.

\bibitem[Shi et~al.(2021)Shi, Zhang, Lu, and Song]{Shi2021}
Shi, C., Zhang, S., Lu, W., and Song, R.
\newblock Statistical inference of the value function for reinforcement
  learning in infinite-horizon settings.
\newblock \emph{Journal of the Royal Statistical Society Series B: Statistical
  Methodology}, 84\penalty0 (3):\penalty0 765--793, 2021.
\newblock \doi{10.1111/rssb.12465}.

\bibitem[Shi et~al.(2022)Shi, Zhu, Ye, Luo, Zhu, and Song]{Shi2022}
Shi, C., Zhu, J., Ye, S., Luo, S., Zhu, H., and Song, R.
\newblock Off-policy confidence interval estimation with confounded markov
  decision process.
\newblock \emph{Journal of the American Statistical Association}, 0\penalty0
  (0):\penalty0 1--12, 2022.
\newblock \doi{10.1080/01621459.2022.2110878}.

\bibitem[Sutton \& Barto(2018)Sutton and Barto]{Sutton2018}
Sutton, R.~S. and Barto, A.~G.
\newblock \emph{Reinforcement Learning: An Introduction}.
\newblock The MIT Press, second edition, 2018.
\newblock ISBN 9780262039246.

\bibitem[Tang \& Wiens(2021)Tang and Wiens]{Tang2021}
Tang, S. and Wiens, J.
\newblock Model selection for offline reinforcement learning: Practical
  considerations for healthcare settings.
\newblock In \emph{Proceedings of the 6th Machine Learning for Healthcare
  Conference (MLHC 2021)}, pp.\  2--35, 2021.
\newblock URL \url{https://proceedings.mlr.press/v149/tang21a.html}.

\bibitem[Thomas \& Brunskill(2016)Thomas and Brunskill]{Thomas2016}
Thomas, P. and Brunskill, E.
\newblock Data-efficient off-policy policy evaluation for reinforcement
  learning.
\newblock In \emph{Proceedings of the 33rd International Conference on Machine
  Learning (ICML 2016)}, pp.\  2139--2148, 2016.
\newblock URL \url{https://proceedings.mlr.press/v48/thomasa16.html}.

\bibitem[Thomas et~al.(2015)Thomas, Theocharous, and Ghavamzadeh]{Thomas2015}
Thomas, P., Theocharous, G., and Ghavamzadeh, M.
\newblock High-confidence off-policy evaluation.
\newblock In \emph{Proceedings of the 29th AAAI Conference on Artificial
  Intelligence (AAAI-15)}, pp.\  2094–2100, 2015.
\newblock \doi{10.1609/aaai.v29i1.9541}.

\bibitem[Tsiatis et~al.(2019)Tsiatis, Davidian, Holloway, and
  Laber]{Tsiatis2021}
Tsiatis, A.~A., Davidian, M., Holloway, S.~T., and Laber, E.~B.
\newblock \emph{Dynamic Treatment Regimes: Statistical Methods for Precision
  Medicine}.
\newblock CRC Press, first edition, 2019.
\newblock \doi{10.1201/9780429192692}.

\bibitem[Uehara et~al.(2021)Uehara, Imaizumi, Jiang, Kallus, Sun, and
  Xie]{Uehara2022}
Uehara, M., Imaizumi, M., Jiang, N., Kallus, N., Sun, W., and Xie, T.
\newblock Finite sample analysis of minimax offline reinforcement learning:
  Completeness, fast rates and first-order efficiency.
\newblock \emph{arXiv preprint arXiv:2102.02981}, 2021.

\bibitem[Voloshin et~al.(2021{\natexlab{a}})Voloshin, Jiang, and
  Yue]{Voloshin2021model}
Voloshin, C., Jiang, N., and Yue, Y.
\newblock Minimax model learning.
\newblock In \emph{Proceedings of The 24th International Conference on
  Artificial Intelligence and Statistics (AISTATS 2021)}, pp.\  1612--1620,
  2021{\natexlab{a}}.
\newblock URL \url{https://proceedings.mlr.press/v130/voloshin21a.html}.

\bibitem[Voloshin et~al.(2021{\natexlab{b}})Voloshin, Le, Jiang, and
  Yue]{Voloshin2021}
Voloshin, C., Le, H.~M., Jiang, N., and Yue, Y.
\newblock Empirical study of off-policy policy evaluation for reinforcement
  learning.
\newblock In \emph{35th Conference on Neural Information Processing Systems
  (NeurIPS 2021) Track on Datasets and Benchmarks}, 2021{\natexlab{b}}.
\newblock URL \url{https://openreview.net/pdf?id=IsK8iKbL-I}.

\bibitem[Wang et~al.(2016)Wang, Schaul, Hessel, Hasselt, Lanctot, and
  Freitas]{Wang2016}
Wang, Z., Schaul, T., Hessel, M., Hasselt, H., Lanctot, M., and Freitas, N.
\newblock Dueling network architectures for deep reinforcement learning.
\newblock In \emph{Proceedings of the 33rd International Conference on Machine
  Learning (ICML 2016)}, pp.\  1995--2003, 2016.
\newblock URL \url{http://proceedings.mlr.press/v48/wangf16.html}.

\bibitem[Weltz et~al.(2022)Weltz, Volfovsky, and Laber]{Weltz2022}
Weltz, J., Volfovsky, A., and Laber, E.~B.
\newblock Reinforcement learning methods in public health.
\newblock \emph{Clinical Therapeutics}, 44\penalty0 (1):\penalty0 139--154,
  2022.
\newblock \doi{10.1016/j.clinthera.2021.11.002}.

\bibitem[Wu et~al.(2019)Wu, Tucker, and Nachum]{Wu2019}
Wu, Y., Tucker, G., and Nachum, O.
\newblock Behavior regularized offline reinforcement learning.
\newblock \emph{arXiv preprint arXiv:1911.11361}, 2019.

\bibitem[Xie \& Jiang(2021)Xie and Jiang]{Xie2021}
Xie, T. and Jiang, N.
\newblock Batch value-function approximation with only realizability.
\newblock In \emph{Proceedings of the 38th International Conference on Machine
  Learning (ICML 2021)}, pp.\  11404--11413, 2021.
\newblock URL \url{https://proceedings.mlr.press/v139/xie21d.html}.

\bibitem[Xie et~al.(2019)Xie, Ma, and Wang]{Xie2019}
Xie, T., Ma, Y., and Wang, Y.-X.
\newblock Towards optimal off-policy evaluation for reinforcement learning with
  marginalized importance sampling.
\newblock In \emph{Advances in Neural Information Processing Systems (NeurIPS
  2019)}, volume~32, 2019.
\newblock URL
  \url{https://proceedings.neurips.cc/paper/2019/file/4ffb0d2ba92f664c2281970110a2e071-Paper.pdf}.

\bibitem[Yang et~al.(2020)Yang, Nachum, Dai, Li, and Schuurmans]{Yang2020}
Yang, M., Nachum, O., Dai, B., Li, L., and Schuurmans, D.
\newblock Off-policy evaluation via the regularized {L}agrangian.
\newblock In \emph{Advances in Neural Information Processing Systems (NeurIPS
  2020)}, volume~33, pp.\  6551--6561, 2020.
\newblock URL
  \url{https://proceedings.neurips.cc//paper/2020/file/488e4104520c6aab692863cc1dba45af-Paper.pdf}.

\bibitem[Yu et~al.(2021)Yu, Liu, Nemati, and Yin]{Yu2021}
Yu, C., Liu, J., Nemati, S., and Yin, G.
\newblock Reinforcement learning in healthcare: A survey.
\newblock \emph{ACM Computing Surveys (CSUR)}, 55\penalty0 (1):\penalty0 1--36,
  2021.
\newblock \doi{10.1145/3477600}.

\bibitem[Yu et~al.(2020)Yu, Chen, Wang, Xian, Chen, Liu, Madhavan, and
  Darrell]{Yu2020}
Yu, F., Chen, H., Wang, X., Xian, W., Chen, Y., Liu, F., Madhavan, V., and
  Darrell, T.
\newblock {BDD100K}: A diverse driving dataset for heterogeneous multitask
  learning.
\newblock In \emph{Proceedings of the IEEE/CVF Conference on Computer Vision
  and Pattern Recognition (CVPR 2020)}, 2020.
\newblock \doi{10.1109/CVPR42600.2020.00271}.

\bibitem[Zhang et~al.(2021)Zhang, Paine, Nachum, Paduraru, Tucker, ziyu wang,
  and Norouzi]{Zhang2021}
Zhang, M.~R., Paine, T., Nachum, O., Paduraru, C., Tucker, G., ziyu wang, and
  Norouzi, M.
\newblock Autoregressive dynamics models for offline policy evaluation and
  optimization.
\newblock In \emph{The 9th International Conference on Learning Representations
  (ICLR 2021)}, 2021.
\newblock URL \url{https://openreview.net/pdf?id=kmqjgSNXby}.

\bibitem[Zhang \& Jiang(2021)Zhang and Jiang]{Zhang2021ps}
Zhang, S. and Jiang, N.
\newblock Towards hyperparameter-free policy selection for offline
  reinforcement learning.
\newblock In \emph{Advances in Neural Information Processing Systems (NeurIPS
  2021)}, volume~34, pp.\  12864--12875, 2021.
\newblock URL
  \url{https://proceedings.neurips.cc/paper_files/paper/2021/file/6add07cf50424b14fdf649da87843d01-Paper.pdf}.

\bibitem[Zhu et~al.(2020)Zhu, Lu, and Song]{Zhu2020}
Zhu, L., Lu, W., and Song, R.
\newblock Causal effect estimation and optimal dose suggestions in mobile
  health.
\newblock In \emph{Proceedings of the 37th International Conference on Machine
  Learning (ICML 2020)}, pp.\  11588--11598, 2020.
\newblock URL \url{https://proceedings.mlr.press/v119/zhu20c.html}.

\end{thebibliography}
\bibliographystyle{icml2023}

\onecolumn
\twocolumn

\appendix
\renewcommand\thefigure{\thesection.\arabic{figure}}   
\part*{Appendix}
\label{sec:Appendix}

\counterwithin{figure}{section}
\counterwithin{table}{section}
\counterwithin{equation}{section}
\counterwithin{algorithm}{section}

\section{Extended Details of OMS Algorithms}
\label{append:oms}

\subsection{Extended SBV Algorithms}
\label{append:alg}

\setcounter{algorithm}{0}
\begin{algorithm}[ht]
  \caption{SBV with Tuned Regression Algorithm}\label{alg:sbv_tune}
  \begin{algorithmic}[1]
    \REQUIRE{Offline dataset $\mathcal D=\{(s,a,r,s')\}$} 
    \REQUIRE{Set of offline RL algorithms \\ $\mathcal H=\{H_1,...,H_M\}$}
    \REQUIRE{Set of regression algorithms \\ $\mathcal R=\{R_1,...,R_N\}$}
    \STATE{Randomly partition trajectories in $\mathcal D$ to training set $\mathcal D_T$ and validation set $\mathcal D_V$}
    \FOR{RL algorithm $m\in\{1,...,M\}$}
        \STATE{Estimate $Q^*$ as $Q_m$ by running offline RL algorithm $H_m$ on $\mathcal D_T$}
        \FOR{regression algorithm $n\in\{1,...,N\}$}
            \STATE{Estimate $\mathcal B^*Q_m$ as $\widehat{\mathcal B}^*_{n}Q_m$ by running regression algorithm $R_n$ on $\mathcal D_T$ to minimize Bellman backup MSE (Equation \ref{eq:mse_backup})}
            \STATE{Evaluate the error of $R_n$ in estimating $\mathcal B^*Q_m$ as \resizebox{\linewidth}{!} 
            {$\text{Err}_{m,n}=\mathbb E_{\mathcal D_V}\left[(r+\gamma\max_{a'}Q_m(s',a')-\widehat B_n^* Q_m)^2\right]$}}
        \ENDFOR
        \STATE{Estimate $\mathcal B^*Q_m$ using the best regression algorithm as $\widehat{\mathcal B}^*_{n^*(m)}Q_m$ with $n^*(m)=\mbox{argmin}_{1\leq n \leq N}\text{Err}_{m,n}$}
        \STATE{Estimate the MSBE of $Q_m$ as \\ $\mathbb E_{\mathcal D_V}[(Q_m(s,a)-\widehat{\mathcal B}^*_{n^*(m)}Q_m(s,a))^2]$}
    \ENDFOR
    \STATE{{\bfseries Output:} $Q_{m^*}$ as our estimate of $Q^*$ with $m^*=\underset{1\leq m \leq M}{\mbox{argmin}}\mathbb E_{\mathcal D_V}[(Q_m(s,a)-(\widehat{\mathcal B}^*_{n^*(m)}Q_m)(s,a))^2]$}
  \end{algorithmic}
\end{algorithm}

\makeatletter
\setlength{\@fptop}{0pt}
\makeatother
\begin{algorithm}[!ht]
    \caption{Applying Early Stopping to DQN with SBV. Here $Q_{\theta_{k+1}}$ denotes the trained Q-Network after $k+1$ iterations of the DQN algorithm while $\mathcal B_{\phi_{k+1}}$ denotes the \textit{Bellman network} subsequently updated by SBV in order to approximate $\mathcal B^*Q_{\theta_{k+1}}$. }\label{alg:dqn_sbv}
    \begin{algorithmic}[1]
        \REQUIRE{Offline dataset $\mathcal D=\{(s,a,r,s')\}$}
        \STATE{Randomly partition trajectories in $\mathcal D$ to training set $\mathcal D_T$ and validation set $\mathcal D_V$}
        \STATE{Initialize deep Q-network $Q_{\theta_{0}}$ with trainable weights $\theta_0$ and Bellman network $\mathcal B_{\phi_0}$ with trainable weights $\phi_0$}
        \FOR{iteration $k\in\{0,...,K-1\}$}
            \STATE{Update Q-network weights from $\theta_k$ to $\theta_{k+1}$ by running DQN on $\mathcal D_T$ for one iteration}
            \STATE{Update Bellman network weights from $\phi_{k}$ to $\phi_{k+1}$ by running gradient descent with loss function: \resizebox{\linewidth}{!} {$L(\phi)=\mathbb E_{\mathcal D_T}\left[\left(r+\gamma\max_{a'}Q_{\theta_{k+1}}(s',a')-\mathcal B_\phi(s,a)\right)^2\right]$}}
            \STATE{Estimate the MSBE of $Q_{\theta_{k+1}}$ as \\ $\mathbb E_{\mathcal D_V}[(Q_{\theta_{k+1}}(s,a)-\mathcal B_{\phi_{k+1}}(s,a))^2]$}
        \ENDFOR
        \STATE{{\bfseries Output:} $Q_{\theta_{k^*}}$ as our estimator of $Q^*$ where $k^*=\mbox{argmin}_{1\leq k \leq K} \mathbb E_{\mathcal D_V}[(Q_{\theta_{k}}(s,a)-\mathcal B_{\phi_{k}}(s,a))^2]$}
    \end{algorithmic}
\end{algorithm}

\subsection{Extended Details of Model-Free Baselines}
\label{append:baselines}

Computationally- and memory-efficient implementations of Supervised Bellman Validation (SBV), the empirical mean squared Bellman error (EMSBE), weighted per-decision importance sampling (WIS) \cite{Precup2000}, Fitted Q-Evaluation (FQE) \cite{Le2019} and Batch Value Function Tournament (BVFT) \cite{Xie2021} on Atari can be found in our repository \url{https://github.com/jzitovsky/SBV}. The methodologies of validation EMSBE, WIS and FQE is discussed below and implementation details on Atari are discussed in Appendix \ref{append:wis_atari}. SBV is discussed in Section \ref{sec:sbv} and its implementation on Atari is discussed in Appendix \ref{append:sbv_atari}. BVFT is discussed in Appendix \ref{append:bvft} and its implementation is discussed in Figure \ref{figure:bvft}. Also see Sections \ref{sec:background} and \ref{sec:previous} for relevant background and notational definitions. 

We define validation EMSBE as:
\begin{equation}
\mathbb E_{\mathcal D_V}\left[\left(Q_m(s,a)-r-\gamma\max_{a'\in\mathcal A}Q_m(s',a')\right)^2\right].\label{eq:emsbe_val}
\end{equation}
More discussion of the EMSBE can be found in Section \ref{sec:bellman_error}, Section \ref{sec:theory} and Appendix \ref{append:wis_atari}. The WIS estimator from \citet{Precup2000} is defined as:
\begin{equation}
\widehat J_{\mbox{WIS}}(\pi)=\sum_{i=1}^N\frac{\sum_{t=0}^{T}\gamma^tR_t^i\prod_{v=0}^t\frac{\pi(A_v^i|S_v^i)}{\mu(A_v^i|S_v^i)}}{\sum_{t=0}^{T}\gamma^t\prod_{v=0}^t\frac{\pi(A_v^i|S_v^i)}{\mu(A_v^i|S_v^i)}},\label{eq:wis}
\end{equation}
where $S_t^i$ is the state of the $i$th observed trajectory at time step $t$ and similarly for $A_t^i$ and $R_t^i$, and $T$ is some large horizon time of interest. Some works also define weighted IS differently \cite{Thomas2016}. In the event that all trajectories end in a terminal state, we can set $T=T_i$ where $T_i$ is the length of the $i$th observed trajectory and the horizon becomes infinite. It can be shown that as the number of observed trajectories $N\to\infty$, $\widehat J_{\mbox{WIS}}(\pi)$ converges with probability one to a normalized version of $J(\pi)$. In the event that the behavioral policy $\mu$ is unknown, we can estimate it by behavioral cloning (BC) \cite{Osa2018}. As the only potential hyperparameters of WIS relate to those of the BC algorithm, and as a BC model can be evaluated and tuned offline via cross-entropy on a held-out validation set, we say that WIS can easily tune its own hyperparameters offline. 

The main problem with this estimator is the estimation variance: the sample variance of the importance weights $\prod_{v=0}^t\frac{\pi(A_v|S_v)}{\mu(A_v|S_v)}$ increases exponentially with $v$, making it difficult for the WIS estimator to accurately model long-term dependencies between actions and rewards and take into account rewards occurring far after the initial state. 

FQE estimates the action-value function $Q^\pi$ of policy $\pi$ as $\widehat Q^\pi$ using a modified off-policy Q-learning or actor-critic algorithm. For example, we could estimate $Q^\pi$ by modifying FQI or DQN to perform updates:
\vspace{-0.2cm}
\begin{equation*}
\resizebox{\linewidth}{!} 
{
    $Q^{(k+1)}\leftarrow \mbox{argmin}_f \mathbb E_{\mathcal D}\left[\left(f(s,a)-r-\gamma\mathbb E_{a'\sim \pi(\cdot|s')}Q^{(k)}(s',a')\right)^2\right]$
}.
\end{equation*}
$J(\pi)$ is then estimated as: 
\begin{equation}
\widehat J_{\mbox{FQE}}(\pi)=\mathbb E_{s_0\sim \widehat d_0}[\hat Q^\pi(s,\pi(s))],\label{eq:fqe}
\end{equation}
where $\hat d_0$ is the empirical distribution of initial states. There is currently no established or well-known procedure to choose or tune the algorithm used to estimate $Q^\pi$, though there have been a few approaches proposed in very recent work \citep{Zhang2021ps, Miyaguchi2022}.

\onecolumn
\section{Extensions to Infinite State Spaces with Mathematical Proofs}
\label{append:proofs}

\setcounter{theorem}{0}

We begin with some additional preliminaries: Let $\lambda$ be the dominating measure of density $P^\mu$ such that $\mathbb E_{(s,a)\sim P^\mu}[Q(s,a)]=\int_{\mathcal S\times\mathcal A}Q(s,a)P^\mu(s,a)d\lambda$ and let $||f||_\infty$ denote the essential supremum of function $f$ with respect to measure $\lambda$. If the state space is finite, $\lambda$ is equal to the counting measure, $\mathbb E_{(s,a)\sim P^\mu}[Q(s,a)]=\sum_{(s,a)\in\mathcal S\times\mathcal A}[Q(s,a)P^\mu(s,a)]$ and $||f||_\infty=\max_{(s,a)\in\mathcal S\times\mathcal A}|f(s,a)|$. Let $d(s)=d(s,a)=I(s\text{ is terminal})$. Let $P^\mu_T(s,a,r,s')=d^\mu(s)\mu(a|s)T(s'|s,a)$ denote the underlying population distribution of our observed transitions. Let $F^\mu$ denote the (cumulative) distribution function (CDF) of $P^\mu_T$ and $F^{\mathcal D}$ denote the \textit{empirical distribution function (EDF)} of transitions associated with $\mathcal D$. Under general conditions, $||F^{\mathcal D}-F^\mu||_\infty\to 0$ as $|\mathcal D|\to\infty$ with probability one \cite{Kotz2006}. Theoretical results given here imply those present in Section \ref{sec:theory} when the state space is finite.

We make a few notes about the assumptions used by the first proposition. First, when our state space is finite, our assumptions about $d_0$, $d^\mu$ and $T$ always hold. Second, our assumptions on $P^\mu$, $d_0$ and $T$ could be weakened, but this would lead to our derived bounds being less interpretable. Specifically, it is sufficient for $\max_\pi||\sqrt{(1-d)P_t^{\pi}/P^\mu}||_\infty=o(1/\gamma^t)$ and $\max_{\pi,m}||\sqrt{(1-d)m_t^{\pi}/P^\mu}||_\infty=o(1/\gamma^t)$ where $P^\pi_t$ is the distribution of state-action pairs induced from starting at state-action pair $(s_0,a_0)\sim P^\mu$ and following policy $\pi$ for $t$ time steps and $m_t^{\pi}$ is the distribution of state-action pairs induced from starting at state $s_0\sim d_0$, following policy $\pi$ for $t$ time steps and then applying policy $m$ for a final time step. Under this weakened assumption, we still have that $C_Q<C_J<\infty$ by the (Cauchy) ratio test for series convergence, which now depends on $\sum_{t=0}^\infty\gamma^t\max_\pi||\sqrt{(1-d)P_t^{\pi}/P^\mu}||_\infty$ and $\sum_{t=0}^\infty\gamma^t \max_{\pi,m}||\sqrt{(1-d)m_t^{\pi}/P^\mu}||_\infty$.

While our analysis in Section \ref{sec:theory} focused on components unrelated to the MDP such as the behavioral policy and estimation accuracy of the MSBE, our analysis here shows that the MSBE's theoretical performance will also depend on the MDP's ability to control the rate at which distribution shift occurs when following alternative policies to $\mu$. This rate is automatically bounded when the transition and initial state probabilities are bounded, though as discussed above this assumption is not necessary. There may also be a way to add a regularization term to the MSBE to strengthen theoretical guarantees when stochasticity in the behavioral policy, transitions or initial state distribution is more restricted. We leave this to future work. 

\begin{proposition}\label{prop1_append}
Assume $P^\mu(s,a)\geq\psi$ and $d^\mu(s),(1-d)d_0(s),(1-d)(s')T(s'|s,a)\leq B_T$ for some $0<\psi,B_T<\infty$ and all $(s,a,s')\in\mathcal S\times\mathcal A\times\mathcal S$ with probability one. Let $\hat m(Q_m)$ be an estimate of $||Q_m-\mathcal B^*Q_m||_{P^\mu}$ with absolute estimation error $e(\hat m(Q_m))$ and assume that $\hat m(Q_m)\leq\epsilon$ and $e(\hat m(Q_m))\leq\delta$. Then with probability one i) $||Q^*-Q_m||_{P^\mu}\leq C_Q(\epsilon+\delta)$ where $C_Q=\frac{\sqrt{B_T}}{\sqrt{\psi}(1-\gamma)}$ and ii) $J(\pi^*)-J(\pi_{Q_m})\leq C_J(\epsilon+\delta)$ where $C_J=\frac{2B_T}{\psi(1-\gamma)^2}$.
\end{proposition}

We will prove this proposition using a series of lemmas, from which this proposition will be a direct corollary. 

\newtheorem{lemma}{Lemma}[theorem]

\begin{lemma}
If $\hat m(Q_m)\leq\epsilon$ and $e(\hat m(Q_m))\leq\delta$, then $||Q^*-Q_m||_{P^\mu}\leq C(\epsilon+\delta)$ where $C=\frac{\sqrt{B_T}}{\sqrt{\psi}(1-\gamma)}$ with probability one.
\end{lemma}

\begin{proof}
For any density function $P$ of state-action pairs, we have by Minkowski's inequality:

\begin{equation*}
||Q_m-Q^*||_{P}\leq ||Q_m-\mathcal B^*Q_m||_{P}+||\mathcal B^*Q_m-Q^*||_{P} =  ||Q_m-\mathcal B^*Q_m||_{P}+||\mathcal B^*Q_m-\mathcal B^*Q^*||_{P}. 
\end{equation*}

 As for the first term from the last line, we can use importance sampling to obtain the identity $||Q_m-\mathcal B^*Q_m||_{P}=\left|\left|(Q_m-\mathcal B^*Q_m)\sqrt{P/P^\mu}\right|\right|_{P^\mu}$. As to the second term, let $m(s)=\text{argmax}_{a}|Q_m(s,a)-Q^*(s,a)|$ denote the \textit{max-error} policy of $Q_m$. Let $P^\mu_v(s',a')$ be the marginal density of state-action pairs wrt dominating measure $\lambda$ induced from sampling $(s_0,a_0)\sim P^\mu$ and sampling $s_{t+1}\sim T(\cdot|s_t,a_t)$ and $a_{t+1}\sim m(\cdot|s_{t+1})$ for $v$ time steps (we assume this density always exists).  Observe that for any $v\geq 0$:

\begin{align*}
||\mathcal B^*Q_m-\mathcal B^*Q^*||_{P^\mu_v}^2=&\mathbb E_{(s,a)\sim P^\mu_v}\left\{\left(\mathbb E_{s'\sim T(\cdot|s,a)}\left[r+\gamma\max_{a'}Q_m(s',a')\right]-\mathbb E_{s'\sim T(\cdot|s,a)}\left[r+\gamma\max_{a'}Q^*(s',a')\right]\right)^2\right\} \\
=&\gamma^2\mathbb E_{(s,a)\sim P^\mu_v}\left\{\left(\mathbb E_{s'\sim T(\cdot|s,a)}\left[\max_{a'}Q_m(s',a')-\max_{a'}Q^*(s',a')\right]\right)^2\right\} \\
\leq&\gamma^2\mathbb E_{(s,a)\sim P^\mu_v}\left\{\left(\mathbb E_{s'\sim T(\cdot|s,a),a'\sim m(\cdot|s')}\left|Q_m(s',a')-Q^*(s',a')\right|\right)^2\right\} \\
\leq& \gamma^2\mathbb E_{(s,a)\sim P^\mu_v}\left\{\mathbb E_{s'\sim T(\cdot|s,a),a'\sim m(\cdot|s')}\left[\left(Q_m(s',a')-Q^*(s',a')\right)^2\right]\right\} \\
=& \gamma^2\mathbb E_{s',a'\sim P_{v+1}^\mu}\left[\left(Q_m(s',a')-Q^*(s',a')\right)^2\right] \\
\leq& \gamma^2\left(||Q_m-\mathcal B^*Q_m||_{P_{v+1}^\mu}+||\mathcal B^*Q_m-\mathcal B^*Q^*||_{P_{v+1}^\mu}\right)^2,  \\
\end{align*}

where the inequality on the third-to-last line comes from Jensen's inequality. Putting this all together, we have for any $N\in\mathbb{N}$:

\begin{align*}
||Q_m-Q^*||_{P^\mu} \leq& ||Q_m-\mathcal B^*Q_m||_{P^\mu_{0}}+||\mathcal B^*Q_m-\mathcal B^*Q^*||_{P^\mu_0} \\
\leq& ||Q_m-\mathcal B^*Q_m||_{P^\mu_{0}}+\gamma\left(||Q_m-\mathcal B^*Q_m||_{P_{1}^\mu}+||\mathcal B^*Q_m-\mathcal B^*Q^*||_{P_{1}^\mu}\right) \\
\leq&
\resizebox{0.85\linewidth}{!} 
{$||Q_m-\mathcal B^*Q_m||_{P^\mu_0}+\gamma \left|\left|(Q_m-\mathcal B^*Q_m)\sqrt{P_1^\mu/P_0^\mu}\right|\right|_{P^\mu_{0}}+\gamma^2\left(||Q_m-\mathcal B^*Q_m||_{P_{2}^\mu}+||\mathcal B^*Q_m-\mathcal B^*Q^*||_{P_{2}^\mu}\right)$} \\
\leq& 
||Q_m-Q^*||_{P^\mu}+\gamma \left|\left|(Q_m-\mathcal B^*Q_m)\sqrt{P_1^\mu/P_0^\mu}\right|\right|_{P^\mu_{0}}+\gamma^2\left|\left|(Q_m-\mathcal B^*Q_m)\sqrt{P_2^\mu/P_0^\mu}\right|\right|_{P^\mu_{0}}&\\&+\gamma^3\left(||Q_m-\mathcal B^*Q_m||_{P_{3}^\mu}+||\mathcal B^*Q_m-\mathcal B^*Q^*||_{P_{3}^\mu}\right) \\
\vdots & \\
\leq& \sum_{t=0}^N\gamma^t\left|\left|(Q_m-\mathcal B^*Q_m)\sqrt{P_t^\mu/P_0^\mu}\right|\right|_{P^\mu_{0}}+\gamma^{N+1}\left(||Q_m-\mathcal B^*Q_m||_{P_{N+1}^\mu}+||\mathcal B^*Q_m-\mathcal B^*Q^*||_{P_{N+1}^\mu}\right). \\
\end{align*}

We assume that the following assumptions hold almost surely with respect to $P^\mu$: (A1) $Q_m=(1-d)Q_m$, $|Q_m|\leq B_Q$ and $|r|\leq B_R$ for some $B_Q,B_R<\infty$; (A2) $P_0^\mu\geq \psi$ for some $\psi>0$; (A3) $d^\mu(s), (1-d(s'))T(s'|s,a)\leq B_T$ for some $B_T<\infty$ and all $(s,a,s')\in\mathcal S\times\mathcal A\times\mathcal S$. The first assumption on $Q_m$ is arbitrary, and assuming that the reward function is bounded is standard. Under condition (A1), we have that $||Q_m-\mathcal B^*Q_m||_\infty,||\mathcal B^*Q_m-\mathcal B^*Q^*||_\infty<\infty$, which means that $\lim_{N\to\infty}\gamma^{N+1}\left(||Q_m-\mathcal B^*Q_m||_{P_{N+1}^\mu}+||\mathcal B^*Q_m-\mathcal B^*Q^*||_{P_{N+1}^\mu}\right)=0$. Moreover, by condition (A3), $(1-d)P^\mu\leq d^\mu\leq B_T$ and we have by mathematical induction that for $t\geq 1$:
$$
(1-d)P^\mu_t(s',a')=\int_{\mathcal S\times\mathcal A}(1-d)(s',a')m(a'|s')T(s'|s,a)P^\mu_{t-1}(s,a)d\lambda(s,a)\leq B_T\int_{\mathcal S\times\mathcal A}P^\mu_{t-1}(s,a)d\lambda=B_T.
$$
Therefore, under conditions (A1)-(A3), we have for any $t\geq 0$:

\begin{align*}
\left|\left|(Q_m-\mathcal B^*Q_m)\sqrt{P_t^\mu/P_0^\mu}\right|\right|_{P^\mu_{0}}^2=& \mathbb E_{s,a\sim P^\mu_0}\left[\left(Q_m(s,a)-(\mathcal B^*Q_m)(s,a)\right)^2\frac{P^\mu_t(s,a)}{P^\mu_0(s,a)}\right] \\
=& \mathbb E_{s,a\sim P^\mu_0}\left[\left(Q_m(s,a)-(\mathcal B^*Q_m)(s,a)\right)^2\frac{(1-d(s,a))P^\mu_t(s,a)}{P^\mu_0(s,a)}\right] \\
\leq& \frac{B_T}{\psi}||Q_m-\mathcal B^*Q_m||_{P_0^\mu}^2.
\end{align*}

Therefore:

\begin{align*}
||Q_m-Q^*||_{P^\mu} \leq& \resizebox{0.85\linewidth}{!} 
{$\lim_{N\to\infty}\left\{\sum_{t=0}^N\gamma^t\left|\left|(Q_m-\mathcal B^*Q_m)\sqrt{P_t^\mu/P_0^\mu}\right|\right|_{P^\mu_{0}}+\gamma^{N+1}\left(||Q_m-\mathcal B^*Q_m||_{P_{N+1}^\mu}+||\mathcal B^*Q_m-\mathcal B^*Q^*||_{P_{N+1}^\mu}\right)\right\}$}
\\
=&\sum_{t=0}^\infty\gamma^t\left|\left|(Q_m-\mathcal B^*Q_m)\sqrt{P_t^\mu/P_0^\mu}\right|\right|_{P^\mu_{0}} \\
\leq&\frac{||Q_m-\mathcal B^*Q_m||_{P^\mu_{0}}\sqrt{B_T}}{\sqrt{\psi}}\sum_{t=0}^\infty\gamma^t \\
\leq&\left[\hat m(Q_m)+e(\hat m(Q_m))\right]\frac{\sqrt{B_T}}{\sqrt{\psi}(1-\gamma)}.
\end{align*}
\end{proof}

\begin{lemma}
If $||Q^*-Q_m||_{P^\mu}\leq\epsilon$, it holds that $J(\pi^*)-J(\pi_{Q_m})\leq C\epsilon$ with probability one where $C=\frac{2\sqrt{B_T}}{\sqrt{\psi}(1-\gamma)}$.
\end{lemma}

\begin{proof}
Let $\pi_m=\pi_{Q_m}$ and $m(a|s)$ be the max-error policy of $Q_m$ defined in the proof for the previous lemma. Let $d^{0}_{v}(s)$ be the marginal density of states induced from sampling $s_0\sim d_0$ and sampling $a_{t}\sim \pi_m(\cdot|s_{t})$ and $s_{t+1}\sim T(\cdot|s_t,a_t)$ for $v$ time steps, and  let $m_v(s,a)$ be the density of state-action pairs wrt dominating measure $\lambda$ defined as $d^0_v(s)m(a|s)$ (we assume this density always exists). Observe that for any $v\geq 0$:

\begin{align*}
||V^*-V^{\pi_{Q_m}}||_{d^0_v}^2=& \mathbb E_{s\sim d^0_v}\left[\left(Q^*(s,\pi^*(s))-Q^{\pi_{Q_m}}(s,\pi_m(s))\right)^2\right]\\
=& \mathbb E_{s\sim d^0_v}\left[\left(Q^*(s,\pi^*(s))-Q^*(s,\pi_m(s))+Q^*(s,\pi_m(s))-Q^{\pi_{Q_m}}(s,\pi_m(s))\right)^2\right] \\
=& \mathbb E_{s\sim d^0_v}\left[\left(Q^*(s,\pi^*(s))-Q^*(s,\pi_m(s))+\gamma\mathbb E_{s'\sim T(\cdot|s,\pi_m(s))}[V^*(s')-V^{Q_m}(s')]\right)^2\right]\\
\leq&  \mathbb E_{s\sim d^0_v}\left[\left(Q^*(s,\pi^*(s))-Q_m(s,\pi^*(s))+Q_m(s,\pi_m(s))-Q^*(s,\pi_m(s))\right.\right.\\&\left.\left. +\gamma\mathbb E_{s'\sim T(\cdot|s,\pi_m(s))}[V^*(s')-V^{Q_m}(s')]\right)^2\right].
\end{align*}

Note that for any $s\in\mathcal S$, it holds that $a(s)=Q^*(s,\pi^*(s))-Q^{*}(s,\pi_m(s))=\max_aQ^*(s,a)-Q^{*}(s,\pi_m(s))\geq0$, $b(s)=-Q_m(s,\pi^*(s))+Q_m(s,\pi_m(s))=-Q_m(s,\pi^*(s))+\max_aQ_m(s,a)\geq0$ and 
$c(s)=\gamma\mathbb E_{s'\sim T(\cdot|s,\pi_m(s))}[V^*(s')-V^{Q_m}(s')]=\gamma\mathbb E_{s'\sim T(\cdot|s,\pi_m(s))}[\max_\pi V^\pi(s')-V^{Q_m}(s')]\geq0$. Then $(a(s)+c(s))^2\leq (a(s)+b(s)+c(s))^2$ for any $s\in \mathcal S$, which implies the inequality on the final line. Then by Minkowski's inequality and Jensen's inequality:

\begin{align*}
||V^*-V^{\pi_{Q_m}}||_{d^0_v} \leq& \sqrt{\mathbb E_{s\sim d^0_v}\left[\left(Q^*(s,\pi^*(s))-Q_m(s,\pi^*(s))\right)^2\right]} + \sqrt{\mathbb E_{s\sim d^0_v}\left[\left(Q_m(s,\pi_m(s))-Q^*(s,\pi_m(s))\right)^2\right]}\\&+\sqrt{\mathbb E_{s\sim d^0_v}\left[\left(\gamma\mathbb E_{s'\sim T(\cdot|s,\pi_m(s))}[V^*(s')-V^{Q_m}(s')]\right)^2\right]} \\
\leq& 2\sqrt{\mathbb E_{s\sim d^0_v}\mathbb E_{a\sim m(\cdot|s)}\left[(Q^*(s,a)-Q_m(s,a))^2\right]}+\gamma\sqrt{\mathbb E_{s\sim d^0_v}\mathbb E_{s'\sim T(\cdot|s,\pi_m(s))}\left(V^*(s')-V^{Q_m}(s')\right)^2} \\
=& 2\sqrt{\mathbb E_{(s,a)\sim P^\mu}\left[(Q^*(s,a)-Q_m(s,a))^2\frac{m_v(s,a)}{P^\mu(s,a)}\right]}+\gamma\sqrt{\mathbb E_{s\sim d^0_{v+1}}\left(V^*(s')-V^{Q_m}(s')\right)^2} \\
=& 2||(Q^*-Q_m)\sqrt{m_v/P^\mu}||_{P^\mu}+\gamma ||V^*-V^{Q_m}||_{d^0_{v+1}}.
\end{align*}

Therefore, for any $N\in \mathbb{N}$:

\begin{align*}
||V^*-V^{\pi_m}||_{d_0}\leq& 2||(Q^*-Q_m)\sqrt{m_0/P^\mu}||_{P^\mu}+\gamma ||V^*-V^{Q_m}||_{d^0_{1}}\\
\leq& 2||(Q^*-Q_m)\sqrt{m_0/P^\mu}||_{P^\mu} + 2\gamma ||(Q^*-Q_m)\sqrt{m_1/P^\mu}||_{P^\mu}+\gamma^2 ||V^*-V^{Q_m}||_{d^0_{2}}\\
\vdots& \\
\leq& \sum_{t=0}^N2\gamma^t||(Q^*-Q_m)\sqrt{m_t/P^\mu}||_{P^\mu}+\gamma^{N+1}||V^*-V^{\pi_m}||_{d_{N+1}^\mu}.
\end{align*}

We make the same assumptions as Lemma B.2.1 and additionally assume that $d_0(s)\leq B_T$ for all $s\in\mathcal S$ with probability one (A4). Under condition (A1), we have that $||V^*-V^{\pi_m}|_\infty<\infty$, which means that $\lim_{N\to\infty}\gamma^{N+1}||V^*-V^{\pi_m}||_{d_{N+1}^\mu}=0$. Moreover, under conditions (A3) and (A4), $(1-d)m_0(s,a)\leq d_0(s)\leq B_T$ and for $t\geq 1$, we have by mathematical induction that:
$$
(1-d)m_t(s',a')=\int_{\mathcal S\times\mathcal A}(1-d)(s',a')m(a'|s')T(s'|s,a)\pi_m(a|s)d^{\mu}_{t-1}(s)\lambda(s,a)\leq B_T\int_{\mathcal S\times\mathcal A}d^{\mu}_{t-1}(s)\pi_m(a|s)\lambda=B_T.
$$

Therefore, under conditions (A1)-(A4), we have for any $t>0$ $||(Q^*-Q_m)\sqrt{m_t/P^\mu}||_{P^\mu}^2\leq \frac{B_T}{\psi}||Q^*-Q_m||_{P^\mu}^2$. Putting this all together, we have:
$$
J(\pi^*)-J(\pi_m)=||V^*-V^{\pi_m}||_{d^0}\leq \frac{2\sqrt{B_T}}{\sqrt{\psi}(1-\gamma)}||Q^*-Q_m||_{P^\mu}.
$$
\end{proof}

\begin{proposition}\label{prop2_append}
Assume that $||F^{\mathcal D}-F^\mu||_\infty=0$. Then:
\begin{equation*}
\begin{split}
||Q_m-\mathcal B_{\mathcal D}Q_m||^2_{\mathcal D}-||Q_m-\mathcal B^*Q_m||^2_{P^\mu}=\quad\quad&\\\mathbb E_{(S_t,A_t)\sim P^\mu}\left\{\mbox{Var}\left[R_t+\gamma\max_{a'\in\mathcal A}Q_m(S_{t+1},a')|S_t,A_t\right]\right\}.
\end{split}
\end{equation*}
\end{proposition}

\begin{proof}
Suppose $||F^{\mathcal D}-F^{\mathcal \mu}||_\infty=0$. Note $\mathbb E_{(s,a,r,s')\sim P^\mu_T}[f(s,a,r,s')]=\mathbb E_{(s,a)\sim P^\mu}\mathbb E_{s'\sim T(\cdot|s,a)}[f(s,a,R(s,a,s'),s')]$ for any function of transitions $f$. Then:

\begin{align*}
||\mathcal B_{\mathcal D}Q_m-Q_m||^2_{\mathcal D}=& \mathbb E_{(s,a,r,s')\sim P^\mu}\left[\left(r+\gamma\max_{a'}Q_m(s',a')-Q_m(s,a)\right)^2\right] \\
=& \mathbb E_{(s,a,r,s')\sim P^\mu_T}\left[\left(r+\gamma\max_{a'}Q_m(s',a')-\mathcal B^*Q_m(s,a)+\mathcal B^*Q_m(s,a)-Q_m(s,a)\right)^2\right] \\
=& ||\mathcal B^*Q_m-Q_m||^2_{P^\mu}\\& + \mathbb E_{(s,a)\sim P^\mu}\mathbb E_{s'\sim T(\cdot|s,a)}\left[\left(R(s,a,s')+\gamma\max_{a'}Q_m(s',a')-\mathcal B^*Q_m(s,a)\right)^2\right] \\+& 2\mathbb E_{(s,a)\sim P^\mu}\mathbb E_{s'\sim T(\cdot|s,a)}\left[\left(\mathcal B^*Q_m(s,a)-Q_m(s,a)\right)\left(R(s,a,s')+\gamma\max_{a'}Q_m(s',a')-\mathcal B^*Q_m(s,a\right)\right]. 
\end{align*}

Note that $\mathbb E_{s'\sim T(\cdot|s,a)}\left[\left(R(s,a,s')+\gamma\max_{a'}Q_m(s',a')-\mathcal B^*Q_m(s,a)\right)^2\right]$ is equal to $\text{Var}\left[R_t+\gamma\max_{a'}Q_m(S_{t+1},a')|S_t=s,A_t=a\right]$.

Moreover:

\begin{align*}
&\mathbb E_{(s,a)\sim P^\mu}\left\{\mathbb E_{s'\sim T(\cdot|s,a)}\left[\left(\mathcal B^*Q_m(s,a)-Q_m(s,a)\right)\left(R(s,a,s')+\gamma\max_{a'}Q_m(s',a')-\mathcal B^*Q_m(s,a\right)\right]\right\} \\
=& \mathbb E_{(s,a)\sim P^\mu}\left\{\left(\mathcal B^*Q_m(s,a)-Q_m(s,a)\right)\mathbb E_{s'\sim T(\cdot|s,a)}\left[\left(R(s,a,s')+\gamma\max_{a'}Q_m(s',a')-\mathcal B^*Q_m(s,a\right)\right]\right\} \\
=& \mathbb E_{(s,a)\sim P^\mu}\left\{\left(\mathcal B^*Q_m(s,a)-Q_m(s,a)\right)\left(\mathcal B^*Q_m(s,a)-\mathcal B^*Q_m(s,a)\right)\right\}=0. \\
\end{align*}

Therefore:

$$
||Q_m-\mathcal B_{\mathcal D}Q_m||^2_{\mathcal D}=||Q_m-\mathcal B^*Q_m||^2_{P^\mu}+\mathbb E_{(S_t,A_t)\sim P^\mu}\left\{\text{Var}\left[R_t+\gamma\max_{a'}Q_m(S_{t+1},a')|S_t,A_t\right]\right\}.
$$

This concludes the proof.
\end{proof}

\begin{proposition}\label{prop3_append}
Assume that $||F^{\mathcal D_V}-F^\mu||_\infty=||F^{\mathcal D_T}-F^\mu||_\infty=0$ and $\widehat{\mathcal B}^*Q_m=\mbox{arginf}_{f}||\mathcal B_{\mathcal D_T}Q_m-f||^2_{\mathcal D_T}$. Then $||Q_m-\widehat{\mathcal B}^*Q_m||^2_{\mathcal D_V}=||Q_m-{\mathcal B}^*Q_m||^2_{P^\mu}$.
\end{proposition}

\begin{proof}
 As $||F^{\mathcal D_T}-F^\mu||_\infty=0$, we have:

\begin{align*}
||\mathcal B_{\mathcal D_T}Q_m-f||^2_{\mathcal D_T}=&\mathbb E_{(s,a,r,s')\sim P^\mu}\left[\left(r+\gamma\max_{a'}Q_m(s',a')-f(s,a)\right)^2\right] \\
=& \mathbb E_{(s,a)\sim P^\mu}\mathbb E_{s'\sim T(\cdot|s,a)}\left[\left(R(s,a,s')+\gamma\max_{a'}Q_m(s',a')-f(s,a)\right)^2\right]. \\
\end{align*}

This is just a population MSE loss function with targets $y=R(s,a,s')+\gamma\max_{a'}Q_m(s',a'),s'\sim T(\cdot|s,a)$ and covariates $x=(s,a)\sim P^\mu$. It is well-known that the function $f$ minimizing this expectation satisfies $f(x)=\mathbb E[Y|X=x]=\mathcal B^*Q_m$ except possibly for some set $\mathcal X^C$ such that $P^\mu(x\in\mathcal X^C)=0$ (see for example \citet{Hastie2009}). We thus have that $\widehat{\mathcal B}^*Q_m=\mathcal B^*Q_m$ almost surely with respect to $P^\mu$. Then it is easy to see that $||Q_m-\widehat{\mathcal B}^*Q_m||^2_{P^\mu}=||Q_m-\mathcal B^*Q_m||^2_{P^\mu}$. Finally, as $||F^{\mathcal D_V}-F^\mu||_\infty=0$, $||Q_m-\widehat{\mathcal B}^*Q_m||^2_{\mathcal D_V}=||Q_m-\widehat{\mathcal B}^*Q_m||^2_{P^\mu}$. This concludes the proof.
\end{proof}

Our last proposition is an analogue of Proposition \ref{prop1_append} in $L_\infty$ space. Because the Bellman operator is an $L_\infty$ contraction, our derived bounds based on the $L_\infty$ norm of the Bellman error is much tighter than those based on the $L_2$ norm of the Bellman error. 

\begin{proposition}\label{prop0_append}
Let $\widehat{\mathcal B}^*Q_m$ be an estimate of $\mathcal B^*Q_m$ and assume that $||Q_m-\widehat{\mathcal B}^*Q_m||_\infty\leq \epsilon$ and $||\widehat{\mathcal B}^*Q_m-\mathcal B^*Q_m||_\infty\leq \delta$. Then i) $||Q_m-Q^*||_\infty\leq \frac{1}{1-\gamma}(\epsilon+\delta)$ and ii) $||V^{\pi^*}-V^{\pi_{Q_m}}||_\infty\leq \frac{2}{(1-\gamma)^2}(\epsilon+\delta)$.
\end{proposition}

\begin{proof}
Suppose $||Q_m-\widehat{\mathcal B}^*Q_m||_\infty\leq \epsilon$ and  $||\widehat{\mathcal B}^*Q_m-\mathcal B^*Q_m||_\infty\leq \delta$. Then:

\begin{align*}
||Q_m-Q^*||_\infty =& ||Q_m-\mathcal B^*Q_m+\mathcal B^*Q_m-Q^*||_\infty \\
\leq& ||Q_m-\mathcal B^*Q_m||_\infty +||\mathcal B^*Q_m-Q^*||_\infty &&(\text{Subadditivity of $L_\infty$ norm})\\
=& ||Q_m-\mathcal B^*Q_m||_\infty + ||\mathcal B^*Q_m-\mathcal B^*Q^*||_\infty &&(Q^* \text{ is a fixed point of }\mathcal B^*) \\
\leq& ||Q_m-\mathcal B^*Q_m||_\infty + \gamma||Q_m-Q^*||_\infty &&(\mathcal B^* \text{ is a } \gamma\text{-contraction}) \\
\Rightarrow ||Q_m-Q^*||_\infty\leq& \frac{1}{1-\gamma}||Q_m-\mathcal B^*Q_m||_\infty \\
=& \frac{1}{1-\gamma}||Q_m-\widehat{\mathcal B}^*Q_m+\widehat{\mathcal B}^*Q_m-\mathcal B^*Q_m||_\infty \\
\leq& \frac{1}{1-\gamma}\left(||Q_m-\widehat{\mathcal B}^*Q_m||_\infty+||\widehat{\mathcal B}^*Q_m-\mathcal B^*Q_m||_\infty\right) &&(\text{Subadditivity of $L_\infty$ norm}) \\
\leq& \frac{1}{1-\gamma}\left(\epsilon+\delta\right).
\end{align*}

Moreover, it was proven in Lemma 1.11 of \citet{Agarwal2022} that $||V^*-V^{\pi_{Q_m}}||_\infty\leq \frac{2}{1-\gamma}||Q_m-Q^*||_\infty$. This concludes the proof. 
\end{proof}

\twocolumn

\section{Extended Details of non-Atari Experiments}
\label{append:non-atari}

\subsection{Extended Details of Toy Experiments}
\label{append:toy}

For our toy environment, we consider a stochastic MDP with four continuous states and a binary action $(\mathcal A=\{0,1\})$. Let $S_{t,j}$ denote the $j$th component of the observed state at time step $t$. The states are initialized from standard normal distributions and evolve as described below:

\begin{align*}
S_{t+1,1}=&\sqrt{x}S_{t,1}+(A_t-0.5)+\epsilon_{t},\\
S_{t+1,j}=& \sqrt{4x-2}S_{t,j}+\epsilon_{t,j}, \quad  j\in\{2,3,4\}, \\
\epsilon_{t}\sim& N(0,0.75-x), \epsilon_{t,j}\sim N(0,3-4x).
\end{align*}

The variable $x\in[0.5,0.75]$ controls the stochasticity of the environment: the MDP is deterministic when $x=0.75$, and entropy in the transition probabilities increase as $x$ decreases. We define the stochasticity variable as $\phi=0.75-x$, so that $\phi$ corresponds to the signal-to-noise ratio of $R_t$ and is zero when the MDP is deterministic. The reward is taken as the first component of the next state ($R_t=S_{t+1,1}$).  This MDP was constructed to be simple and interpretable. For example, it is easy to see that for any choice of $x$, the first and second univariate moments of the state remain constant over time $t$, $\pi^*(s)=1$ for all states $s$, $Q^*$ is a linear function, and  $S_{t+1,j},j\in\{2,3,4\}$ is independent of both the action and reward. 

The observed datasets used in our experiments consisted of $N=25$ trajectories of length $T=25$ each, which we split into training datasets of $N=20$ trajectories and validation datasets of $N=5$ trajectories. The behavioral policy was a completely random policy. We chose a discount factor of $\gamma=0.9$ during training. Our candidate set $\mathcal Q$ primarily consisted of estimates fit by FQI with polynomial ridge regression used as the regression algorithm. These estimates varied in the polynomial degree $d$ and $L_2$ weight penalty parameter $\lambda$. Also included in $\mathcal Q$ was the true optimal action-value function $Q^*$. For each $Q_m\in\mathcal Q$, SBV estimated $\mathcal B^*Q_m$ by training a polynomial ridge regression algorithm on the training set with degree-penalty combination $(d,\lambda)$ tuned to minimize error on the validation set. 

To estimate the true MSBE $||\mathcal B^*Q_m-Q_m||_{P^\mu}^2$, we approximated $\mathcal B^*Q_m$ using a KNN algorithm with a test set $\mathcal D_{T_1}$ of $2000$ trajectories each with $100$ time steps and generated using behavioral policy $\mu$, $k=100$ neighbors, and $S_{t,1}$ multiplied by two prior to calculating Euclidean distances (recall that $S_{t,k},2\leq k \leq 4$ are independent of $S_{t+1,1}=R_t$ and $A_{t-1}$ and thus should have diminished importance). We then approximated $P^\mu$ using an independent test set $\mathcal D_{T_2}$ of $1000$ trajectories each with $25$ time steps. We approximated $Q^*$ via linear regression with covariates $(A_0,S_{0,1})$ and outcomes $\sum_{t=0}^{100}\gamma^tR_t$, on a test set $\mathcal D_O$ generated by running optimal policy $\pi^*\equiv 1$ on $10,000$ simulated patients. Finally, the returns $\mathbb E_{\pi_{Q_j}}[\sum_{t=0}^\infty \gamma^tR_t]$ were estimated by running each greedy policy $\pi_{Q_j}$ on $1000$ simulated patients for $100$ time steps. Similar to what is often done on Atari experiments \cite{Agarwal2020}, we used a larger discount factor of $\gamma=1$ for evaluation, so that returns correspond to the expected sum of rewards over a trajectory.

\subsection{Extended Details of Bicycle Experiments}
\label{append:bike}

In the bicycle balancing control problem \cite{Randlov1998}, the MDP relates to a bicycle in a simulated physical environment which moves at constant speed on a horizontal plane. There are four relevant observed continuous state variables: the angle from vertical to bicycle $\omega$, its instantaneous rate of change $\dot\omega$, the displacement angle of the handlebars from normal $\theta$ and its instantaneous rate of change $\dot\theta$. While previous descriptions of the MDP include three additional states, these states are independent to the first four states and rewards, thereby being irrelevant to the control problem of interest, and two are usually considered hidden. A terminal state is reached when $|\omega|>12$, at which point the bicycle has fallen down. The actions are the torque $T\in\{-2,0,2\}$ applied to the handlebars and the displacement $d\in\{-0.02,0,0.02\}$ of the rider. The noise in the system is a uniformly distributed term in $[-0.02,0.02]$ added to the displacement. The reward function $R(s_t,a_t,s_{t+1})=-1I(|\omega_t|>12)$ corresponds to an optimal control problem that teaches an agent to balance the bicycle for as long as possible. 

The Bicycle MDP consists of highly nonlinear transition dynamics with sparse rewards and long-term consequences of taking certain actions. However, there is little noise in the environment, and the observed data typically consists of over $1000$ episodes. For our experiments, we chose a discount factor of $\gamma=0.99$ and a completely random behavioral policy $\mu(a|s)=1/9$. The environment, as well as the method it was created to benchmark, is well-known and has been extensively studied in the offline RL literature \cite{Lagoudakis2003, Ernst2005, Sutton2018}. More details about the MDP can be found in \citet{Randlov1998} and \citet{Ernst2005}. The value of a policy is measured as the expected number of time steps before a terminal state is reached when applying the policy online. 

For the Bicycle MDP, we generated 10 offline datasets consisting of $240$ episodes of $500$ time steps each, with $160$ episodes partitioned for training and the remaining $80$ episodes reserved for validation. Following \citet{Ernst2005}, Our candidate set $\mathcal Q$ were primarily fit by FQI with random forests used as the regression algorithm, with different hyperparameters for the number of training iterations $K$, the minimum node size $n_{\min}$ of the trees and the number of covariates $m_{\text{try}}$ randomly sampled to be considered at each split during tree growing. Also included in $\mathcal Q$ was the zero Q-function $Q\equiv 0$ corresponding to the behavioral policy (ties are broken randomly when calculating greedy policies) and polynomial Q-functions fit by minimizing EMSBE on the training set.  For each Bellman backup function $\mathcal B^*Q_m, Q_m\in\mathcal Q$ of interest, SBV used a separate random forest regression algorithm with the hyperparameters tuned to minimize MSE on the validation set. 

To maximize sample-efficiency when implementing WIS and FQE, we used the full dataset when calculating the sums in Equation \ref{eq:wis} as well as estimating the action-value function $Q^\pi$ and initial state distribution $d_0$ in Equation  \ref{eq:fqe}. Note that when $\mu$ is a completely random policy and $\pi$ is a deterministic policy, $\Pr_{\mu}(\prod_{v=0}^t\frac{\pi(A_v|S_v)}{\mu(A_v|S_v)}\neq 0)=|\mathcal A|^{-v}$. Moreover, for the Bicycle datasets a non-zero reward is observed only when a terminal state is reached and is usually only observed after nearly 100 zero rewards.
As the Bicycle datasets have under $250$ observed trajectories, it is easy from Equation \ref{eq:wis} why WIS would yield $\widehat J_{\mbox{WIS}}(\pi)=0$ for most policies $\pi$. As WIS gave identical ranks to all policies with standardized policy values ranging from $0$ to $1$, we assigned it a mean top-3 policy value of $0.5$ (corresponding roughly to random chance) for constructing Figure \ref{figure:bike}.

\subsection{Extended Details of mHealth Experiments}
\label{append:mhealth}

In the mHealth control problem \cite{Luckett2020}, the MDP relates to a disease process that evolves over time for patients in a simulated micro-randomized clinical trial \cite{Klasnja2015}. There are two observed continuous state variables, a binary action (indicating whether to apply treatment or do nothing) and a continuous reward function which trades-off the burden/cost of applying treatment with its effectiveness. The MDP has quadratic transition dynamics, dense rewards and short-term consequences of taking actions. However, the transitions are much noisier than those of the Bicycle MDP, reflecting the stochasticity typically observed in human health outcomes and behavior \cite{Kazdin2000}. Moreover, the observed data typically consist of under $100$ episodes, reflecting the smaller sample sizes that typically plague clinical trials and micro-randomized studies. All episodes in the observed datasets follow a completely random behavioral policy $\mu(a|s)=1/2$. Following \citet{Luckett2020}, an initial burn-in period of $50$ time steps were applied prior to both data generation and online policy evaluation to ensure stationary. The environment, as well as the method it was created to benchmark, is well-known and has been extensively studied in healthcare RL \citep{Zhu2020, Tsiatis2021, Liao2021, Yu2021, Liao2022, Shi2022}. More details can be found in \citet{Luckett2020}. The value of a policy is measured as the expected reward when applying the policy online for $100$ time steps. 

For the mHealth MDP, observed datasets consisted of $30$ episodes of $25$ time steps each, with $24$ episodes partitioned for training and the remaining $6$ episodes reserved for validation. Following \cite{Luckett2020}, Our candidate Q-functions $\mathcal Q$ were primarily polynomial functions fit by ridge-penalized least square policy iteration (LSPI) \cite{Lagoudakis2003}, with different hyperparameters for the degree of the polynomial $d$ and the $L_2$ weight/ridge penalty parameter $\lambda$. Also included in $\mathcal Q$ was polynomial Q-functions fit by minimizing EMSBE on the training set, and random forest Q-functions fit by FQI and varying in the hyperparameters $n_{\min}$ and $m_{\text{try}}$ (see Appendix \ref{append:bike} for details on these parameters).  

For each Bellman backup function $\mathcal B^*Q_m, Q_m\in\mathcal Q$ of interest, SBV used a separate regression algorithm tuned to minimize MSE on the validation set. The space of possible regression algorithms considered included random forest algorithms with different hyperparameters for $n_{\min}$ and $m_{\text{try}}$, and polynomial ridge regression with different hyperparameters for $d$ and $\lambda$. To maximize sample-efficiency when implementing WIS and FQE, we used the full dataset when calculating the sums in Equation \ref{eq:wis} as well as estimating the action-value function $Q^\pi$ and initial state distribution $d_0$ in Equation  \ref{eq:fqe}.

The scripts present in our repository \url{https://github.com/jzitovsky/SBV} can reproduce all the results from our paper. For those interested, these scripts contain information on the exact grid of hyperparameters used to tune the $Q^*$ estimation algorithm and SBV for the toy, mHealth and bicycle datasets. All scripts have some element of multi-threading. Non-Atari experiments were conducted using 2.50 GHz Intel CPU cores from our university's computing cluster. The most computationally-intensive experiments were from the bicycle environment: Here implementing tree-based FQI for all configurations in our hyperparameter grid was the primary computational bottleneck, and took approximately three hours per dataset with 30 CPUs.

\section{Extended Details of Atari Experiments}
\label{append:atari}

\subsection{Q-Learning Configurations for Atari}
\label{append:dqn_deep}

There were two training configurations explored when running Q-learning on Asterix, Seaquest and Breakout: a shallow and a deep configuration. For Pong, only the shallow configuration was used to speed-up our experiments. The shallow configuration uses the standard Nature DQN architecture and training algorithm \cite{Mnih2015} but with an Adam optimizer instead of an RMSProp optimizer. Most of the details of this configuration can be found in \citet{Agarwal2020} (they call this configuration ``Offline DQN (Adam)") and thus we do not repeat them here. We did however change the batch size from 32 to 128 to speed-up training. We should also mention that our definition of ``training iteration" differs from that of \citet{Agarwal2020} in that it involves more weight updates to the Q-network overall (640,000 training steps instead of 250,000 training steps) as well as more frequent re-loading of data from disk into memory to reduce the correlation of the mini-batches sampled for training. These changes apply to both the shallow and deep configurations. The scripts we wrote to run DQN with these configurations made heavy use of the Dopamine library \cite{dopamine}.

The deep configuration uses a deeper and more complex Q-Network architecture that features 16 convolution layers as well as max pooling layers, leaky ReLU activation layers \cite{Geron2019} and skip connections \cite{He2016}. The architecture here is just a high-performing architecture for the Bellman network that doesn't use batch normalization \cite{Ioffe2015}, according to validation Bellman backup MSE. While architectures that used batch normalization performed better as Bellman networks, they made DQN training prohibitively slow. We tweaked the learning rate and \textit{target update frequency} (as defined in \citet{Agarwal2020}) to 2.5e-5 and 32,000, respectively, so as to improve training stability.  Finally, we used double Q-learning targets \cite{Hasselt2016} to reduce overestimation bias, though we found that this modification had a minor effect relative to the other modifications discussed. 

The network graph of the deep configuration architecture can be found in Figure \ref{figure:dqn_deep}. We use Leaky ReLU activation layers in place of the standard ReLU activation layers to mitigate issues with vanishing gradients and dying ReLUs \cite{Geron2019}. A pre-processing layer first divides each element of the input state $s \in\mathcal S$ by $255$ to scale the elements to the $[0,1]$ range where $\mathcal S=\{0,1,...,255\}^{84\times 84 \times 4}$. After the pre-processing layer, the network starts with two convolutional layers and a pooling layer, with the first convolutional layer having a kernel size of $5\times 5$ and both convolutional layers having a filter size of $48$. After this are three units which we will call \textit{(DQN) convolutional stacks}. Each DQN convolutional stack consists of four convolution layers, all of which have the same number of filters, with a pooling layer inserted after the first two convolution layers. Moreover, the input to the stack is connected to the output of the pooling layer of the stack via a skip connection similar to other ResNet architectures \cite{Chollet2017}, and the input to the third convolution layer is connected to the output of the final convolution layer of the stack via a skip connection as well. The three convolution layers have filter sizes $72$, $144$ and $248$, in ascending order. After the three convolution stacks is a global average pooling layer and a dense layer with units equal to the number of possible actions $|\mathcal A|$. The output of this layer gives the Q-network estimates $Q_\theta(s,a),a\in\mathcal A$ of the true optimal action-values $Q^*(s,a),a\in\mathcal A$ for the given input state $s\in\mathcal S$.

\setcounter{figure}{0}
\begin{figure*}[t]
\centering
\includegraphics[width=\linewidth]{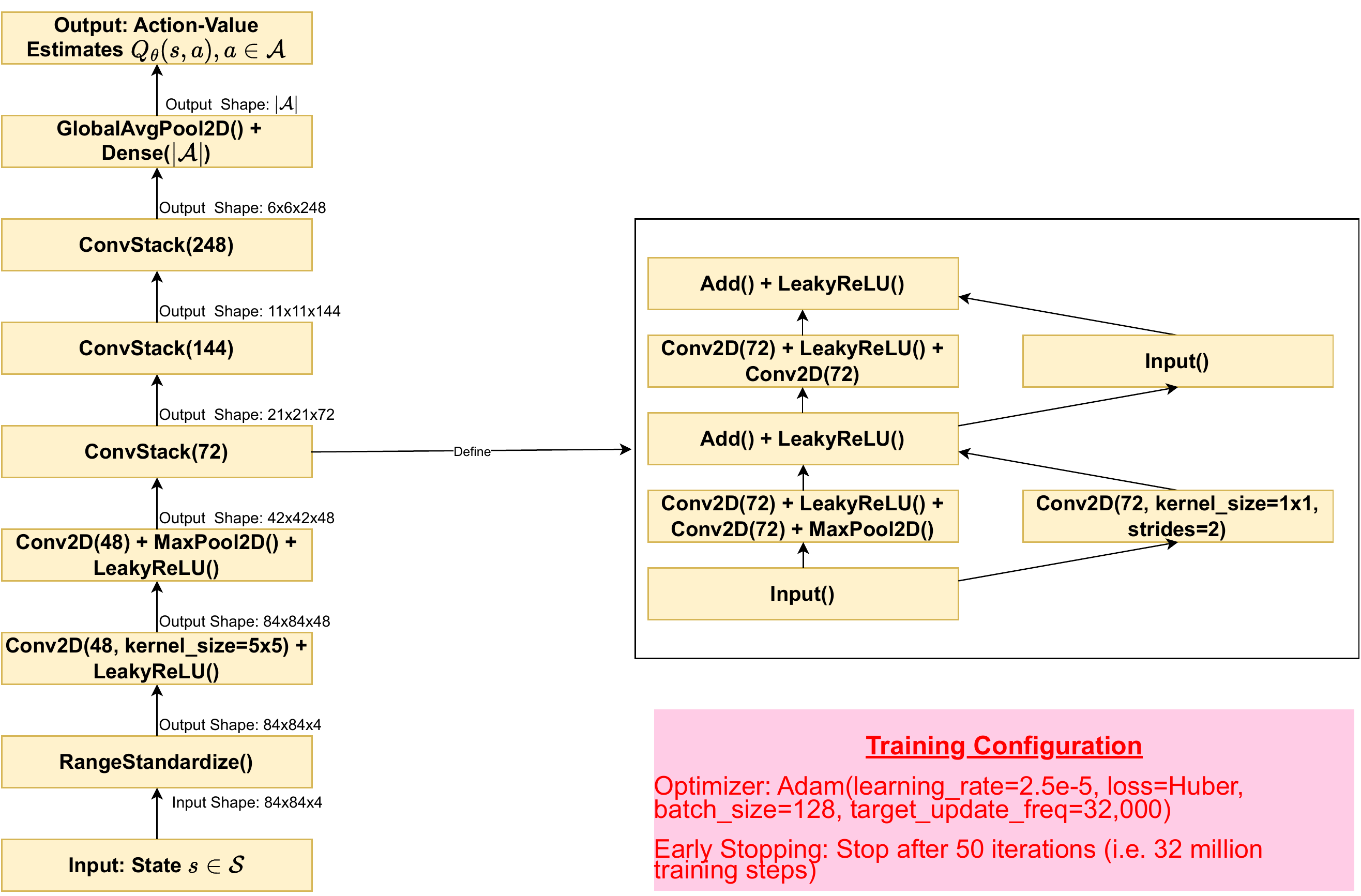}
\caption{Network Graph of the Deep Configuration of DQN. Unlabeled arrows represent feed-forward connections. Unless otherwise specified, all layers use the default parameters specified by TensorFlow v2.5.0 \cite{tf2015}, with the following exceptions:  1) convolutional layers use 3x3 kernels, zero padding (padding=``SAME") and a He uniform weight initialization \cite{He2015}; 2) max pooling layers use 3x3 kernels, a vertical and horizontal stride of 2 and apply zero padding (padding=``SAME").}
\label{figure:dqn_deep}
\vspace{-0.3cm}
\end{figure*}

\subsection{Implementation Details of SBV for Atari}
\label{append:sbv_atari}

We used two training configurations for the Bellman network: A simpler configuration for Pong and a more complex one for the other three games. The more complex architecture was tuned to minimize validation MSE across the Breakout, Asterix and Seaquest datasets and performs significantly better (in terms of validation MSE) than previously-used architectures and configurations such as Nature \cite{Mnih2015} and IMPALA \cite{Espeholt2018}. In theory, we could have used a different training configuration for every dataset, or even for every Bellman backup function, but we avoided doing this to simplify the experiments. The complex Bellman network configuration is more reminiscent of the deep neural networks applied to modern image classification problems such as ImageNet \cite{ILSVRC15} than the small networks typically employed in deep RL and includes commonly-used design choices in deep learning, such as batch normalization \cite{Ioffe2015} and residual connections \cite{He2016}, to maximize generalization.  The simpler configuration applied to Pong was utilized because it achieved almost the same performance on Pong as the more complex configuration while requiring significantly less computational resources. The more complex configuration is described below. Details of the simpler configuration are not as important and are not described here.

An illustration of the full network graph can be found in Figure \ref{figure:Bellman_net}. As is standard in supervised deep convolutional neural networks, we follow every convolutional layer with a batch normalization layer except for those preceding pooling layers: these convolutional layers are followed first by a pooling layer and then by a batch normalization layer.  We use Leaky ReLU activation layers in place of the standard ReLU activation layers to mitigate vanishing gradient and dying ReLU problems \cite{Geron2019}.

A pre-processing layer first scales each pixel value present in the input state $s\in\mathcal S$ to the $[0,1]$ range, as is standard for image-based neural networks. After the pre-processing layer, the network starts with three convolutional layers and a pooling layer. After this are four units which we will call \textit{(Bellman) convolutional stacks}. Each Bellman convolutional stack consists of three convolutional layers, a pooling layer and a squeeze-and-excitation (SE) block \cite{Hu2018}. The SE block in this case consists of a global average pooling layer, a hidden dense bottleneck layer with $1/4$th as many hidden neurons as the input width (+ batch norm + leaky ReLU), and an output layer with softplus activation. The input signal is connected to the output of the SE block via a skip connection similar to other SE-ResNet architectures \cite{Hu2018}. After the four convolutional stacks are two convolutional layers that use depthwise separable convolutions \cite{Chollet2017} followed by an SE block. A skip connection connects the input of the first depthwise separable layer to the output of this final SE block. 

The number of feature maps per layer increases with depth, increasing from $48$ feature maps for the first convolutional layer to $240$ for the final depthwise separable layers.   Following the final SE block is a global average pooling to reduce the number of parameters and a dense layer with units equal to the number of possible actions $|\mathcal A|$. The output of the network $\mathcal B_\phi(s,a)$ is the dot product between the output of this dense layer and a one-hot transformation of the inputted action $a\in\mathcal A$, and is an estimate of the Bellman backup $(\mathcal B^*Q_m)(s,a)$ where $Q_m$ is a candidate Q-function that we are trying to evaluate.

\begin{figure*}[htbp]
\centering
\includegraphics[width=\linewidth]{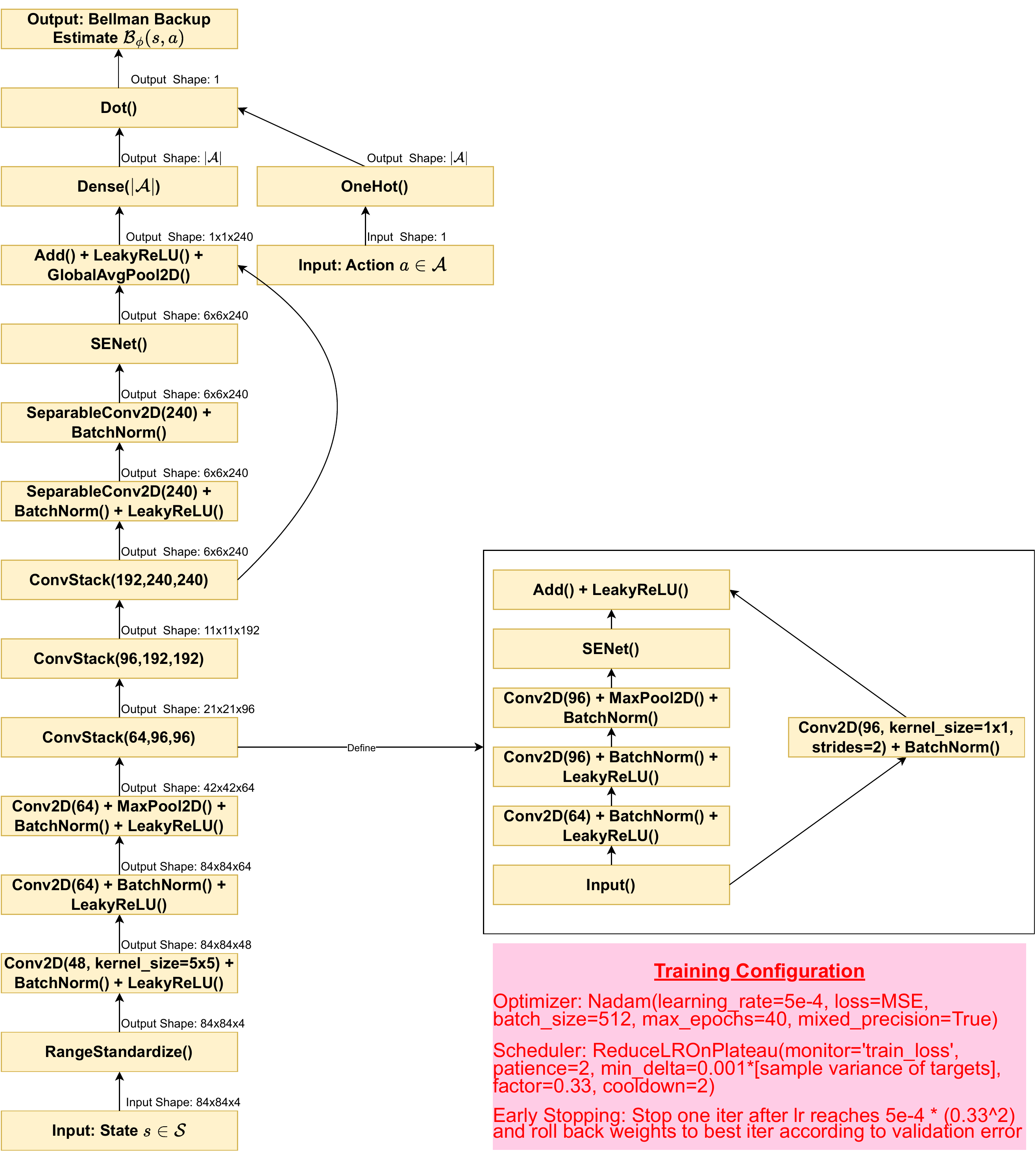}
\caption{Network Graph of Bellman Network. Unlabeled arrows represent feed-forward connections. Unless otherwise specified, all layers use the default parameters specified by TensorFlow v2.5.0 \cite{tf2015}, with the following exceptions:  1) convolutional layers use 3x3 kernels, zero padding (padding=``SAME"), no bias and a He uniform weight initialization \cite{He2015}; 2) max pooling layers use 3x3 kernels, a vertical and horizontal stride of 2 and apply zero padding (padding=``SAME"); 3) SENet() uses an identical architecture to the squeeze-and-excitation units in \citet{Hu2018}, except the bottleneck layer uses a reduction factor of 4 instead of 16 and a leaky ReLU activation instead of the standard ReLU, and the output layer uses a softplus activation instead of sigmoid.}
\label{figure:Bellman_net}
\vspace{-0.3cm}
\end{figure*}

The network weights are trained on $\mathcal D_T$ to minimize MSE via the NAdam optimization algorithm \cite{Dozat2016}. To speed up computations, we enabled the mixed precision feature of TensorFlow \cite{tf2015} for the computations and used a batch size of $512$. For the learning rate scheduler, we used an initial learning rate of $5\times 10^{-4}$ and multiplied the learning rate by a factor of $0.33$ whenever training loss did not improve by at least $0.001\times\text{var}(\mathcal B_{\mathcal D_T}Q_m(s_t,a_t)))$ over $2$ consecutive iterations where $\text{var}(\mathcal B_{\mathcal D_T}Q_m(s_t,a_t))$ is the sample variance of the targets within $\mathcal D_T$ being used to train $\mathcal B_\phi$. Training was terminated after the first iteration with the learning rate equal to $5\times 10^{-4}\times 0.33^2$. The performance of the network after each epoch was evaluated by calculating the MSE from validation set $\mathcal D_V$, and only the weights from the iteration with the lowest validation loss were saved. 

For Asterix, Seaquest and Pong, rather than loading all the data into memory at once, we sharded the data into $100$ separate files and performed a training epoch by repeatedly loading $10$ different files into memory. This reduced memory costs substantially, and was found to not significantly degrade computational or generalization performance on most games. An exception was for Breakout, where we noticed modest gains in generalization error by loading the full data into memory all at once when conducting training.  

Denote $Q_{\theta_k}$ as the trained Q-network after the $k$th iteration for a particular DQN  configuration. We wish to estimate the Bellman backup functions $\mathcal B^*Q_{\theta_k}, 1\leq k \leq 50$. Let $\mathcal B_{\phi_k}$ denote the trained Bellman network for estimating the $\mathcal B^*Q_{\theta_k}$. We perform two additional tricks to speed up computations. First, for most iterations, we initialize the trained weights of $\mathcal B_{\phi_{k+1}}$ as $\phi_k$ before conducting training, similar to Algorithm \ref{alg:dqn_sbv}. This can be thought of as a kind of transfer learning and greatly speeds up computation. We do, however, occasionally re-randomize initial weights for an iteration via He and Xavier initialization schemes \citep{Glorot2010, He2015} to avoid overfitting. The number of consecutive iterations $I$ where weights are transferred from previous iterations before re-randomizing the weights is a hyperparameter that we tuned to minimize validation MSE. In our case, we set $I=5$ for Pong, $I=3$ for Seaquest and Breakout, and $I=1$ for Asterix. 

The second trick we apply is stopping training once the top-5 policies according to SBV fail to change for five consecutive iterations. By stopping after it is clear that the optimal iteration has been found, we avoid wasting time and resources evaluating Q-functions that aren't going to be selected anyway. As calculating Spearman correlations with policy returns would require running SBV on all 100 Q-functions for all 12 datasets, this trick meant that Spearman correlations between SBV and policy returns would not be available. Given that Spearman correlation is not the most important metric anyway for reasons discussed in previous sections, we felt that calculating these Spearman correlations was not necessary for the Atari games and their exclusion was worth the improved compute. 

\subsection{Other Implementation Details for Atari}
\label{append:wis_atari}

As WIS requires knowledge of the unknown behavioral policy $\mu$, we estimated it using a neural network approximator $\mu_\beta$ with trainable parameter vector $\beta$ which we dub our \textit{propensity network}. $\mu_\beta$ was trained on $\mathcal D_T$ so as to maximize log-likelihood $\sum_{(a,s)\in \mathcal D}[\log \mu_\theta(a|s)]$ while its log-likelihood on $\mathcal D_V$ was used to tune hyperparameters and evaluate performance. We ended up using the same configuration for the propensity network as that of the Bellman network (Figure \ref{figure:Bellman_net}), except with the loss function changed to categorical cross-entropy, the output activation layer changed to softmax, and the initial learning rate occasionally reduced to avoid divergence. We found that the Bellman network architecture achieved higher validation log-likelihood than other custom architectures explored as well as architectures used in previous literature such as IMPALA \cite{Espeholt2018}. From our propensity network $\mu_\beta$, we calculated Equation \ref{eq:wis} on $\mathcal D_V$ to avoid bias issues with using the same dataset to train and evaluate the Q-network and propensity network \cite{Nie2022}. We found that not using such data partitioning yielded worse performance. The code repository associated with \citet{Kostrikov2020} helped write our scripts for implementing WIS. When implementing FQE (Equation \ref{eq:fqe}), we estimated the action-value function on $\mathcal D_T$ by modifying the shallow DQN training algorithm to perform policy evaluation instead of optimization, and estimated the initial state distribution using $\mathcal D_V$. Similar to SBV, the neural network for estimating the value of a DQN configuration after a specific iteration was oftentimes initialized using the trained neural network from the previous iteration to speed-up computation. 

For generating all tables and figures related to Atari, the value of a policy is measured as the expected sum of rewards from the beginning of a game to its termination, or the expectation of returns $\mathbb E_\pi[\sum_{t\geq 0}\gamma^tR_t]$ observed from applying the policy online but with $\gamma=1$. For those interested in more details related to our Atari experiments, including the SBV implementation on Pong and our ablation experiments in Section \ref{sec:ablations}, please see the relevant scripts in our repository \url{https://github.com/jzitovsky/SBV}. 

\subsection{Computational Performance on Atari}
\label{append:compute}

Atari experiments were conducted using a mix of A100 and V100 GPUs from both our university's computing cluster and GCP virtual machines. Running the deep DQN configuration for 50 iterations took around six days with a single A100 GPU. However, it is worth noting that the best Q-function was almost always obtained in under 25 iterations. SBV and FQE both had to fit a separate neural network for every Q-function being evaluated, making computation quite intensive. With four A100s and four V100s (or with six A100s), running SBV on a single dataset would usually take around a week, while running FQE on the same dataset would usually take around two weeks. SBV had a tendency to output a very large criterion when evaluating Q-functions trained using a large number of iterations. This allowed us to ``early stop" SBV after its criterion became too large and avoid evaluating all Q-functions without changing the top-5 Q-functions selected by SBV (see Section \ref{append:sbv_atari}). On the other hand, FQE did not share this tendency, making such an early stopping procedure difficult to apply to FQE. The ability to avoid evaluating all Q-functions with SBV is one reason why it required less computational resources than FQE. Another reason why SBV required less compute is that it was stopped training the Bellman network when validation loss stopped improving. No similar validation metric exists for FQE, and with no ability to cease neural network training early, FQE ended up using many more epochs than SBV. Finally, while SBV used a much larger and deeper architecture than FQE, this was offset by a larger batch size and not having to keep track of a separate target network during training. 

We can easily derive and compare computational complexity of SBV to other baselines. Let $V$ and $T$ be the number of transitions in the validation and training set respectively and $M$ be the number of Q-functions we wish to evaluate. Suppose DQN approximates the value function, SBV approximates the backups and FQE approximates the value function with a deep neural network, and each such network has $P$ parameters and is trained for $K$ epochs. The number of operations for evaluating all Q-functions is $O(TPKM)$ for FQE and SBV vs. $O(VPM)$ for EMSBE. While FQE and SBV have similar computational complexity in theory, SBV is faster to execute in practice for reasons discussed above. However, the EMSBE is much faster than SBV both in theory and in practice, with computation times negligible compared to FQE and SBV. EMSBE does not fit a regression algorithm like SBV, but rather only calculates a sample mean, and can be computed almost instantly. For noisy MDPs, the EMSBE can select arbitrarily poor policies (see Figure \ref{figure:noise}), and in these cases its computational performance is secondary. However, for low-noise MDPs, the EMSBE will only be slightly biased, and because it is faster it may be preferred for datasets where SBV would be slow to implement. Moreover, in deterministic MDPs, the validation EMSBE is unbiased and would likely perform better than SBV due to not incurring any function approximation error. 

An important diagnostic when running SBV is to compare the Bellman backup MSE to the EMSBE on the validation set $\mathcal D_V$. It is easy to see that the EMSBE upper bounds the Bellman backup MSE when the Bellman network is close to the true Bellman backup. Therefore, if validation MSE of the regression estimator is larger than validation EMSBE, this means that the Bellman network has room for improvement and it may be worth changing the architecture or other training hyperparameters. Also recall from Proposition \ref{prop2} that the EMSBE is positively biased for the true MSBE. Thus, if the SBV estimate of the MSBE is larger than the EMSBE, this would also indicate a problem.

\onecolumn
\section{The Relationship between SBV and Other Algorithms}
\label{append:relation}

For the subsections below, we assume readers have already read Section \ref{sec:sbv}.

\subsection{The Relationship between SBV and FQI}
\label{append:fqi}

For this section, we assume readers are already familiar with Fitted Q-Iteration \cite{Ernst2005}. Recall from Proposition \ref{prop3} that the expected value of the (variable-target) EMSBE (Equation \ref{eq:emsbe}) is approximately equal to the (variable-target) MSBE (Equation \ref{eq:msbe}) plus a bias term related to the variance of the targets. Therefore, the EMSBE will favor Q-functions with less-variable targets more than the MSBE, making the EMSBE a poor proxy for the MSBE for model selection. Likewise, the loss function minimized by FQI at every iteration can be considered a fixed-target EMSBE, and this will also be biased for the corresponding fixed-target MSBE with bias proportional to the target variance. The main difference, however, is that the targets do not change when evaluating different Q-functions to perform an FQI update. This means that FQI updates will not be affected by the bias in the fixed-target empirical Bellman errors, and minimizing over the fixed-target EMSBE will give a similar solution as minimizing over the fixed-target MSBE. Therefore, the fixed-target EMSBE will still be a good proxy for the fixed-target MSBE for model training. 

To fix ideas, suppose we had an estimate $\widehat Q^*$ for $Q^*$ obtained by running FQI for $K$ iterations over functional class $\mathcal F$. In other words, we iteratively update a Q-function as $Q^{(k+1)}\leftarrow\text{argmin}_{f\in\mathcal F}\mathbb E_{\mathcal D_T}\left[\left(r+\gamma\max_{a'}Q^{(k)}(s',a')-f(s,a)\right)^2\right]$ for iterations $k=0,1,...,K-1$, and then set $\widehat Q^*=Q^{(K)}$. It can be shown that:\footnote{This approximation holds exactly for transitions in $\mathcal D_V$, but for transitions in $\mathcal D_T$, the fact that $Q^{(k)}$ was trained on $\mathcal D_T$ introduces an additional bias term that depends on the complexity of $\mathcal F$, the iteration number $k$ and the sample size \cite{Antos2007fqi}. For simplicity, we assume that the sample size is sufficiently large relative to the complexity of $\mathcal F$ such that this approximation is accurate.}
\begin{align*}
\mathbb E\left[\left(r+\gamma\max_{a'}Q^{(k)}(s',a')-Q^{(k+1)}(s,a)\right)^2\right]\approx& ||\mathcal B^*Q^{(k)}-Q^{(k+1)}||^2_{P^\mu}\\
&+\mathbb E_{(S_t,A_t)\sim P^\mu}\left[\text{Var}\left(R_t+\gamma\max_{a'}Q^{(k)}(S_{t+1},a')|S_t,A_t\right)\right].
\end{align*}
It is easy to see that $\text{argmin}_{f\in\mathcal F}\mathbb E_{\mathcal D_T}\left[\left(r+\gamma\max_{a'}Q^{(k)}(s',a')-f(s,a)\right)^2\right]\approx \text{argmin}_f ||\mathcal B^*Q^{(k)}-f||^2_{P^\mu}$. This is because the targets $r+\gamma \max_{a'}Q^{(k)}(s',a')$ are not affected by the minimizing function $f$ and thus the bias term $\mathbb E_{(S_t,A_t)\sim P^\mu}\left[\text{Var}\left(R_t+\gamma\max_{a'}Q^{(k)}(S_{t+1},a')|S_t,A_t\right)\right]$ does not affect the optimization. Therefore, when applying FQI updates, the empirical fixed-target Bellman errors will be a good substitute for the true fixed-target Bellman errors.

To understand why SBV is still needed to tune hyperparameters of FQI, suppose we now wished to evaluate the quality of $\widehat Q^*$ in estimating $Q^*$. This would make sense, for example, if we wished to compare $\widehat Q^*$ to other proposed estimates of $Q^*$. We could use SBV to evaluate $\widehat Q^*$ by estimating its MSBE $||\widehat Q^*-\mathcal B^*\widehat Q^*||^2_{P^\mu}$, as it can be shown that the MSBE upper bounds estimation error $||\widehat Q^*-Q^*||^2_{P^\mu}$ (Proposition \ref{prop1}). However, instead of SBV, suppose that we were to assess $\widehat Q^*$ using the empirical loss functions minimized throughout FQI but applied to the validation set, $\sum_{k=0}^{K-1}\mathbb E_{\mathcal D_V}\left[\left(r+\gamma\max_{a'}Q^{(k)}(s',a')-Q^{(k+1)}(s,a)\right)^2\right]$. It can be shown that $\sum_k ||\mathcal B^*Q^{(k)}-Q^{(k+1)}||^2_{P^\mu}$ upper bounds estimation error of FQI in a manner similar to the MSBE \cite{Munos2005}, and to the untrained eye, $\sum_k\mathbb E_{\mathcal D_V}\left[\left(r+\gamma\max_{a'}Q^{(k)}(s',a')-Q^{(k+1)}(s,a)\right)^2\right]$ may seem like a reasonable proxy for $\sum_k ||\mathcal B^*Q^{(k)}-Q^{(k+1)}||^2_{P^\mu}$. 

The problem here is that when using $\sum_k\mathbb E_{\mathcal D_V}\left[\left(r+\gamma\max_{a'}Q^{(k)}(s',a')-Q^{(k+1)}(s,a)\right)^2\right]$ to evaluate $\widehat Q^*$, this term will depend not only on $\sum_k ||\mathcal B^*Q^{(k)}-Q^{(k+1)}||^2_{P^\mu}$ but also on the variance of the targets $\sum_k \mathbb E_{(S_t,A_t)\sim P^\mu}\left[\text{Var}\left(R_t+\gamma\max_{a'}Q^{(k)}(S_{t+1},a')|S_t,A_t\right)\right]$. Even if $\sum_k ||\mathcal B^*Q^{(k)}-Q^{(k+1)}||^2_{P^\mu}=0$ (indicating that $\widehat Q^*\approx Q^*$), our proposed criterion may still be very large because the targets have high variance. This also means that when evaluating two candidates $Q_1$ and $Q_2$ estimated by FQI, such a procedure may choose $Q_1$ over $Q_2$ even though $Q_2$ yields a smaller value of $\sum_k ||\mathcal B^*Q^{(k)}-Q^{(k+1)}||^2_{P^\mu}$, because the targets observed throughout training for $Q_1$ have smaller variance.

Therefore, while the errors used by FQI are effective for model training, they are less effective for model evaluation and selection, necessitating the use of SBV. This also highlights the key difference between fixed-target Bellman errors estimated by FQI and variable-target Bellman errors estimated by SBV: empirical fixed-target Bellman errors are good proxies for the true fixed-target Bellman errors with FQI because the targets are treated as fixed in the relevant optimization problems. However, the empirical variable-target Bellman errors used by the EMSBE (Equation \ref{eq:emsbe}) are not good proxies for the true variable-target Bellman errors, because the targets are are not fixed when evaluating different Q-functions, but rather will change depending on the Q-function being evaluated. 

Now let's suppose we were to run a $K+1$th iteration of FQI and assess $\widehat Q^*$ using $\mathbb E_{\mathcal D_V}[(Q^{(K)}(s,a)-Q^{(K+1)}(s,a))^2]$. This is equivalent to our SBV estimate $\mathbb E_{\mathcal D_V}[(\widehat Q^*(s,a)-\widehat{\mathcal B}^*\widehat Q^*(s,a))^2]$ when $\widehat{\mathcal B}^*\widehat Q^*(s,a)$ is estimated by minimizing the Bellman backup MSE (Equation \ref{eq:mse_backup}) over the same functional class $\mathcal F$ used for FQI. The problem with using this quantity to assess $\widehat Q^*$ is that it assumes $Q^{(K+1)}\approx \mathcal B^*\widehat Q^*(s,a)$. SBV makes no such assumption. Instead, $\mathcal B^*\widehat Q^*(s,a)$ is estimated by minimizing Bellman backup MSE over functional class $\mathcal G$, where $\mathcal G\neq \mathcal F$ in general and is tuned to minimize validation error $\mathbb E_{\mathcal D_V}[(r+\gamma\max_{a'}\widehat Q^*(s',a')-\widehat{\mathcal B}^*\widehat Q^*(s,a))^2]$, thus ensuring that $\widehat{\mathcal B}^*\widehat Q^*\approx \mathcal B^*\widehat Q^*$. It is also worth noting that SBV can be used to evaluate estimates of $Q^*$ generated by training algorithms other than FQI, such as Least Square Policy Iteration \citep{Lagoudakis2003}. 

See also Appendix \ref{append:Bellman_q}.

\subsection{The Relationship between SBV and BErMin}
\label{append:bermin}

For this section, we assume readers are already familiar with BErMin \cite{Farahmand2010}. SBV first splits the observed data $\mathcal D$ into a training set $\mathcal D_T$ and a validation set $\mathcal D_V$. We then obtain a set of $M$ estimates of $Q^*$, $\mathcal Q=\{Q_1,...,Q_M\}$ by running $M$ different offline RL algorithms on $\mathcal D_T$, and estimate their Bellman backups $\mathcal B^*Q_1,...,\mathcal B^*Q_M$ as $\widehat{\mathcal B}^*Q_1,...,\widehat{\mathcal B}^*Q_M$ by running $M$ regression algorithms on $\mathcal D_T$. Finally, we perform offline model selection using assessment criterion $\text{SBV}(Q_m)=||Q_m-\widehat{\mathcal B}^*Q_m||_{\mathcal D_V}^2$, which aims to be an accurate point estimator for $\text{MSBE}(Q_m)=||Q_m-\mathcal B^*Q_m||^2_{P^\mu}$. While our theoretical results are restricted to infinite-data settings, we demonstrated strong empirical performance on a diverse set of challenging problems.

In contrast, BErMin splits the observed data $\mathcal D$ into two independent training sets $\mathcal D_{T_1}$ and $\mathcal D_{T_2}$ and a validation set $\mathcal D_{V}$. As before, a set of $M$ estimates of $Q^*$, $\mathcal Q=\{Q_1,...,Q_M\}$ are obtained by running $M$ different offline RL algorithms on $\mathcal D_{T_1}$. Now, however, their Bellman backups $\mathcal B^*Q_1,...,\mathcal B^*Q_M$ are estimated as $\widehat{\mathcal B}^*Q_1,...,\widehat{\mathcal B}^*Q_M$ by running $M$ regression algorithms on a separate training set $\mathcal D_{T_2}$. In other words, BErMin assumes the data used to generate the Q-functions is independent of that used to estimate their Bellman backups. Moreover, we perform offline model selection using assessment criterion $\text{BErMin}(Q_m)=\frac{1}{(1-a)^2}||Q_m-\widehat{\mathcal B}^*Q_m||_{\mathcal D_{V}}^2+\tilde b_m$, where $||\mathcal B^*Q_m-\widehat{\mathcal B}^*Q_m||^2_{P^\mu}\leq \tilde b_m$ with probability $1-\delta/2M$ and $a\in(0,1)$ is a tuning parameter determining the relative weight of $\tilde b_m$ relative to $||Q_m-\widehat{\mathcal B}^*Q_m||_{\mathcal D_{V}}$. $\text{BErMin}(Q_m)$ aims to be an accurate upper bound for $\text{MSBE}(Q_m)=||Q_m-\mathcal B^*Q_m||^2_{P^\mu}$. BErMin also extend this assessment criterion to deal with the case where $M$ is very large (potentially infinite) and where the user can give greater prior weight to certain Q-functions over others (e.g. based on complexity). For simplicity our description is restricted to the case where $M$ is finite and small and where no Q-functions are favored to others a priori. With $Q_{\widehat m}$ the Q-function selected by BErMin, it can be shown that $\text{MSBE}(Q_{\widehat m})\leq 4(1+a)(1-a)^{-2}\min_{1\leq m \leq M}\left[2\text{MSBE}(Q_m)+3\tilde b_k\right]+O(n^{-1}\log(\delta^{-1}))$ with probability $1-\delta$. Therefore, assuming $\max_{1\leq m \leq M}\tilde b_m\to 0$ at a sufficiently fast rate with dataset size (i.e. our upper bound becomes tighter and our regression algorithm becomes more accurate with sample size), BErMin is theoretically guaranteed to select a Q-function with MSBE close to the Q-function from $\mathcal Q$ with smallest MSBE, up to a constant (see also Theorem 3 of \citet{Farahmand2010}). 

BErMin thus has much stronger finite-sample theoretical guarantees than what was derived for our method. However, to our knowledge, BErMin has never been successfully applied to deep RL benchmarks like Atari, despite being published in 2011. We believe there are a few reasons for this. First, BErMin requires specification of tight probabilistic bounds on the excess risk of the regression algorithm, or on $||\mathcal B^*Q_m-\widehat{\mathcal B}^*Q_m||^2_{P^\mu}$. While an upper bound on the generalization MSE $\mathbb E_{(s,a)\sim P^\mu,s'\sim T(\cdot|s,a)}\left[\left(R(s,a,s')+\gamma\max_{a'}Q_m(s',a')-Q_m(s,a)\right)^2\right]$ can be easily obtained using a held-out validation set, bounding the excess risk $||\mathcal B^*Q_m-\widehat{\mathcal B}^*Q_m||^2_{P^\mu}$ is more difficult as it involves the unknown Bellman backup $\mathcal B^*Q_m$. For example, even if we wanted to apply BErMin to Atari, it is not clear how we would derive excess risk bounds for our trained neural network given in Figure \ref{figure:Bellman_net} that are tight enough to be useful without making overly restrictive assumptions. To the best of our knowledge, finite-sample excess risk bounds for neural networks that have been derived in previous work make assumptions on the dependence between observations, class of architectures considered and function class that the true Bellman backup belongs to that would make such bounds inapplicable to our settings. Most such bounds are also likely far too loose to be useful for our observed sample size. 

Moreover, BerMin requires further partitioning the training set $\mathcal D_T$ into separate datasets $\mathcal D_{T_1}$ and $\mathcal D_{T_2}$ to estimate the Q-functions and Bellman backups, respectively, which essentially halves the amount of data that can be used for both. In contrast, SBV uses the same data $\mathcal D_T$ to estimate the Q-functions and backups, and this leads to both estimators performing better, as they can use more data (see Section \ref{sec:ablations}). Finally, applying SBV (or BErMin) to deep RL settings requires building and tuning a regression neural network to minimize validation MSE, and this was non-trivial for Atari. For example, the typical DQN \cite{Mnih2015} and Impala-CNN \cite{Espeholt2018} architectures were ineffective at achieving small MSE and we needed components common in deep learning but not in deep RL, such as batch normalization, leaky ReLU activations and squeeze-and-excitation units (see Appendix \ref{append:sbv_atari} for details).

\subsection{The Relationship between SBV and BVFT}
\label{append:bvft}

For this section, we assume readers are already familiar with BVFT \cite{Xie2021}. For each candidate $Q_m\in\mathcal Q$, SBV uses assessment criterion $\text{SBV}(Q_m)=||Q_m-\widehat{\mathcal B}^*Q_m||_{\mathcal D_V}^2$ where $\widehat{\mathcal B}^*Q_m$ uses a regression algorithm trained on $\mathcal D_T$ to solve $\text{argmin}_{f}\sum_{(s,a,r,s')\in\mathcal D_T}(r+\gamma\max_{a'}Q_m(s',a')-f(s,a))^2$ and tuned to minimize validation error $\sum_{(s,a,r,s')\in\mathcal D_V}(r+\gamma\max_{a'}Q_m(s',a')-\widehat{\mathcal B}^*Q_m(s,a))^2$. The goal is for $\widehat{\mathcal B}^*Q_m$ to be as accurate an estimator for $\mathcal B^*Q_m$ as possible, so that $||Q_m-\widehat{\mathcal B}^*Q_m||_{\mathcal D_V}^2$ will be an accurate estimator of the true MSBE $||Q_m-\mathcal B^*Q_m||^2_{P^\mu}$.  In contrast, BVFT uses assessment criterion $\text{BVFT}(Q_m)=\max_{Q\in\mathcal Q}||Q_m-\widehat{\mathcal B}^*_{\mathcal G(\bar Q_m^{(\epsilon)},\bar Q^{(\epsilon)})}Q_m||^2_{\mathcal D_V}$ where $\bar Q_m^{(\epsilon)}$ discretizes $Q_m$ with resolution $\epsilon$, $\widehat{\mathcal B}^*_{\mathcal G(\bar Q_m^{(\epsilon)},\bar Q^{(\epsilon)})}Q_m$ is a weighted $L_2$ projection of $\mathcal B^*Q_m$ onto functional class $\mathcal G(\bar Q_m^{(\epsilon)},\bar Q^{(\epsilon)})$ with weights given by the empirical distribution of $\mathcal D_V$, and $\mathcal G(\bar Q_m^{(\epsilon)},\bar Q^{(\epsilon)})$ is the smallest piecewise-constant functional class containing both $\bar Q_m^{(\epsilon)}$ and $\bar Q^{(\epsilon)}$. For details, please see \citet{Xie2021}.

Note that the projected Bellman backups in BVFT are restricted to projections onto piecewise-constant functional classes that depend on the resolution hyperparameter $\epsilon$ and the candidate set $\mathcal Q$. In contrast, while SBV can use piecewise-constant projections, it can also use neural network projections, elastic-net penalization, tree-based ensembles, or any other training scheme to estimate the Bellman backup, depending on what maximizes accuracy. Moroever, while SBV uses a single projection of the Bellman error (namely the most accurate that can be found) to assess $Q_m$, BVFT takes the maximum of multiple projections, which could induce overestimation issues when the number of projections is large. These characteristics make SBV a more direct and principled estimator of the MSBE compared to BVFT. 

However, this does not mean that SBV is better than BVFT per se. Indeed, \citet{Xie2021} showed that a weighted $L_2$ projection of the Bellman operator onto a piecewise-constant functional class with weights based on $P^\mu$ will still be a $\gamma$-contraction in $L_\infty$ norm and will still have a unique fixed point at $Q^*$, even though it may be very different from the real Bellman operator, so long as this piecewise-constant functional class contains $Q^*$ and $P^\mu(s,a)>0$ for all $(s,a)\in\mathcal S\times\mathcal A$. Therefore, it can be seen that if $Q^*\in\mathcal Q$ and $Q^*$ does not change following discretization (because it is piecewise-constant), asymptotically $\text{BVFT}(Q_m)$ will give a value of zero if $Q_m=Q^*$ and will give a value greater than zero if $Q_m\neq Q^*$. In such cases, BVFT will correctly recover the true optimal value function regardless of whether or not BVFT outputs accurate estimates of the MSBE for all candidates. Moreover, provided the discretization resolution $\epsilon$ declines with sample size at the right rate, BVFT can still be proven to perform well even if $Q^*$ is not piecewise-constant. Finally, piecewise-constant projections are stable and allow for the asymptotic guarantees of BVFT to easily extend to finite-sample settings. As a result, \citet{Xie2021} were able to derive state-of-the-art finite-sample guarantees for BVFT. 

While the primary goal of SBV is to estimate the MSBE as accurately as possible, it is more useful to think of BVFT as outputting a related quantity that is different from the MSBE but still has its own theoretical guarantees.

\subsection{The Relationship between SBV and ModBE}
\label{append:modbe}

For this section, we assume readers are already familiar with ModBE \cite{Lee2022}. To best explain the differences between SBV and ModBE, we consider a finite-horizon non-discounted MDP where the only hyperparameters of interest are the model class (e.g. not the optimizer or any penalty terms) and we wish to select between two nested model classes $\mathcal F_1\subset\mathcal F_2$ for running FQI. Let $f_j^k$ be the $k$th time-step Q-function obtained by running finite-horizon FQI using model class $\mathcal F_j$ \cite{Lee2022}. SBV compares:
\begin{equation}
\sum_k \mathbb E_{\mathcal D_V}\left[\left(f^k_1(s,a)-\widehat{\mathcal B}^*f^{k+1}_1(s,a)\right)^2\right] \text{ vs. } \sum_k E_{\mathcal D_V}\left[\left(f^k_2(s,a)-\widehat{\mathcal B}^*f^{k+1}_2(s,a)\right)^2\right],\label{eq:sbv_compare1}
\end{equation}
where $\widehat{\mathcal B}^*f^{k+1}_1(s,a)=\text{argmin}_{g\in\mathcal G_1}\sum_{(s,a,r,s')\in\mathcal D_T}(g(s,a)-r-\max_{a'}f_1^{k+1}(s',a'))^2$ and $\widehat{\mathcal B}^*f^{k+1}_2(s,a)=\text{argmin}_{g\in\mathcal G_2}\sum_{(s,a,r,s')\in\mathcal D_T}(g(s,a)-r-\max_{a'}f_2^{k+1}(s',a'))^2$. $\widehat{\mathcal B}^*f^{k+1}_1(s,a)$ estimates $\mathcal B^* f^{k+1}_1(s,a)$ using functional class $\mathcal G_1$ tuned to minimize error on validation set $\mathcal D_V$, i.e. to maximize accuracy of $\mathcal B^* f^{k+1}_1(s,a)$, and similarly $\mathcal G_2$ is tuned to maximize accuracy of $\mathcal B^* f^{k+1}_2(s,a)$. The regression model classes $\mathcal G_1$ and $\mathcal G_2$ can be completely different than the FQI model classes $\mathcal F_1$ and $\mathcal F_2$. The ultimate objective of our goal is to compare $\sum_k ||f^k_1-\mathcal B^*f^{k+1}_1||_{P^\mu}^2 \text{ vs. } \sum_k ||f^k_2-\mathcal B^*f^{k+1}_2||_{P^\mu}^2$.  

In contrast, ModBE compares:
\begin{equation}
\sum_k\mathbb E_{\mathcal D_V}\left[\left(f^k_1(s,a)-r-\max_{a'}f^{k+1}_1(s',a')\right)^2\right] \text{ vs. } \sum_k\mathbb E_{\mathcal D_V}\left[\left(g^k_2(s,a)-r-\max_{a'}f^{k+1}_1(s',a')\right)^2\right]+P,\label{eq:modbe_compare1}
\end{equation}
where $g^k_2=\text{argmin}_{g\in \mathcal F_2}\sum_{(s,a,r,s')\in\mathcal D_T}(g(s,a)-r-\max_{a'}f_1^{k+1}(s',a'))^2$, $P$ is a penalty constant that depends on $|\mathcal F_1|,|\mathcal F_2|,|\mathcal D|$ and the horizon length, among other terms. Unlike Equation \ref{eq:sbv_compare1}, Equation \ref{eq:modbe_compare1} is not comparing the MSBE between $f_1^k$ and $f_2^k$ because $f_2^k$ is not even calculated. Instead, it is measuring the ability of $\mathcal F_2$ to model with greater accuracy the same intermediate targets as those obtained by running FQI with model class $\mathcal F_1$. These are fixed-target empirical Bellman errors as we defined in Section \ref{sec:bellman_error}.  If the right-hand-sum sum of Equation \ref{eq:modbe_compare1} is smaller than the left-hand sum, this suggests $\mathcal F_1$ is not \textit{complete} \citep{Munos2008, Chen2019} and we should choose the more complex model class $\mathcal F_2$. Otherwise, we should stick to $\mathcal F_1$, since it yields smaller estimation variance. 

Let $\widehat{\mathcal B}^*_{\mathcal F} f_1^{k+1}$ approximate $\mathcal B^* f_1^{k+1}$ using an $L_2$ weighted projection onto function class $\mathcal F$ with weights given by the empirical distribution of $\mathcal D_T$. We can also compare ModBE and SBV by noting that $f^k_1=\widehat{\mathcal B}^*_{\mathcal F_1} f_1^{k+1}$ and $f^k_2=\widehat{\mathcal B}^*_{\mathcal F_2} f_2^{k+1}$. The comparison for ModBE is then:
\begin{equation}
\begin{split}
\sum_k\mathbb E_{\mathcal D_V}\left[\left(\widehat{\mathcal B}^*_{\mathcal F_1} f_1^{k+1}(s,a)-r-\max_{a'}f^{k+1}_1(s',a')\right)^2\right] \text{ vs. }  \\ \sum_k\mathbb E_{\mathcal D_V}\left[\left(\widehat{\mathcal B}^*_{\mathcal F_2} f_1^{k+1}(s,a)-r-\max_{a'}f^{k+1}_1(s',a')\right)^2\right]+P.\label{eq:modbe_compare2}
\end{split}
\end{equation}
In contrast, SBV compares:
\begin{equation}
\sum_k \mathbb E_{\mathcal D_V}\sum_k \mathbb E_{\mathcal D_V}\left[\left(f^k_1(s,a)-\widehat{\mathcal B}_{\mathcal G_1}^*f^{k+1}_1(s,a)\right)^2\right] \text{ vs. } \sum_k E_{\mathcal D_V}\left[\left(f^k_2(s,a)-\widehat{\mathcal B}^*_{\mathcal G_2}f^{k+1}_2(s,a)\right)^2\right].\label{eq:sbv_compare2}
\end{equation}
ModBE compares estimation accuracy for $\mathcal B^*f_1^{k+1}$ using model classes $\mathcal F_1$ and $\mathcal F_2$ (averaged over $k$), and chooses $\mathcal F_2$ if accuracy is better. Meanwhile for SBV, we can think of $\mathcal G_1$ as being derived by conducting similar comparisons in equation \ref{eq:modbe_compare2} but without the penalty term and over many possible functional classes $\mathcal G_{1,1}, \mathcal G_{1,2}, ..., \mathcal G_{1,N}$, and choosing that which is smallest. In other words, $\mathcal G_1=\min_{1\leq n \leq N}L(\mathcal G_{1,n},f_1)$ where $L(\mathcal G_{1,n},f_1)=\sum_k\mathbb E_{\mathcal D_V}\left[\left(\widehat{\mathcal B}^*_{\mathcal G_{1,n}} f_1^{k+1}(s,a)-r-\max_{a'}f^{k+1}_1(s',a')\right)^2\right]$. Likewise, $\mathcal G_2$ would be derived as $\mathcal G_2=\min_{1\leq n \leq N}L(\mathcal G_{2,n},f_2)$ with potentially different candidates $\mathcal G_{2,1}, \mathcal G_{2,2}, ..., \mathcal G_{2,N}$. In this case, SBV performs similar comparisons as ModBE. However, unlike in ModBE, the model classes $\mathcal G_{1,1}, \mathcal G_{1,2}, ..., \mathcal G_{1,N}$ compared can be different from $\mathcal F_1$ and $\mathcal F_2$. Moreover, SBV doesn't ultimately choose the function class for FQI based on the Bellman backup accuracy comparisons $L(\mathcal G_{1,n},f_1)$ and $L(\mathcal G_{2,n},f_2)$. Rather, finding the function classes $\mathcal G_1$ and $\mathcal G_2$ that maximizes Bellman backup estimation accuracy is a proxy for estimating the true MSBE accurately, and we ultimately compare estimated MSBE (Equation \ref{eq:sbv_compare2}) between $f_1$ and $f_2$. 

%While \citet{Lee2022} derives very state-of-the-art finite-sample theoretical guarantees for their algorithm, they require restrictive assumptions to obtain such guarantees. For example, they can only compare between nested model classes, assume the RL problem has a finite horizon and assume complexity measures for the model classes which are known and finite. This prevents ModBE from being applied to any of the problems we evaluated in this paper. While a practical implementation was derived to extend ModBE to infinite-horizon setting and relax many of these assumptions, the implementation lacks any theoretical guarantees and it appears the assumption of nested model classes with known complexity measures are still required. The practical algorithm was also only applied to simple environments with candidate Q-functions generated using a two-layer MLP varying only the hidden dimension width. In contrast, SBV can compare arbitrary candidate Q-functions, regardless of their respective model classes or what training algorithm(s) was used to generate them. Moreover, SBV performed well in complex Atari environments (Section \ref{sec:atari}).

\twocolumn
\section{Additional Notes}
\label{append:additional_notes}

\subsection{Bellman Errors and Continuous Control Tasks}
\label{append:cts}

Continuous control problems are typically solved by deep policy gradient algorithms \cite{Schulman2015} or deep actor-critic algorithms \cite{Lillicrap2016}. Bellman errors measure the error of an estimated action-value function, and as policy gradient algorithms do not estimate action-value functions, there are no Bellman errors that can be computed. In contrast, deep actor-critic algorithms require two neural network models: a Q-network (or critic) $Q_\theta$ and a policy model (or actor) $\pi_\phi$. The $\gamma$-contraction properties of the Bellman operator and policy iteration theory \cite{Busoniu2010} suggests that we would need to ensure that (1) $Q_\theta\approx \mathcal B^{\pi_\phi}Q_\theta$ and (2) $Q_\theta(s,\pi_\phi(s))\approx \max_a Q_\theta(s,a)$, where $\mathcal B^\pi Q(s,a)$ is equal to: 
\begin{equation*}
\mathbb E[R_t+\gamma \mathbb E_{a'\sim \pi(\cdot|S_{t+1})}Q(S_{t+1},a')|S_t=s,A_t=a].
\end{equation*}
Indeed, even if the first condition is satisfied completely, this would only relate to the critic $Q_\theta$ accurately estimating $Q^{\pi_\phi}$, and the actor $\pi_\phi$ could still be arbitrarily sub-optimal. While SBV could help ensure the first condition, extending SBV to deal with both conditions is less trivial and left for future work. We should also note that there are a few deep off-policy RL algorithms for continuous control settings that use stochastic optimization to avoid specification of an actor network \cite{Kalashnikov2018, Khan2021}, and SBV would be directly applicable to such algorithms.

\subsection{Using Bellman Networks to Improve DQN}
\label{append:Bellman_q}

In the Atari setting, we found that by using architectures for the Q-network that performed as well as Bellman networks, and by greatly reducing target update frequency, we were able to achieve Q-functions with low Bellman error without having to explore a large number of different RL training configurations. For example, the Q-network architecture utilized for our deep DQN configuration is just an architecture that performed well as a Bellman network according to validation MSE on the four analyzed Atari environments (see Appendix \ref{append:dqn_deep} for more details). We also found a similar relationship for other environments: In the mHealth environment for example, quadratic functions usually performed best both for estimating $Q^*$ and Bellman backups of the candidate Q-functions. Our results indicate that SBV should perform well on remaining Atari environments as well, provided the Bellman network validation error is further reduced and the Q-network architecture is further improved using a similar strategy. 

We further compare SBV and DQN to provide some insight into why this phenomenon might be occurring. Recall that DQN performs updates $Q^{(k+1)}\leftarrow \mbox{argmin}_{f\in\mathcal F} \mathbb E_{\mathcal D_T}\left[(r+\gamma\max_{a'}Q^{(k)}(s',a')-f(s,a))^2\right]$, where $\mathcal F$ is a functional class corresponding to a pre-specified neural network architecture and the minimization is performed using a small number of gradient descent steps determined by the target update frequency. Intuitively, a good choice of $\mathcal F$ is one that can accurately estimate $\mathbb E[R_t+\gamma\max_{a'}Q(S_{t+1},a')|S_t,A_t]$ for all $Q=Q^{(1)},Q^{(2)},...\in\mathcal F$. This is related to an important condition for FQI theoretical guarantees known as \textit{completeness} \cite{Munos2008, Chen2019}. In contrast, SBV uses estimators $\widehat{\mathcal B}^*Q_m=\mbox{argmin}_{f\in\mathcal F} \mathbb E_{\mathcal D_T}\left[(r+\gamma\max_{a'}Q_m(s',a')-f(s,a))^2\right]$ for $Q_m\in\mathcal Q$ where minimization is over a large number of gradient descent steps determined by prediction error on $\mathcal D_V$. Intuitively, a good choice of $\mathcal F$ is one that can accurately estimate $\mathbb E[R_t+\gamma\max_{a'}Q(S_{t+1},a')|S_t,A_t]$ for all $Q=Q_1,...,Q_M\in\mathcal Q$. While these problems are not identical, they are similar: In both cases, we are essentially trying to solve regression problems on the same dataset using the same covariates, with targets given by the reward plus a complex neural-network function of the next state. We would thus expect some relationship between the optimal function approximator or training algorithm for each task. 

\newpage
\section{Additional Tables and Figures}
\label{append:additional_figures}

\setcounter{figure}{0}
\setcounter{table}{0}

\begin{figure}[ht]
\centering
\includegraphics[width=\linewidth]{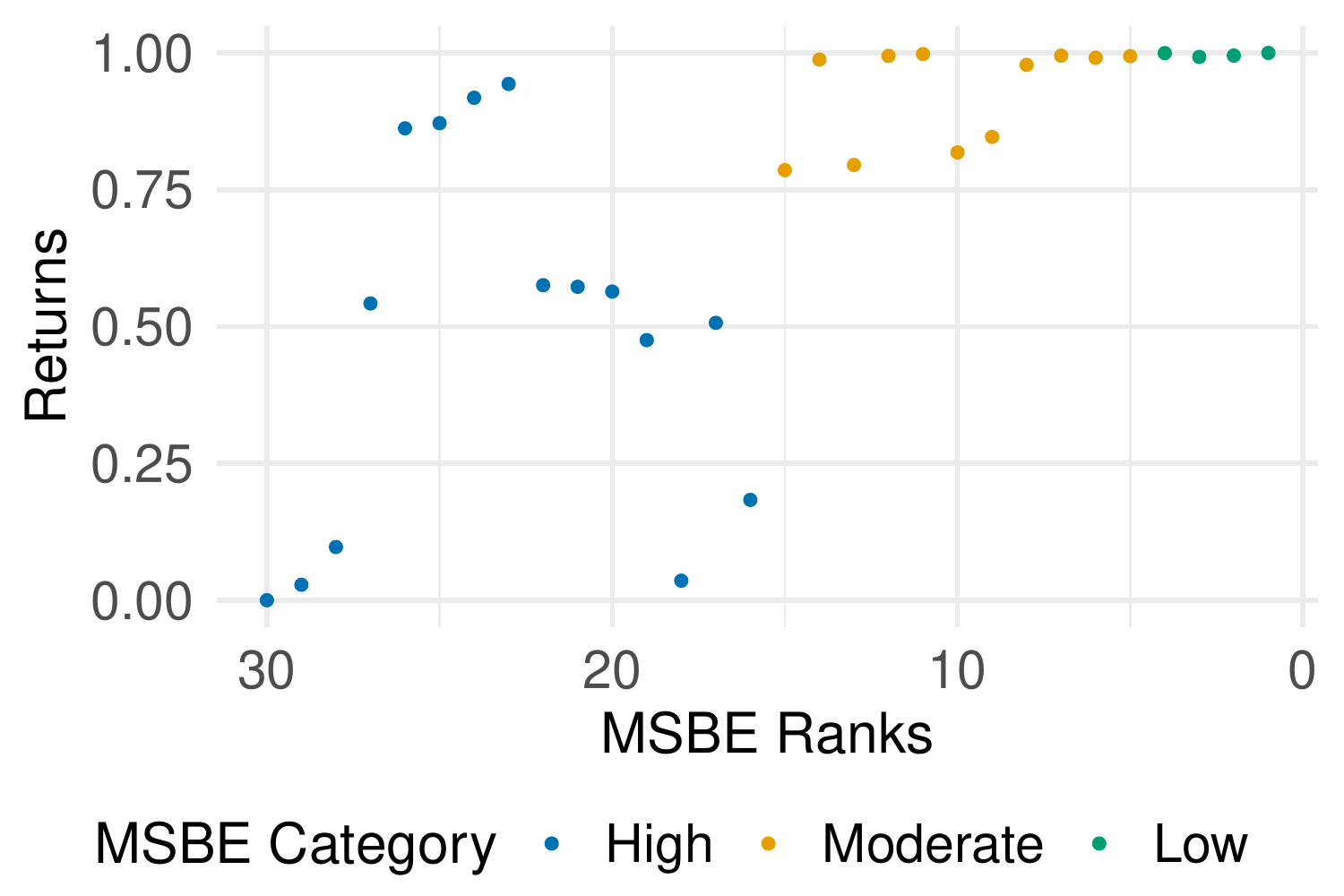} 
\caption{Returns vs. MSBE Ranks. An MSBE rank of $30$ means the estimate has the highest observed MSBE, while a rank of zero means the estimate has the lowest MSBE. MSBE values are grouped similarly to Figure \ref{figure:msbe}, though we use a more color blind-friendly scheme here. Note that even though the relationship between MSBE and returns is not perfectly monotonic, the top policies according to the MSBE (i.e. those categorized as having ``low" MSBE) still perform very well.}\label{figure:msbe_extra}
\end{figure}

\begin{figure}[!ht]
\centering
\includegraphics[width=\linewidth]{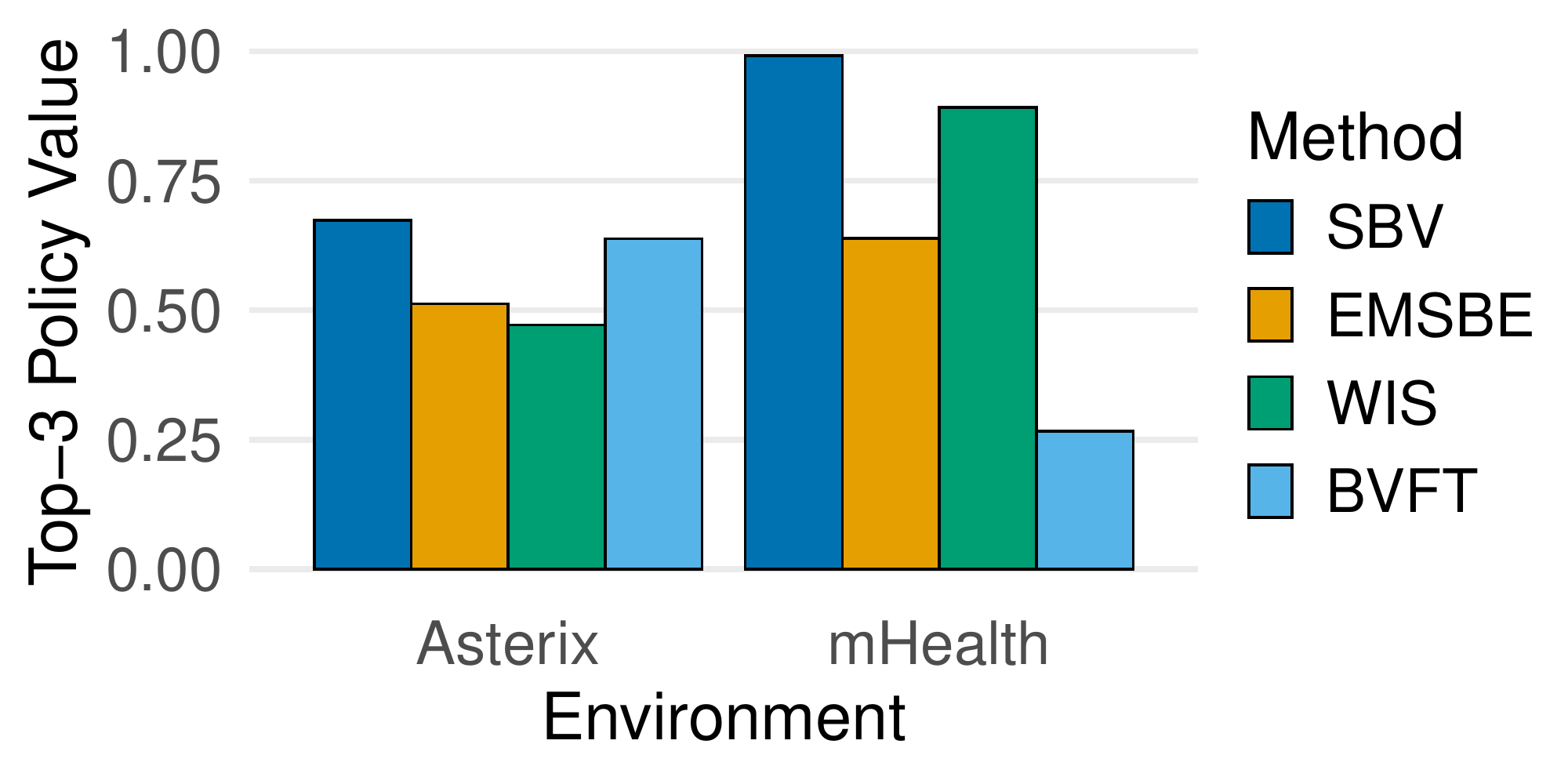} 
\caption{BVFT Performance on Two Environments. We tuned the discretization resolution of BVFT as in \citet{Zhang2021ps}. We report the standardized top-3 policy value averaged over datasets similar to Figure \ref{figure:bike}. While BVFT performs comparably to SBV on Asterix datasets, it performs worse than random chance on mHealth datasets (see Sections \ref{sec:Bike} and \ref{sec:atari} for more details on these datasets and metrics). We admit that we only became aware of a practical implementation of BVFT \cite{Zhang2021ps} very recently. Hence, we did not have time to evaluate BVFT as thoroughly as other OMS algorithms.
}\label{figure:bvft}
\end{figure}

\begin{figure}[ht]
%\vspace{-7cm}
\centering
\includegraphics[width=\linewidth]{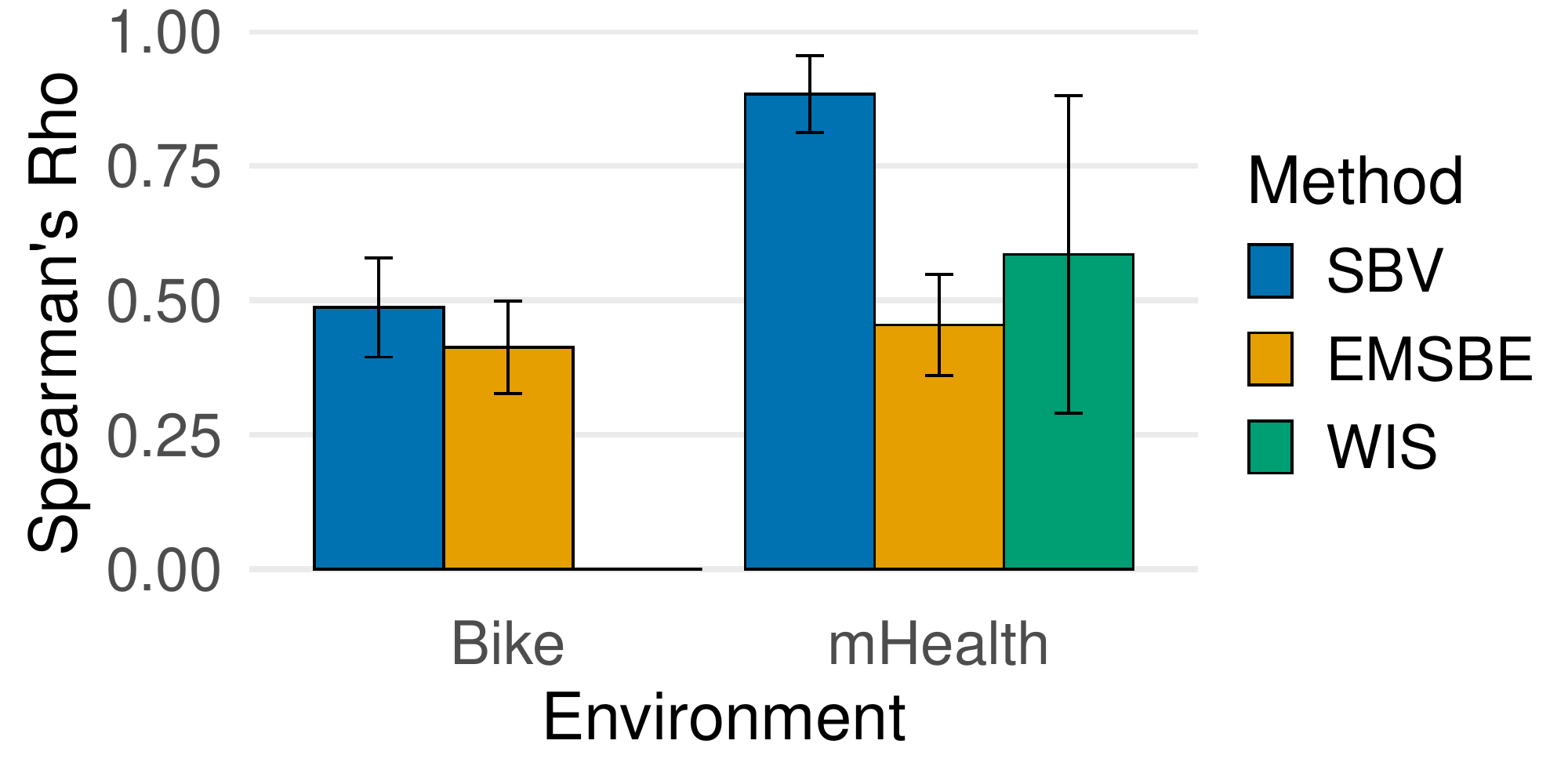} 
\caption{Absolute Spearman Correlation with Policy Value. Mean (Std) of this metric across datasets is reported for each method. While SBV achieves the highest Spearman correlation for both environments, it is still rather low for the Bicycle datasets. Despite this, SBV is still effective at selecting high-quality policies (Figures \ref{figure:bike} and \ref{figure:bike_oracle}).}\label{figure:spearman}
%\vspace{-7cm}
\end{figure}

\begin{figure}[!ht]
\centering
\includegraphics[width=\linewidth]{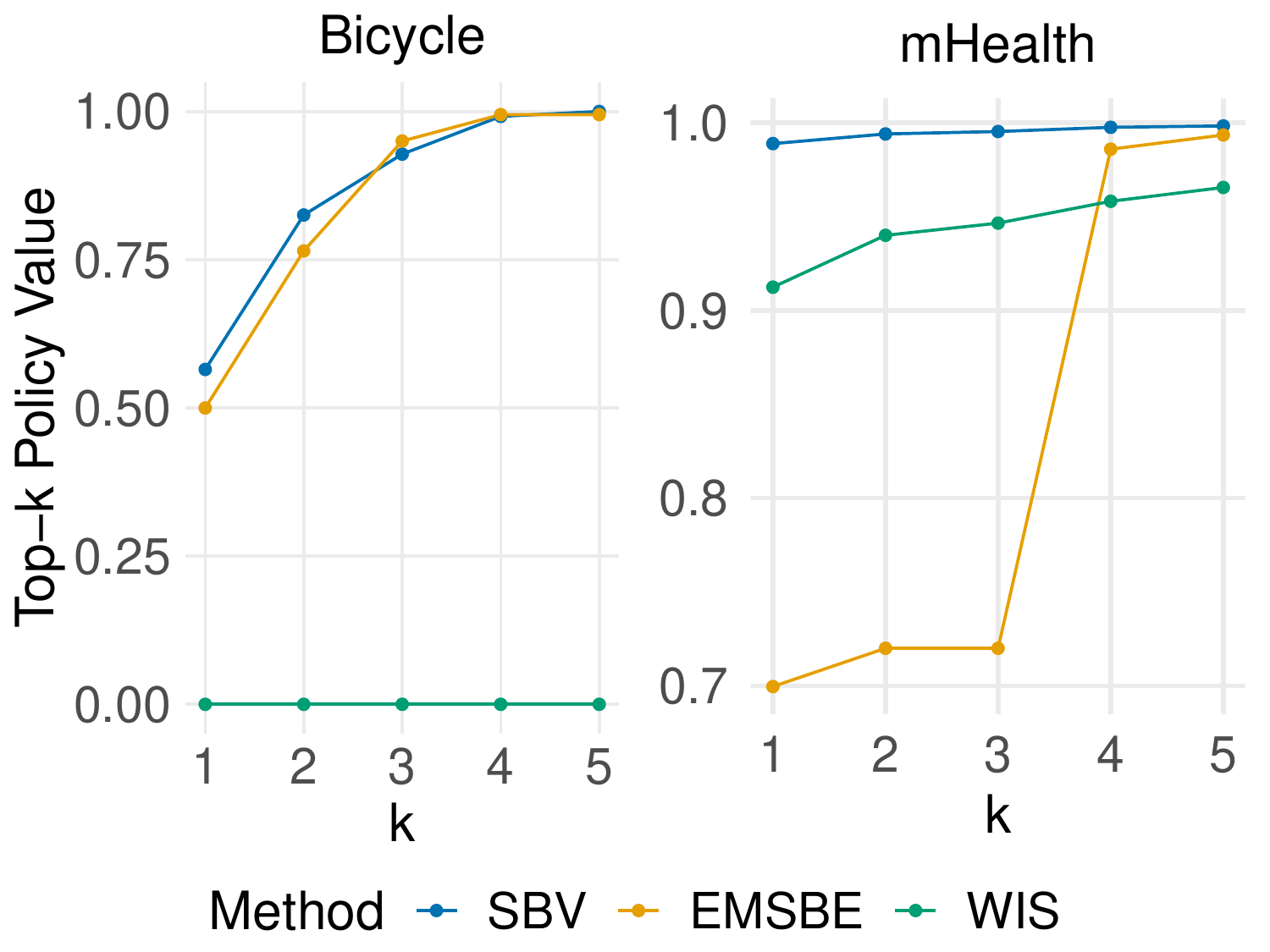} 
\caption{Max Top-$k$ Policy Values for Bicycle and mHealth Environments. Each data point represents the maximum policy value observed among the top-$k$ policies selected for a particular method, averaged over 10 datasets. The max top-$k$ policy value is a useful metric for situations where an online environment is available but can only evaluate up to $k$ policies for safety or cost reasons. While SBV performs well across environments, WIS performs poorly on Bicycle for all values of $k$ while the EMSBE performs poorly on mHealth for small values of $k$ (i.e. when only limited online interactions are available).}\label{figure:bike_oracle}
\end{figure}

\begin{table}[b]
\centering
\caption{Performance of SBV on the Same Four Datasets Evaluated by FQE. Displayed are the top-5 policy returns analogous to Table \ref{table:atari}.}\label{table:atari_fqe}
\begin{tabular}{lllll}
 \toprule
Method& Pong & Breakout & Asterix & Seaquest \\
 \midrule
SBV & 98\% & 69\% & 73\% & 70\% \\
FQE  & 98\% & 41\% & 53\% & 34\% \\
 \bottomrule
\end{tabular}
\end{table}

\begin{figure*}[t]
\centering
\vspace{-15cm}
\includegraphics[width=\textwidth]{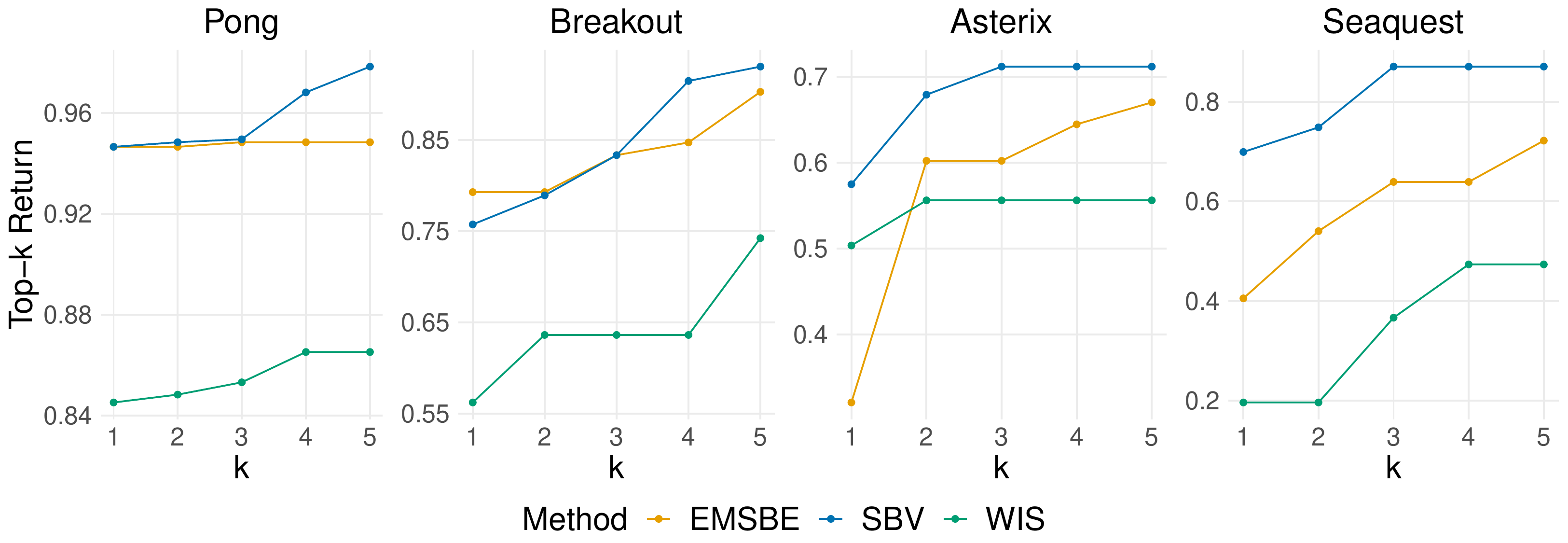} 
\caption{Max Top-$k$ Policy Values for Atari Environments. Each data point represents the maximum expected return observed among the top-$k$ policies selected for a particular method, averaged over 3 datasets. The max top-$k$ policy value is a useful metric for situations where an online environment is available but can only evaluate up to $k$ policies for safety or cost reasons. On average, SBV significantly outperforms competitors across environments and values of $k$.}\label{figure:atari_oracle}
\end{figure*}

\onecolumn

\begin{figure*}
\centering
\includegraphics[width=\textwidth]{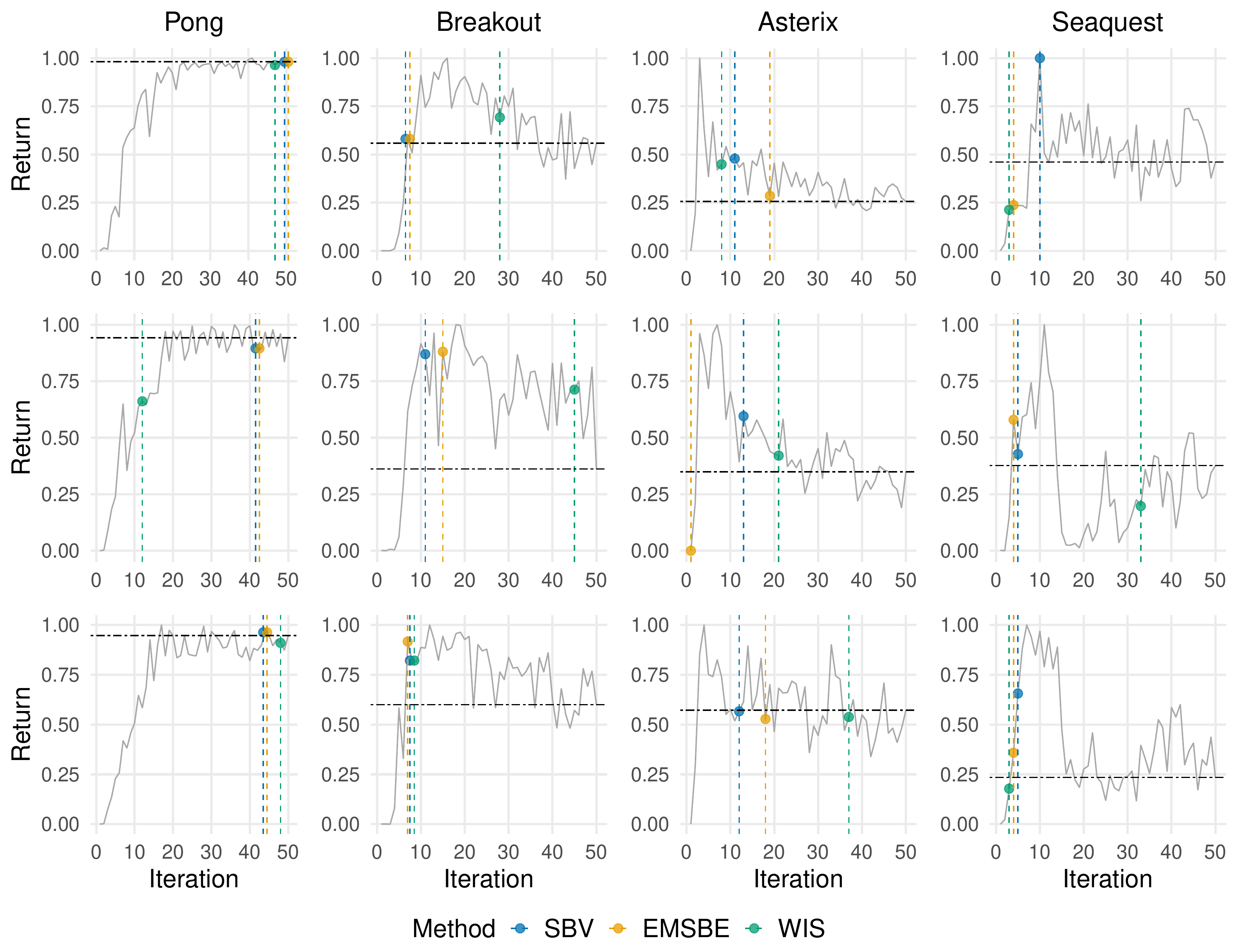} 
\caption{Learning Curves for the Best Configuration for all 12 analyzed Atari datasets. Returns are standardized to a $[0,1]$ range within each dataset. The dashed horizontal line represents performance when no early stopping is applied. The vertical lines represent the iterations where training was stopped according to different methods. Only SBV performs nearly as well as or better than no early stopping for all datasets. FQE was only evaluated on the four datasets present in Figure \ref{figure:atari} and is thus excluded here.}\label{figure:atari_extra}
\end{figure*}

\begin{table*}[b]
\centering
\caption{Performance of Model-Based Evaluation on the Bicycle Balancing Problem (see Section \ref{sec:Bike}). Similar to previous work \citep{Janner2019}, we modelled state transition dynamics as a multivariate normal distribution with a diagonal covariance matrix, where the vector of means and log-standard deviations were outputted from a single feed-forward neural network with two hidden layers of 200 units each. Below the mean ($\pm$ sd) of the top-3 policy values across 10 datasets is given for each method, similar to Figure \ref{figure:bike}. Model-based evaluation performs worse on average and has more variability between datasets.}\label{table:model}
\begin{tabular}{lllll}
 \toprule
Method& Top-3 Policy Return  \\
 \midrule
SBV (Ours)  & 82\% ($\pm$15\%) \\
EMSBE (Equation \ref{eq:emsbe})  & 83\% ($\pm$16\%) \\
Model-Based Evaluation & 72\% ($\pm$25\%)  \\
 \bottomrule
\end{tabular}
\end{table*}

\twocolumn

\begin{figure}
\includegraphics[width=\linewidth]{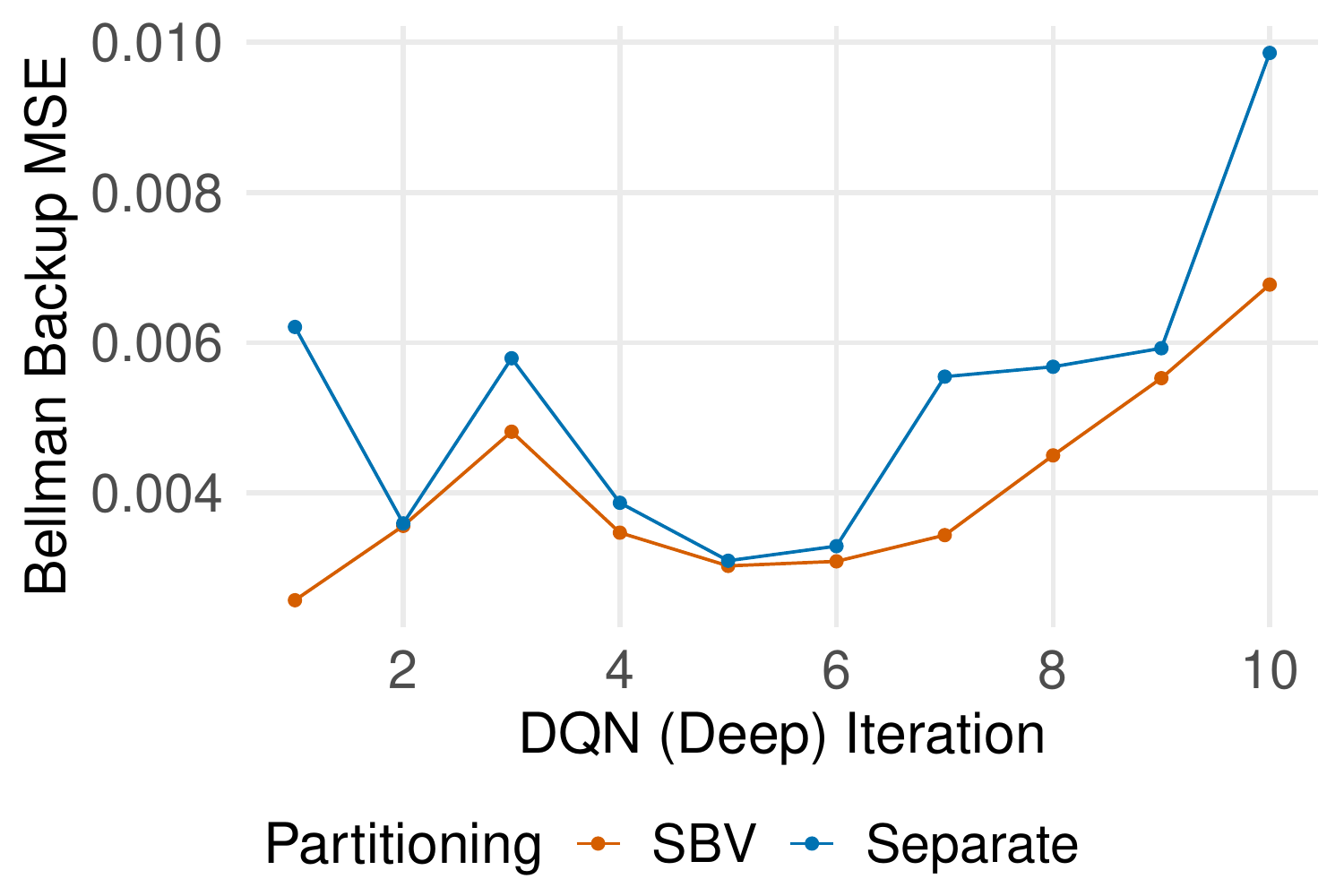}
\centering
\caption{Partitioning Ablation on Breakout. For the first 10 Q-functions generated by the deep DQN configuration, we estimated their Bellman backups using the same dataset as that used to run DQN and plot their validation MSE in red. We then do the same for a separate dataset of 50\% size in blue.}
\label{figure:part_breakout}
\vspace{-0.4cm}
\end{figure}
\end{document}